\pdfoutput=1
\documentclass[twoside,11pt]{article}

\usepackage{blindtext}

\usepackage{jmlr2e}

\usepackage{color}
\newcommand{\normmm}[1]{{\left\vert\kern-0.25ex\left\vert\kern-0.25ex\left\vert #1 
   \right\vert\kern-0.25ex\right\vert\kern-0.25ex\right\vert}}
\newcommand{\normvec}[1]{\Vert #1 \Vert}
\newcommand{\sroman}[1]{(\text{\romannumeral#1})}
\newcommand{\F}{{1,\infty}}

\usepackage[ruled,linesnumbered]{algorithm2e}

\usepackage{diagbox}
\usepackage{multirow}

\usepackage{graphicx}
\usepackage{subfigure}
\usepackage{bm}
\usepackage{ulem}

\usepackage[utf8]{inputenc} 
\usepackage{booktabs}       
\usepackage{amsfonts}       
\usepackage{nicefrac}       
\usepackage{microtype}      
\usepackage{xcolor}         

\usepackage{amsmath,amssymb}
\usepackage{mathrsfs}
\usepackage{empheq}
\usepackage{bbm}

\newtheorem{assumption}[theorem]{Assumption}
\newtheorem{innercustomlemma}{Lemma}
\newenvironment{customlemma}[1]
  {\renewcommand\theinnercustomlemma{#1}\innercustomlemma}
  {\endinnercustomlemma}

\newtheorem{innercustomthm}{Theorem}
\newenvironment{customthm}[1]
  {\renewcommand\theinnercustomthm{#1}\innercustomthm}
  {\endinnercustomthm}

\newtheorem{innercustompro}{Proposition}
\newenvironment{custompro}[1]
  {\renewcommand\theinnercustompro{#1}\innercustompro}
  {\endinnercustompro}

\usepackage{lastpage}
\jmlrheading{24}{2023}{1-\pageref{LastPage}}{7/22; Revised 12/22}{2/23}{22-0866}{Xin Zou and Weiwei Liu}

\ShortHeadings{Generalization Bounds for Adversarial Contrastive Learning}{Zou and Liu}
\firstpageno{1}

\begin{document}

\title{Generalization Bounds for Adversarial Contrastive Learning}

\author{\name Xin Zou \email zouxin2021@gmail.com \\
       \addr School of Computer Science\\
       Wuhan University\\
       Wuhan, Hubei, China
       \AND
       \name Weiwei Liu\thanks{Corresponding Author.} \email liuweiwei863@gmail.com \\
       \addr School of Computer Science\\
       Wuhan University\\
       Wuhan, Hubei, China}

\editor{Joan Bruna}

\maketitle
\begin{abstract}
Deep networks are well-known to be fragile to adversarial attacks, and adversarial training is one of the most popular methods used to train a robust model. To take advantage of unlabeled data, recent works have applied adversarial training to contrastive learning (Adversarial Contrastive Learning; ACL for short) and obtain promising robust performance. However, the theory of ACL is not well understood. To fill this gap, we leverage the Rademacher
complexity to analyze the generalization performance of ACL, with a particular focus on linear models and multi-layer neural networks under $\ell_p$ attack ($p \ge 1$). Our theory shows that the average adversarial risk of the downstream tasks can be upper bounded by the adversarial unsupervised risk of the upstream task. The experimental results validate our theory.
\end{abstract}
\begin{keywords}
  Robustness, Adversarial learning, Contrastive learning, Rademacher complexity, Generalization bound.
\end{keywords}

\section{Introduction} \label{sec-introduction}
Deep neural networks (DNNs) have achieved state-of-the-art performance in many fields. However, several prior works \citep{DBLP:journals/corr/SzegedyZSBEGF13, DBLP:journals/corr/FGSM} have shown that DNNs may be vulnerable to imperceptibly changed adversarial examples, which causes a lot of focus on the robustness of the models \citep{DBLP:conf/iclr/PGD, DBLP:conf/eccv/MaoGNRSYV20, xiyuanli-2022}.

One of the most popular approaches to achieving adversarial robustness is adversarial training, which involves training the model with samples perturbed to maximize the loss on the target model \citep{DBLP:journals/corr/FGSM, DBLP:conf/iclr/PGD, DBLP:conf/icml/TRADES}. \citet{DBLP:conf/nips/SchmidtSTTM18} show that adversarial robust generalization requires a larger amount of data, while \citet{DBLP:conf/icml/YinRB19, DBLP:conf/icml/AwasthiFM20} show that the Rademacher Complexity of adversarial training is strictly larger in theory than that of natural training, which implies that we need more data for adversarial training.

Since the labeled data is limited and expensive to obtain, one option would be to use large-scale unlabeled data and apply self-supervised learning \citep{DBLP:conf/iclr/GidarisSK18, DBLP:conf/eccv/NorooziF16}, an approach that trains the model on unlabeled data in a supervised manner by utilizing self-generated labels from the data itself. Contrastive Learning (CL) \citep{DBLP:conf/icml/ChenK0H20-SimCLR,DBLP:conf/cvpr/He0WXG20-MoCo}, which aims to maximize the feature similarity of similar pairs and minimize the feature similarity of dissimilar ones, is a popular self-supervised learning technique.

Recently, \citet{DBLP:conf/nips/KimTH20, DBLP:conf/nips/HoN20, DBLP:conf/nips/JiangCCW20} apply adversarial training in CL and achieve state-of-the-art model robustness. They find that if adversarial training is conducted on the upstream contrastive task, the trained model will be robust on the downstream supervised task. However, their results are all empirical and lack theoretical analysis. To fill this gap, we here present a theoretical analysis of adversarial contrastive learning through the lens of Rademacher complexity, with a particular focus on linear models and multi-layer neural networks. Our theoretical results show that the average adversarial risk of the downstream tasks can be upper bounded by the adversarial unsupervised risk of the upstream task; this implies that if we train a robust feature extractor on the upstream task, we can obtain a model that is robust on downstream tasks.

The remainder of this article is structured as follows: \S \ref{sec-relatedwork} introduces some related works in this field. \S \ref{sec-setting-up} provides some basic definitions and settings that will be used in the following sections. \S \ref{sec-theoretical-ana} presents our first part's main results, which show a connection between the robust risk for the upstream task and the robust risk for the downstream tasks. \S \ref{sec-bin} shows our second part's results, it outlines our bounds for linear models and multi-layer neural networks by bounding the Rademacher complexity of each hypothesis class. \S \ref{sec-exp} shows some experimental results that verify our theory; finally, the conclusions are presented in the last section.

\section{Related Work} \label{sec-relatedwork}

\textbf{Adversarial Robustness.} \citet{DBLP:journals/corr/SzegedyZSBEGF13} show that DNNs are fragile to imperceptible distortions in the input space. Subsequently, \citet{DBLP:journals/corr/FGSM} propose the fast gradient sign method (FGSM), which perturbs a target sample towards its gradient direction to increase the loss and then uses the generated sample to train the model in order to improve the robustness. Following this line of research, \citet{DBLP:conf/iclr/PGD, DBLP:conf/cvpr/Moosavi-Dezfooli16, DBLP:conf/iclr/KurakinGB17a, DBLP:conf/sp/Carlini017} propose iterative variants of the gradient attack with improved adversarial learning frameworks. Besides, \citet{xingongma-2022} analyze the trade-off between robustness and fairness, \citet{https://doi.org/10.48550/arxiv.2302.04025} study the worst-class adversarial robustness in adversarial training. For the theoretical perspective, \citet{DBLP:conf/colt/MontasserHS19} study the PAC learnability of adversarial robust learning, \citet{jingyuanxu-2022} extend the work of \citet{DBLP:conf/colt/MontasserHS19} to multiclass case and \citet{DBLP:conf/icml/YinRB19, DBLP:conf/icml/AwasthiFM20} give theoretical analysises to adversarial training by standard uniform convergence argumentation and giving a bound of the Rademacher complexity. Our work is quite different from \citet{DBLP:conf/icml/YinRB19, DBLP:conf/icml/AwasthiFM20}, firstly, they just analyze the linear models and two-layer neural networks, but we consider linear models and multi-layer deep neural networks; secondly, they consider the classification loss, which is much easier to analyze than our contrastive loss.

\textbf{Contrastive Learning.} Contrastive Learning is a popular self-supervised learning paradigm that attempts to learn a good feature representation by minimizing the feature distance of similar pairs and maximizing the feature distance of dissimilar pairs \citep{DBLP:conf/icml/ChenK0H20-SimCLR, zekaiwang-2022}. SimCLR \citep{DBLP:conf/icml/ChenK0H20-SimCLR} learns representations by maximizing the agreement between differently augmented views of the same data, while MoCo \citep{DBLP:conf/cvpr/He0WXG20-MoCo} builds large and consistent dictionaries for unsupervised learning with a contrastive loss. From the theoretical perspective, \citet{DBLP:conf/icml/SaunshiPAKK19} first presents a framework to analyze CL. We generalize their framework to the adversarial CL. The key challenging issues in this work are how to define and rigorously analyze the adversarial CL losses. Moreover, we further analyze the Rademacher complexity of linear models and multi-layer deep neural networks composited with a complex adversarial contrastive loss. \citet{DBLP:conf/colt/NeyshaburTS15} show an upper bound of the Rademacher complexity for neural networks under group norm regularization. In fact, the Frobenius norm and the $\ell_{1, \infty}$-norm in our theoretical analysis are also group norms, while the settings and proof techniques between \citet{DBLP:conf/colt/NeyshaburTS15} and our work are quite different: (1) they consider the standard setting while we consider a more difficult adversarial setting; (2) they consider the Rademacher complexity of neural networks with size $1$ output layer, while we consider the case that the neural network is composited with a complex contrastive learning loss; (3) they prove their results by reduction with respect to the number of the layers, while motivated by the technique of \citet{gao2021theoretical}, we use the covering number of the neural network to upper bound the Rademacher complexity in the adversarial case.

\textbf{Adversarial Contrastive Learning.} Several recent works \citep{DBLP:conf/nips/KimTH20, DBLP:conf/nips/HoN20, DBLP:conf/nips/JiangCCW20} apply adversarial training in the contrastive pre-training stage to improve the robustness of the models on the downstream tasks, achieving good robust performance in their experiments. Our work attempts to provide a theoretical explanation as to why models robustly trained on the upstream task can be robust on downstream tasks.

\section{Problem Setup} \label{sec-setting-up}
We introduce the problem setups in this section.

\subsection{Basic Contrastive Learning Settings}

We first set up some notations and describe the contrastive learning framework.

Let $\mathcal{X} \in \mathbb{R}^m$ be the domain of all possible data points, this paper assumes that $\mathcal{X}$ is bounded. Contrastive learning assumes that we get the similar data in the form of pairs $(x, x^+)$, which is drawn from a distribution $\mathcal{D}_{sim}$ on $\mathcal{X}^2$, and $k$ independent and identically distributed (i.i.d.) negative samples $x_1^-, x_2^-, \dots, x_k^-$ drawn from a distribution $\mathcal{D}_{neg}$ on $\mathcal{X}$. Given the training set $\mathcal{S}=\{(x_i,x_i^+,x_{i1}^-,\dots,x_{ik}^-)\}_{i=1}^M$, we aim to learn a representation $f$ from $\mathcal{F}$ that maps similar pairs $(x, x^+)$ into similar points $(f(x), f(x^+))$, while at the same time keeping $f(x_i^-), \cdots, f(x_k^-)$ away from $f(x)$, where $\mathcal{F}$ is a class of representation functions $f : \mathcal{X} \to \mathbb{R}^n$.

\textbf{Latent Classes.} Let $\mathcal{C}$ denote the set of all latent classes \citep{DBLP:conf/icml/SaunshiPAKK19} that are all possible classes for points in $\mathcal{X}$; for each class $c \in \mathcal{C}$ , moreover, the probability $\mathcal{D}_{c}$ over $\mathcal{X}$ captures the probability that a point belongs to class $c$. The distribution on $\mathcal{C}$ is denoted by $\rho$.

To formalize the similarity among data in $\mathcal{X}$, assume we obtain i.i.d. similar data points $x, x^+$ from the same distribution $\mathcal{D}_c$, where c is randomly selected according to the distribution $\rho$ on latent classes $\mathcal{C}$. We can then define $\mathcal{D}_{sim}$ and $\mathcal{D}_{neg}$ as follows:

\begin{definition} [$\mathcal{D}_{sim}$ and $\mathcal{D}_{neg}$, \citealp{DBLP:conf/icml/SaunshiPAKK19}] \label{def-Dsim-Dneg}
For unsupervised tasks, we define the distribution of sampling similar samples $\mathcal{D}_{sim}(x,x^{+})$ and negative sample $\mathcal{D}_{neg}(x^-)$ as follows:
\begin{equation*}
    \mathcal{D}_{sim}(x,x^+) = \underset{c \sim \rho}{\mathbb{E}} \mathcal{D}_{c}(x)\mathcal{D}_{c}(x^+),\ \  \mathcal{D}_{neg}(x^-) = \underset{c \sim \rho}{\mathbb{E}} \mathcal{D}_{c}(x^-).
\end{equation*}
\end{definition}

\textbf{Supervised Tasks.} For supervised tasks, we focus on the tasks in which a representation function $f$ will be tested on a $(k+1)$-way supervised task $\mathcal{T}$ consisting of distinct classes $\{c_1, c_2, \dots , c_{k+1}\} \subseteq \mathcal{C}$, while the labeled data set of task $\mathcal{T}$ consists of $M$ i.i.d. examples drawn according to the following process: A label $c \in \{c_1, \dots , c_{k+1}\}$ is selected according to a distribution $\mathcal{D}_\mathcal{T}$, and a sample $x$ is drawn from $\mathcal{D}_c$. The distribution of the labeled pair $(x,c)$ is defined as: $\mathcal{D}_\mathcal{T}(x,c) = \mathcal{D}_\mathcal{T}(c) \mathcal{D}_c(x)$.

\subsection{Evaluation Metric for Representations}
We evaluate the quality of a representation function $f$ with reference to its performance on a multi-class classification task $\mathcal{T}$ using a linear classifier.

Consider a task $\mathcal{T} = \{c_1, \dots , c_{k+1}\}$. A multi-class classifier for $\mathcal{T}$ is a function $g : \mathcal{X} \xrightarrow[]{} \mathbb{R}^{k+1}$, the output coordinates of which are indexed by the classes $c$ in task $\mathcal{T}$.

Let $\{g(x)_y - g(x)_{y^{\prime}}\}_{y^{\prime}\neq y}$ be a $k$-dimensional vector of differences in the coordinates of the output of the classifier $g$. The loss function of $g$ on a point $(x,y) \in \mathcal{X} \times \mathcal{T}$ is defined as $\ell(\{g(x)_y - g(x)_{y^{\prime}}\}_{y^{\prime}\neq y})$. For example, one often considers the standard hinge loss $\ell(\boldsymbol{v}) = \max\{0, 1+\max_{i}\{-\boldsymbol{v}_i\}\}$ and the logistic loss $\ell(\boldsymbol{v}) = log_2 (1+\sum_i exp(-\boldsymbol{v}_i))$ for $\boldsymbol{v} \in \mathbb{R}^k$.

The supervised risk of classifier $g$ is defined as follows:
$$L_{sup}(\mathcal{T},g) \coloneqq \underset{(x,c) \sim \mathcal{D}_\mathcal{T}}{\mathbb{E}} \left[\ell\left(\left\{g(x)_{c} - g(x)_{c^{\prime}}\right\}_{c^{\prime} \neq c}\right)\right].$$

The risk $L_{sup}(\mathcal{T},g)$ of classifier $g$ on task $\mathcal{T}$ measures the quality of the outputs of $g$, take the hinge loss as an example, our goal is to get a classifier that has much higher confidence for the true class (i.e., the value $g(x)_c$) than others (i.e., the values $g(x)_{c^\prime}$ for $c^\prime \ne c$), so we want all the differences $g(x)_c-g(x)_{c^\prime}$ for $c \ne c^\prime$ to be as large as possible. If $g$ wrongly classifies $x$, i.e., $\underset{c^\prime}{\arg\max} g(x)_{c^\prime} \ne c$, of course, $\ell$ on $(x,c)$ is not smaller than $1$; however, even thought $g$ correctly classifies $x$, the loss value can not decrease to zero unless $g(x)_c-g(x)_{c^\prime} \ge 1$ for all $c^\prime \ne c$.

Let $\{c_i\}_{i=1}^{k+1} = \{c_1,\dots, c_{k+1}\}$ be a set of classes from $\mathcal{C}$. Given the matrix $W \in \mathbb{R}^{(k+1) \times n}$, we have $g(x) = W f(x)$ as a classifier that composite the feature extractor $g$ and linear classifier $W$. The supervised risk of $f$ is defined as the risk of $g$ when the best $W$ is chosen:
\begin{equation*}
    L_{sup}(\mathcal{T},f) \coloneqq \underset{W \in \mathbb{R}^{(k+1) \times n}}{\inf} L_{sup}(\mathcal{T},W f).
\end{equation*}

When training a feature extractor $f$ on the upstream task, we do not make predictions on the examples, so we can't define the risk of the feature extractor $f$ as we do for the classifier $g$. Our goal on the upstream task is to train a feature extractor that can perform well on the downstream tasks, so we consider the potential of the feature extractor here, i.e., the minimal possible classification risk of a linear classifier on the feature extracted by $f$. So we define the risk of the feature extractor $f$ as above by taking the infimum over all linear classifiers.

\begin{definition}[Mean Classifier, \citealp{DBLP:conf/icml/SaunshiPAKK19}] \label{def-mean-classifier}
For a function $f$ and task $\mathcal{T} = \{c_1, \dots, c_{k+1}\}$, the mean classifier is $W^\mu$, whose $c^{th}$ row is the mean $\mu_c \coloneqq \underset{x \sim \mathcal{D}_c}{\mathbb{E}} [f(x)]$ and we define $L_{sup}^{\mu}(\mathcal{T},f) \coloneqq L_{sup}(\mathcal{T},W^{\mu} f)$.
\end{definition}

Since $L_{sup}(\mathcal{T},f)$ involves taking infimum over all possible linear classifiers, it is difficult to analyze $L_{sup}(\mathcal{T},f)$ and establish connections between the risk of the unsupervised upstream task and $L_{sup}(\mathcal{T},f)$. We introduce the mean classifier to bridge them by bounding average $L_{sup}^{\mu}(\mathcal{T},f)$ over the tasks with the risk of the unsupervised upstream risk (for more details, please refer to \S\ref{sec-theoretical-ana}).

\begin{definition}[Average Supervised Risk] \label{def-avg-sup-loss}
The average supervised risk for a function $f$ on $(k+1)$-way tasks is defined as:
\begin{equation*}
    L_{sup}(f) \coloneqq \underset{\{c_i\}_{i=1}^{k+1} \sim \rho^{k+1}}{\mathbb{E}} [L_{sup}(\{c_i\}_{i=1}^{k+1}, f) | c_i \neq c_j]
\end{equation*}
Given the mean classifier, the average supervised risk for a function $f$ is defined as follows:
\begin{equation*}
    L_{sup}^{\mu}(f) \coloneqq \underset{\{c_i\}_{i=1}^{k+1} \sim \rho^{k+1}}{\mathbb{E}} [L_{sup}^{\mu}(\{c_i\}_{i=1}^{k+1}, f) | c_i \neq c_j].
\end{equation*}
\end{definition}

In contrastive learning, the feature extractor is trained as a pretrained model to be used on downstream tasks. There may be many different downstream tasks such as binary classification tasks with classes $\{c_1, c_2 \} \subseteq \mathcal{C}$ and many other multi-class classification tasks, so here we consider the average error of $L_{sup}^{\mu}(\mathcal{T},f)$ over all possible tasks $\mathcal{T}$ sampled from $\rho^{k+1}$ as the final performance measure of a feature extractor $f$.

\subsection{Contrastive Learning Algorithm}
We denote $k$ as the number of negative samples used for training and $(x,x^+)\!\!\sim\!\!\mathcal{D}_{sim}$, $(x_1^{-}, \dots , x_k^{-}\!)\!\!\sim\!\!\mathcal{D}_{neg}^k$.

\begin{definition}[Unsupervised Risk] \label{def-unsup-loss}
The unsupervised risk is defined as follows:
\begin{equation*}
    L_{un}(f) \coloneqq \mathbb{E}[\ell(\{f(x)^T(f(x^+)-f(x_i^-))\}_{i=1}^k)].
\end{equation*}
Given $M$ samples $\left\{(x_j,x_j^+,x_{j1}^-, \dots, x_{jk}^-)\right\}_{j=1}^M$ from $\mathcal{D}_{sim} \times \mathcal{D}_{neg}^k$, the empirical counterpart of unsupervised risk is defined as follows:
\begin{equation*}
    \widehat{L}_{un}(f) \coloneqq \frac{1}{M}\sum_{j=1}^M \ell(\{f(x_j)^T(f(x_j^+)-f(x_{ji}^-))\}_{i=1}^k).
\end{equation*}
\end{definition}

In the contrastive learning upstream task, we want to learn a feature extractor $f$ such that $f$ maps examples from the same class to similar features and makes the features of examples from different classes far away from each other. So for the examples $(x, x^+, x_1^-, \dots, x_k^-)$, we want $f(x)^T f(x^+)$ to be as large as possible while $f(x)^T f(x_i^-), i=1,\dots,k$ to be as small as possible. The loss $\ell$ we defined before elegantly captures the aim of contrastive learning, if we set $v_i = f(x)^T f(x^+) - f(x)^T f(x_i^-)$, minimizing $\ell(\{f(x)^T(f(x^+)-f(x_i^-))\}_{i=1}^k)$ will yield large $f(x)^T f(x^+)$ and small $f(x)^T f(x_i^-), i=1,\dots,k$. So optimizing over $L_{un}(f)$ can reach our goal for contrastive learning.

Following \cite{DBLP:conf/icml/SaunshiPAKK19}, the unsupervised risk can be described by the following equation:
\begin{equation*}
    L_{un}(f) = \underset{c^+,c_i^- \sim \rho^{k+1} }{\mathbb{E}}\  \underset{x,x^+ \sim \mathcal{D}_{c^+}^2, x_i^- \sim \mathcal{D}_{c_i^-}}{\mathbb{E}}[\ell(\{f(x)^T (f(x^+) - f(x_i^-))\}_{i=1}^k)].
\end{equation*}
The Empirical Risk Minimization (ERM) algorithm is used to find a function $\widehat{f}_{ERM} \!\! \in \!\! \underset{f \in \mathcal{F}}{\arg\min}\  \widehat{L}_{un}\!(f)$ that minimizes the empirical unsupervised risk. The function $\widehat{f}_{ERM}$ can be subsequently used for supervised linear classification tasks.

\subsection{Adversarial Setup}
A key question in adversarial contrastive learning is that of how to define the adversary sample in contrastive learning. From a representation perspective, one can find a point $\widetilde{x}$ that is close to $x$ and keeps the feature $f(\widetilde{x})$ both as far from $f(x^+)$ as possible and as close to the feature of some negative sample $f(x^-)$ as possible. Inspired by this intuition, we define the Contrastive Adversary Sample as follows:
\begin{definition}[$\mathcal{U}$-Contrastive Adversary Sample] \label{def-u-contra-adv}
Given a neighborhood $\mathcal{U}(x)$ of $x$, $x^+ \sim \mathcal{D}_{c^+}, x_i^- \sim \mathcal{D}_{c_i^-}$ for $i = 1,\dots,k$, we define the $\mathcal{U}$-Contrastive Adversary Sample of $x$ as follows:
\begin{equation*}
    \widetilde{x} \coloneqq \underset{x^\prime \in \mathcal{U}(x)}{\arg \sup}\  \underset{x^+ \sim \mathcal{D}_{c^+}, x_i^- \sim \mathcal{D}_{c_i^-}}{\mathbb{E}}\ \  [\ell(\{f(x^\prime)^T(f(x^+)-f(x_i^-))\}_{i=1}^k)].
\end{equation*}
\end{definition}

In this article, we suppose that the loss function $\ell$ is convex. By the subadditivity of $\sup$, we have, $\forall{f \in \mathcal{F}}$:
\begin{equation} \label{eq-attack-obj-leq-sur-obj}
    \begin{aligned}
        & \underset{x^\prime \in \mathcal{U}(x)}{\sup} \underset{x_i^- \sim \mathcal{D}_{c_i^-}}{\underset{x^+ \sim \mathcal{D}_{c^+}}{\mathbb{E}}}[\ell(\{f(x^\prime)^T\!(f(x^+)\!\!-\!\!f(x_i^-))\}_{i=1}^k)]\!\leq \!\!\underset{x_i^- \sim \mathcal{D}_{c_i^-}}{\underset{x^+ \sim \mathcal{D}_{c^+}}{\mathbb{E}}}\![\!\underset{x^\prime \in \mathcal{U}(x)}{\sup}\!\!\ell(\!\{f(x^\prime)^T\!(f(x^+)\!\!-\!\!f(x_i^-))\}_{i=1}^k\!)].
    \end{aligned}
\end{equation}

In the following sections, we analyze the theoretical properties of the right-hand side of \eqref{eq-attack-obj-leq-sur-obj}, which can be easily optimized in practice by Adversarial Empirical Risk Minimization (AERM).

The $\mathcal{U}$-Adversarial Unsupervised Risk and its surrogate risk can be defined as below. 

\begin{definition}[$\mathcal{U}$-Adversarial Unsupervised Risk] \label{def-u-adv-unsup-loss}
Given a neighborhood $\mathcal{U}(x)$ of $x$, the $\mathcal{U}$-Adversarial Unsupervised Risk of a presentation function $f$ is defined as follows:
\begin{equation} \label{eq-U-adv-Lun}
    \begin{aligned}
        \widetilde{L}_{un}(f) &\coloneqq \underset{\sim \rho^{k+1}}{\underset{c^+,c_i^- }{\mathbb{E}}} \underset{x \sim \mathcal{D}_{c^+}}{\mathbb{E}} [ \underset{x^\prime \in \mathcal{U}(x)}{\sup} \underset{x_i^- \sim \mathcal{D}_{c_i^-}}{\underset{x^+ \sim \mathcal{D}_{c^+}}{\mathbb{E}}}[\ell(\{f(x^\prime)^T(f(x^+)-f(x_i^-))\}_{i=1}^k)]].
    \end{aligned}
\end{equation}
Moreover, the surrogate risk of \eqref{eq-U-adv-Lun} is as follows:
\begin{equation*}
   \begin{aligned}
       \widetilde{L}_{sun}(f) &= \underset{\sim \rho^{k+1}}{\underset{c^+,c_i^- }{\mathbb{E}}} \underset{x_i^- \sim \mathcal{D}_{c_i^-}}{\underset{x,x^+ \sim \mathcal{D}_{c^+}^2}{\mathbb{E}}} \underset{x^\prime \in \mathcal{U}(x)}{\sup} \ell(\{f(x^\prime)^T(f(x^+)-f(x_i^-))\}_{i=1}^k).
   \end{aligned}
\end{equation*}
\end{definition}

By \eqref{eq-attack-obj-leq-sur-obj}, we have $\widetilde{L}_{un}(f) \leq \widetilde{L}_{sun}(f)$ for any $f \in \mathcal{F}$.
The Adversarial Supervised Risk of a classifier $g$ for a task $\mathcal{T}$ is defined as follows:
$$\widetilde{L}_{sup}(\mathcal{T},g) \coloneqq \underset{(x,c) \sim \mathcal{D}_\mathcal{T}}{\mathbb{E}} [\underset{x^\prime \in \mathcal{U}(x)}{\sup} \ell(\{g(x)_{c} - g(x)_{c^{\prime}}\}_{c^{\prime} \neq c})].$$
The Adversarial Supervised Risk of a representation function $f$ for a task $\mathcal{T}$ is defined as follows:
\begin{equation} \label{eq-adv-Lsup-T}
    \widetilde{L}_{sup}(\mathcal{T},f) \coloneqq \underset{W \in \mathbb{R}^{(k+1) \times n}}{\inf} \widetilde{L}_{sup}(\mathcal{T},W f).
\end{equation}
For the mean classifier, we define:
$$\widetilde{L}_{sup}^\mu(\mathcal{T},f) \coloneqq \widetilde{L}_{sup}(\mathcal{T},W^\mu f)$$

The Average Adversarial Supervised Risk for a representation function is as defined below.

\begin{definition}[Average $\mathcal{U}$-Adversarial Supervised Risk] \label{def-adv-sup-loss}
\begin{equation*}
    \widetilde{L}_{sup}(f) \coloneqq \underset{\{c_i\}_{i=1}^{k+1} \sim \rho^{k+1}}{\mathbb{E}} [\widetilde{L}_{sup}(\{c_i\}_{i=1}^{k+1}, f) | c_i \neq c_j].
\end{equation*}
For the mean classifier, we have the following:
\begin{equation*}
    \widetilde{L}_{sup}^{\mu}(f) \coloneqq \underset{\{c_i\}_{i=1}^{k+1} \sim \rho^{k+1}}{\mathbb{E}} [\widetilde{L}_{sup}^{\mu}(\{c_i\}_{i=1}^{k+1}, f) | c_i \neq c_j].
\end{equation*}
\end{definition}

\section{Theoretical Analysis for Adversarial Contrastive Learning} \label{sec-theoretical-ana}
This section presents some theoretical results for adversarial contrastive learning.

\subsection{One Negative Sample Case} \label{subsec-one-negative-sample-case}
Let $\tau = \underset{c,c^\prime \sim \rho^2}{\mathbb{P}} [c=c^\prime]$, $\bm{\sigma}$ be an $M$-dimensional Rademacher random vector with i.i.d. entries, define $(g_f)_{|\mathcal{S}} = \left(g_f(z_1), \dots, g_f(z_M)\right)$ and $\mathcal{R}_\mathcal{S}(\mathcal{G}) \coloneqq \underset{\bm{\sigma} \sim \{\pm 1\}^M}{\mathbb{E}} \left[\underset{f \in \mathcal{F}}{\sup} \left< \bm{\sigma},(g_f)_{|\mathcal{S}}\right>\right]$ where $\mathcal{G} \coloneqq \{ g_f(x,x^+,x_1^-,\dots,x_k^-) = \underset{x^\prime \in \mathcal{U}(x)}{\sup} \ell\left(\{f(x^\prime)^T(f(x^+)-f(x_i^-))\}_{i=1}^k\right) | f \in \mathcal{F}\}$, let $\widehat{f} \in  \underset{f \in \mathcal{F}}{\arg\min}\  \widehat{\widetilde{L}}_{sun}(f)$ where $\widehat{\widetilde{L}}_{sun}(f)$ is the empirical counterpart of $\widetilde{L}_{sun}(f)$, we have:

\begin{theorem}[The proof can be found in the Appendix \ref{apd-prf-thm-2-adv-Lsup-mu-leq-adv-Lsun}] \label{thm-2-adv-Lsup-mu-leq-adv-Lsun}
Let $\ell: \mathbb{R}^k \xrightarrow[]{} \mathbb{R}$ be bounded by $B$. Then, for any $\delta \in (0,1)$,with a probability of at least $1 - \delta$ over the choice of the training set $\mathcal{S}=\{(x_j,x_j^+,x_{j}^-)\}_{j=1}^M$, for any $f \in \mathcal{F}$:
\begin{equation*}
    \widetilde{L}_{sup}(\widehat{f}) \leq \widetilde{L}_{sup}^{\mu}(\widehat{f}) \leq \frac{1}{1-\tau} (\widetilde{L}_{sun}(f) - \tau \ell (0)) + \frac{1}{1-\tau} AG_M,
\end{equation*}
where
\begin{equation} \label{eq-AG}
    AG_M = O(\frac{\mathcal{R}_\mathcal{S}(\mathcal{G})}{M} + B \sqrt{\frac{log \frac{1}{\delta}}{M}}).
\end{equation}
\end{theorem}

\begin{remark}
    Theorem \ref{thm-2-adv-Lsup-mu-leq-adv-Lsun} shows that when the hypothesis class $\mathcal{F}$ is rich enough to contain some $f$ with low surrogate adversarial unsupervised risk, the empirical minimizer of the surrogate adversarial unsupervised risk will then obtain good robustness on the supervised downstream task.
    
    Note that we can take $f = \widehat{f}$ in the upper bound of Theorem \ref{thm-2-adv-Lsup-mu-leq-adv-Lsun} and get a bound $\widetilde{L}_{sup}(\widehat{f}) \leq \frac{1}{1-\tau} (\widetilde{L}_{sun}(\widehat{f}) - \tau \ell (0)) + \frac{1}{1-\tau} AG_M$. Then we can see that if the output of AERM (i.e., $\widehat{f}$) achieves small unsupervised adversarial risk, then $\widehat{f}$ is a robust feature extractor such that it can achieve good robustness after fine-tuning on the downstream tasks.
    
    Theorem \ref{thm-2-adv-Lsup-mu-leq-adv-Lsun} gives a relationship between the robustness of the contrastive (upstream) task and the robustness of the downstream classification tasks and explains why adversarial contrastive learning can help improve the robustness of the downstream task, as shown empirically in \citet{DBLP:conf/nips/KimTH20, DBLP:conf/nips/HoN20, DBLP:conf/nips/JiangCCW20}.
\end{remark}

\subsection{Blocks of Similar Points}\label{subsec-blocks-of-similar-points}
\citet{DBLP:conf/icml/SaunshiPAKK19} show a refined method that operates by using blocks of similar data and determine that the method achieves promising performance both theoretically and empirically. We adapt this method to adversarial contrastive learning.

Specifically, we sample $(b+1)$ i.i.d. similar samples $x, x_1^+, \dots, x_b^+$ from $c^+ \sim \rho$ and $b$ negative i.i.d. samples from $c^- \sim \rho$. The block adversarial contrastive learning risk is as follows:
\begin{equation*}
   \widetilde{L}_{sun}^{block}(f) \coloneqq \mathbb{E} \left[\underset{x^\prime \in \mathcal{U}(x)}{\sup}\ell\left(f(x^\prime)^T\left(\frac{\sum_{i=1}^b f(x_{i}^+)}{b} - \frac{\sum_{i=1}^b f(x_{i}^-)}{b}\right) \right)\right],
\end{equation*}
and its empirical counterpart is as follows:
\begin{equation*}
    \widehat{\widetilde{L}}_{sun}^{block}(f) \coloneqq \frac{1}{M} \sum_{i=1}^M \left[\underset{x^\prime \in \mathcal{U}(x_i)}{\sup}\ell\left(f(x^\prime)^T\left(\frac{\sum_{j=1}^b f(x_{ij}^+)}{b} - \frac{\sum_{j=1}^b f(x_{ij}^-)}{b}\right) \right)\right],
\end{equation*}

\begin{theorem}[The proof can be found in the Appendix \ref{apd-prf-thm-adv-sup-leq-adv-sun-block}] \label{thm-adv-sup-leq-adv-sun-block}
For any $f \in \mathcal{F}$, we have:
\begin{equation*}
    \begin{aligned}
        \widetilde{L}_{sup}(f) &\leq \frac{1}{1-\tau} \left(\widetilde{L}_{sun}^{block}(f) - \tau \ell(0)\right) \leq \frac{1}{1-\tau} \left(\widetilde{L}_{sun}(f) - \tau \ell(0)\right).
    \end{aligned}
\end{equation*}
\end{theorem}
\begin{remark} \label{rmk-of-thm-adv-sup-leq-adv-sun-block}
   Theorem \ref{thm-adv-sup-leq-adv-sun-block} shows that using blocks of similar data yields a tighter upper bound for the adversarial supervised risk than in the case for pairs of similar data. Theorem \ref{thm-adv-sup-leq-adv-sun-block} implies that using the blocks in adversarial contrastive learning may improve the robust performance of the downstream tasks; this will be verified by the empirical results in \S \ref{sec-exp}.
\end{remark}

\subsection{Multiple Negative Sample Case} \label{subsec-multi-negative-sample-case}
This subsection extends our results to $k$ negative samples. To achieve this, more definitions are required. Let $[k]$ denote the set $\{1,2, \dots, k\}$. 
\begin{definition} \label{def-task-distribution}
    We define a distribution $\mathcal{D}$ over the supervised tasks as follows: First, sample $k+1$ classes (allow repetition) $c^+, c_1^-, \dots, c_k^- \sim \rho^{k+1}$, conditioned on the event that $c_i^- \neq c^+, \forall{i \in [k]}$. Then, set the task $\mathcal{T}$ as the set of distinct classes in $\{c^+, c_1^-, \dots, c_k^-\}$.
\end{definition}
\begin{definition} \label{def-avg-adv-sup-loss}
    The Average $\mathcal{U}$-Adversarial Supervised Risk of a representation function $f \in \mathcal{F}$ over $\mathcal{D}$ is defined as follows:
    \begin{equation*}
        \widetilde{\mathcal{L}}_{sup}(f) \coloneqq \underset{\mathcal{T} \sim \mathcal{D}}{\mathbb{E}}\left[\widetilde{L}_{sup}(\mathcal{T}, f)\right].
    \end{equation*}
\end{definition}
Let $E_{distinct}$ be the event such that $\{c^+,c_1^-,\dots,c_k^-\}$ is distinct and $p = \underset{(c^+,c_1^-,\dots,c_k^-) \sim \mathcal{D}}{\mathbb{P}}[E_{distinct}]$. For any $f \in \mathcal{F}$, we have (The proof can be found in the Appendix \ref{apd-prf-ieq-cLsup-leq-Lsup}):
\begin{equation} \label{ieq-cLsup-leq-Lsup}
    \widetilde{L}_{sup}(f) \leq \frac{\widetilde{\mathcal{L}}_{sup}(f)}{p}.
\end{equation}

From \eqref{ieq-cLsup-leq-Lsup}, we can turn to analyze the relation between $\widetilde{\mathcal{L}}_{sup}(f)$ and $\widetilde{L}_{sun}(f)$ in the multiple negative sample case. Our Theorem \ref{thm-of-thm-k-general-bound} handles $\widetilde{\mathcal{L}}_{sup}(f)$ instead of $\widetilde{L}_{sup}(f)$.

\begin{assumption} \label{ass-loss-function}
    Assume that $\forall{I_1, I_2 \subseteq [d]}$ such that $I_1 \cup I_2 = [d]$, $\ell$ satisfies the following inequations:
    \begin{equation} \label{eq-loss-property-1}
        \ell(\{v_i\}_{i \in I_1}) \leq \ell(\{v_i\}_{i \in [d]}) \leq \ell(\{v_i\}_{i \in I_1}) + \ell(\{v_i\}_{i \in I_2}),
    \end{equation}
    \vskip -0.2in
    \begin{equation} \label{eq-loss-property-2}
        \ell(\{v_i\}_{i \in I_2}) \leq \ell(\{v_i\}_{i \in [d]}) \leq \ell(\{v_i\}_{i \in I_1}) + \ell(\{v_i\}_{i \in I_2}).
    \end{equation}
\end{assumption}

\begin{proposition}[The proof can be found in the Appendix \ref{apd-prf-prop-loss-function}] \label{prop-loss-function}
    The hinge loss and the logistic loss satisfy Assumption \ref{ass-loss-function}.
\end{proposition}

If $\mathcal{C}$ is finite, we obtain a simple (and informal) bound of $\widetilde{\mathcal{L}}_{sup}(\widehat{f})$, for the more complex case that allows infinite $\mathcal{C}$ and the formal form of Theorem \ref{thm-of-thm-k-general-bound} (Theorem \ref{thm-k-general-bound}), please refer to the Appendix \ref{apd-prf-thm-k-general-bound}.

\begin{theorem}[The proof can be found in the Appendix \ref{apd-prf-thm-of-thm-k-general-bound}] \label{thm-of-thm-k-general-bound}
    Suppose $\mathcal{C}$ is finite, for any $c \in \mathcal{C}$, $\rho(c) > 0$, and $\ell$ satisfies Assumption \ref{ass-loss-function}. Then, with a probability of at least $1-\delta$ over the choice of the training set $\mathcal{S}$, $\forall{f \in \mathcal{F}}$:
    \begin{equation*}
        \widetilde{\mathcal{L}}_{sup}(\widehat{f}) \leq \alpha(\rho)\left(\widetilde{L}_{sun}(f) + AG_M\right) - \beta.
    \end{equation*}
\end{theorem}

\section{Generalization Bounds for Example Hypothesis Classes} \label{sec-bin}
This section presents a concrete analysis of the Rademacher complexity for linear hypothesis class and multi-layer neural networks based on covering number {\citep[Definition 5.1]{wainwright2019high}}.
We first introduce some definitions and required lemmas.

\begin{lemma}[{\citealp[Lemma 5.7]{wainwright2019high}}, volume ratios and metric entropy] \label{lma-hds-lma5.7}
Consider a pair of norms $\Vert \cdot \Vert$ and $\Vert \cdot \Vert^\prime$ on $\mathbb{R}^d$, and let $\mathbb{B}$ and $\mathbb{B}^\prime$ be their corresponding unit balls (i.e. $\mathbb{B} = \{\theta \in \mathbb{R}^d | \Vert \theta \Vert \leq 1 \}$, with $\mathbb{B}^\prime$ similarly defined). The $\delta$-covering number of $\mathbb{B}$ in the $\Vert \cdot \Vert^\prime$-norm then obeys the following bounds \citep[Lemma 5.7]{wainwright2019high}:
$$
\left( \frac{1}{\delta} \right)^d \frac{vol(\mathbb{B})}{vol(\mathbb{B}^\prime)} \leq \mathcal{N}(\delta;\mathbb{B},\Vert \cdot \Vert^\prime) \leq \frac{vol(\frac{2}{\delta}\mathbb{B}+\mathbb{B}^\prime)}{vol(\mathbb{B}^\prime)},
$$
where we define the Minkowski sum $A+B \coloneqq \{ a + b : a \in A, b \in B \}$, $vol(\mathbb{B}) \coloneqq \int \mathbbm{1} \{ x \in \mathbb{B}\} dx$ is the volume of $\mathbb{B}$ based on the Lebesgue measure, and $\mathcal{N}(\delta;\mathbb{B},\Vert \cdot \Vert^\prime)$ is the $\delta$-covering number of $\mathbb{B}$ with respect to the norm $\Vert \cdot \Vert^\prime$.
\end{lemma}

\begin{lemma}[The proof can be found in the Appendix \ref{apd-prf-lma-unit-ball-covering-number-bound}] \label{lma-unit-ball-covering-number-bound}
Let $\mathbb{B}_p(r)$ be the $p$-norm ball in $\mathbb{R}^d$ with radius $r$. The $\delta$-covering number of $\mathbb{B}_p(r)$ with respect to $\Vert \cdot \Vert_p$ thus obeys the following bound:
$$
\mathcal{N}(\delta;\mathbb{B}_p(r),\Vert \cdot \Vert_p) \leq \left( 1 + \frac{2r}{\delta} \right)^d.
$$
\end{lemma}

\begin{definition}[{\citealp[Definition 5.16]{wainwright2019high}}, sub-Gaussian process] \label{def-sub-gaussian-process}
A collection of zero-mean random variables $\{X_\theta, \theta \in \mathbb{T}\}$ is a sub-Gaussian process with respect to a metric $\rho_X$ on $\mathbb{T}$ if:
$$
\mathbb{E}[e^{\lambda(X_\theta-X_{\widetilde{\theta}})}] \leq e^{\frac{\lambda^2 \rho_X^2(\theta,\widetilde{\theta})}{2}}, \forall{\theta,\widetilde{\theta} \in \mathbb{T}, \lambda \in \mathbb{R}}.
$$
\end{definition}
It is easy to prove that the Rademacher process {\citep[$\S 5.2$]{wainwright2019high}} satisfies the condition in Definition \ref{def-sub-gaussian-process} with respect to the $\ell_2$-norm.

\begin{lemma}[{\citealp[the Dudley's entropy integral bound]{wainwright2019high}}] \label{lma-dudley-int}
Let $\{X_\theta, \theta \in \mathbb{T}\}$ be a zero-mean sub-Gaussian process with respect to the induced pseudometric $\rho_X$ from Definition \ref{def-sub-gaussian-process}. Then, for any $\delta \in [0,D]$, we have:
\begin{equation} \label{ieq-dudley-int}
    \begin{aligned}
    \mathbb{E}\left[ \underset{\theta,\widetilde{\theta} \in \mathbb{T}}{\sup} (X_\theta - X_{\widetilde{\theta}}) \right] &\leq 2 \mathbb{E} \left[ \underset{\underset{\rho_X(\gamma,\gamma^\prime)\leq \delta}{\gamma,\gamma^\prime \in \mathbb{T}}}{\sup}(X_\gamma - X_{\gamma^\prime}) \right] + 32 \mathcal{J}(\delta/4; D).
    \end{aligned}
\end{equation}
Here, $D = \sup_{\theta,\widetilde{\theta} \in \mathbb{T}} \ \rho_X(\theta, \widetilde{\theta})$ and $\mathcal{J}(a; b) = \int_a^b \sqrt{ln \mathcal{N}_X(u; \mathbb{T})}du$, where $\mathcal{N}_X(u; \mathbb{T})$ is the $u$-covering number of $\mathbb{T}$ in the $\rho_X$-metric.
\end{lemma}
\begin{remark}
Given $\theta_0 \in \mathbb{T}$, since $\mathbb{E}\left[ X_{\theta_0} \right] \coloneqq \mathbb{E}\left[ \left<\theta_0,\bm{\sigma}\right> \right]=0$, we have:
\begin{equation} \label{ieq-dudley-remark}
    \mathbb{E}\left[ \underset{\theta \in \mathbb{T}}{\sup} X_\theta \right]=\mathbb{E} \left[ \underset{\theta \in \mathbb{T}}{\sup} (X_\theta - X_{\theta_0}) \right]\leq\mathbb{E}\left[ \underset{\theta,\widetilde{\theta} \in \mathbb{T}}{\sup} (X_\theta - X_{\widetilde{\theta}}) \right].
\end{equation}

Combining \eqref{ieq-dudley-int} with \eqref{ieq-dudley-remark}, we have:
\begin{equation*}
	\mathbb{E} \left[ \underset{\theta \in \mathbb{T}}{\sup} X_\theta \right] \leq 2 \mathbb{E} \left[ \underset{\underset{\rho_X(\gamma,\gamma^\prime)\leq \delta}{\gamma,\gamma^\prime \in \mathbb{T}}}{\sup}(X_\gamma - X_{\gamma^\prime}) \right]+ 32 \mathcal{J}(\delta/4; D),
\end{equation*}
which can be used to draw the upper bound of $\mathcal{R}_\mathcal{S}(\mathcal{G})$ by establishing a connection between the Rademacher process and the Rademacher complexity of the hypothesis classes when proper norm is chosen. For more details, please refer to the Appendix, details are in the proof of Theorem \ref{thm-linear-Rs(H)-bound}, Theorem \ref{thm-nn-Rs(H)-bound-Fnorm} and Theorem \ref{thm-nn-Rs(H)-bound-1,infnorm}.

\end{remark}

To make Theorem \ref{thm-2-adv-Lsup-mu-leq-adv-Lsun} and Theorem \ref{thm-of-thm-k-general-bound} concrete, we need to upper bound $\mathcal{R}_\mathcal{S}(\mathcal{G})$. Assume that loss $\ell$ is a non-increasing function; for example, hinge loss and logistic loss satisfy this assumption. Let $\mathcal{G} = \{ g_f(x,x^+,x^-) = \underset{x^\prime \in \mathcal{U}(x)}{\sup} \ell \left( f(x^\prime)^T\left( f(x^+) - f(x^-) \right) \right)|f \in \mathcal{F}\}$. Since $\ell$ is non-increasing, we have:
$$
\mathcal{G} = \left\{ \ell \left( \underset{x^\prime \in \mathcal{U}(x)}{\min} \left( f(x^\prime)^T\left( f(x^+) - f(x^-) \right) \right) \right) |f \in \mathcal{F}\right\}.
$$
Let $\mathcal{H} = \left\{\underset{x^\prime \in \mathcal{U}(x)}{\min} \left( f(x^\prime)^T\left( f(x^+) - f(x^-) \right) \right) |f \in \mathcal{F} \right\}$. Suppose $\ell$ is $\eta$-Lipschitz. By the Ledoux-Talagrand contraction inequality \citep{ledoux2013probability}, we have:
\begin{equation} \label{ieq-contraction}
    \mathcal{R}_\mathcal{S}(\mathcal{G}) \leq \eta \mathcal{R}_\mathcal{S}(\mathcal{H}).
\end{equation}

Thus, we only need to upper bound $\mathcal{R}_\mathcal{S}(\mathcal{H})$. Let $\normvec{A}_{a,b}$ be the $\ell_b$-norm of the $\ell_a$-norm of the rows of $A$. Consider the training set $\mathcal{S}=\{(x_i,x_i^+,x_i^-)\}_{i=1}^M$. Let $X$ be a matrix whose $i$th row is $x_i$. We define $X^+$ and $X^-$ in a similar way. It is easy to see that $\forall{p \ge 1},\forall{i=1,\dots,M}, \normvec{x_i}_p \leq \normvec{X}_{p,\infty}$.

\subsection{Linear Hypothesis Class} \label{subsec-linear-class}
Let $\mathcal{F}\! =\! \{f\!: x \xrightarrow[]{} \!Wx | W \in \mathbb{R}^{n\times m},\normmm{W}_p \leq w\}$. To simplify notations, $\forall{p} \ge 1$ and $\frac{1}{p}+\frac{1}{p^*}=1$, let
\begin{equation} \label{eq-def-P}
    P\!=\!\max\!\left\{ \!\normvec{X}_{p,\infty}, \normvec{X^+}_{p,\infty}, \normvec{X^-}_{p,\infty}\! \right\}, P^*\!=\!\max\!\left\{\! \normvec{X}_{p^*,\infty}, \normvec{X^+}_{p^*,\infty}, \normvec{X^-}_{p^*,\infty} \!\right\}.
\end{equation}
We then have $\forall{p \ge 1}$,
\begin{equation*}
    \forall{i=1,\dots,M}, \normvec{x_i}_p, \normvec{x_i^+}_p, \normvec{x_i^-}_p \leq P.
\end{equation*}
We now present the upper bound of $\mathcal{R}_\mathcal{S}(\mathcal{H})$ under $\normvec{\cdot}_r$ attack.
\begin{theorem}[$\mathcal{R}_\mathcal{S}(\mathcal{H})$ under $\normvec{\cdot}_r$ attack for linear models] \label{thm-linear-Rs(H)-bound}
Consider the $\ell_r$ attack, i.e. let $\mathcal{U}(x)=\left\{ x^\prime | \normvec{x^\prime - x}_r \leq \epsilon\right\}$. We then have:
\begin{equation*}
    \begin{aligned}
        &\mathcal{R}_\mathcal{S}(\mathcal{H}) = O\left( \left[ PP^* + \epsilon R^* s(r^*,p,m) \right]\left[ m s(p^*,p,m) w^2 \sqrt{M} \right]  \right),
    \end{aligned}
\end{equation*}
where $s(p,q,n) \coloneqq n^{\max\left\{\frac{1}{p}-\frac{1}{q},\frac{1}{q}-\frac{1}{p} \right\}}$, $\frac{1}{p} + \frac{1}{p^*} = 1,\frac{1}{r} + \frac{1}{r^*} = 1$, and $R^*$ is defined similarly to \eqref{eq-def-P}.
\end{theorem}
The proof can be found in the Appendix \ref{apd-prf-thm-linear-Rs(H)-bound}.

\begin{remark}

Combining Theorem \ref{thm-linear-Rs(H)-bound} with \eqref{eq-AG}, we have:
\begin{equation} \label{eq-linear-AG}
    \begin{aligned}
       & AG_M =  O\left(\frac{\left[ PP^*+\epsilon R^* s(r^*,p,m) \right] m \eta s(p^*,p,m) w^2+B \sqrt{log \frac{1}{\delta}}}{\sqrt{M}} \right).
    \end{aligned}
\end{equation}
\end{remark}

\subsection{Multi-layer Neural Network} \label{subsec-nn}
In this section, we analyze fully connected multi-layer neural networks.

Suppose that $\mathcal{X} \subseteq \mathbb{R}^m$. Let $\mathcal{F}=\left\{W_d \sigma(W_{d-1} \sigma(\cdots \sigma(W_1 x)))\  | \  \normmm{W_l} \leq M_l, l=1,\dots, d  \right\}$, where $\normmm{\cdot}$ is the norm of the matrix and $\sigma(\cdot)$ is an elementwise $L$-Lipschitz function with $\sigma(0)=0$ and $W_l \in \mathbb{R}^{h_l \times h_{l-1}}$, where $h_d=n, h_0=m$.
Assume $\ell$ is $\eta$-Lipschitz and non-increasing. From \eqref{ieq-contraction}, we need only to bound the Rademacher complexity of $\mathcal{H}$.

We here consider two cases of the matrix norm $\normmm{\cdot}$. Let $\mathcal{U}(x)=\left\{ x^\prime | \normvec{x^\prime - x}_p \leq \epsilon \right\}$ for some $p \ge 1$.
\subsubsection{Frobenius-norm Case.} \label{sbsbsec-Fnorm}
We first consider using the Frobenius-Norm in the definition of the multi-layer neural networks hypothesis class $\mathcal{F}$.
\begin{theorem}[$\mathcal{R}_\mathcal{S}(\mathcal{H})$ under $\normvec{\cdot}_p$ attack for NNs under $\normmm{\cdot}_F$ constraint] \label{thm-nn-Rs(H)-bound-Fnorm}
Let $\mathcal{U}(x)=\left\{ x^\prime | \normvec{x^\prime - x}_p \leq \epsilon \right\}$ (i.e. consider the $\ell_p$ attack), $\sigma(0)=0$ with Lipschitz constant $L$ and let $\mathcal{F}=\left\{W_d \sigma(W_{d-1} \sigma(\cdots \sigma(W_1 x)))\big| \normmm{W_l}_F\leq M_l^F, l=1,\dots, d \right\}$. We then have:
\begin{equation} \label{eq-nn-AG}
    \mathcal{R}_\mathcal{S}(\mathcal{H}) = O\left( \sqrt{\sum_{l=1}^d h_l h_{l-1}} K \sqrt{d} \sqrt{M} \right),
\end{equation}
where
\begin{equation*}
    K = 2 B_{X,\epsilon}^F \cdot \left( B_{X^+}^F + B_{X^-}^F \right),
\end{equation*}
where
\begin{equation*}
    \begin{aligned}
       B_{X,\epsilon}^F \!&=\!L^{d-1} \!\prod_{l=1}^d \!M_l^F \!\max\left\{ 1, m^{\frac{1}{2}-\frac{1}{p}} \right\}\left( \normvec{X}_{p, \infty}\!+\!\epsilon \right), B_X^F\!=\!L^{d-1}\!\prod_{l=1}^d \!M_l^F \!\max\left\{ 1, m^{\frac{1}{2}-\frac{1}{p}} \right\} \!\normvec{X}_{p, \infty}.
    \end{aligned}
\end{equation*}
\end{theorem}
The proof can be found in the Appendix \ref{apd-prf-thm-nn-Rs(H)-bound-Fnorm}.

\begin{remark}
Combining Theorem \ref{thm-nn-Rs(H)-bound-Fnorm} with \eqref{eq-AG}, we have:
\begin{equation*}
    AG_M \!=\! O\left( K \eta \sqrt{\frac{d \sum_{l=1}^d h_l h_{l-1}}{M}} + B \sqrt{\frac{log \frac{1}{\delta}}{M}} \right).
\end{equation*}
\end{remark}

\subsubsection{$\ell_{1, \infty}$-norm Case} \label{sbsbsec-1,inf-norm}
We consider the $\normvec{\cdot}_{1, \infty}$ norm constraint.
\begin{theorem}[$\mathcal{R}_\mathcal{S}(\mathcal{H})$ under $\normvec{\cdot}_p$ attack for NNs under $\normvec{\cdot}_{1, \infty}$ constraint] \label{thm-nn-Rs(H)-bound-1,infnorm}
Let $\mathcal{U}(x)=\left\{ x^\prime | \normvec{x^\prime - x}_p \leq \epsilon \right\}$ (i.e. consider the $\ell_p$ attack), $\sigma(0)=0$ with Lipschitz constant $L$; moreover, let $\mathcal{F}=\left\{\!W_d \sigma(\!W_{d-1} \sigma(\!\cdots\!\sigma(\!W_1 x)))\big|\normvec{W_l}_{1, \infty}\!\!\leq\!\!M_l^{1, \infty}, l\!\!=\!\!1\!,\dots,\!d\! \right\}$. We then have:
\begin{equation*}
    \mathcal{R}_\mathcal{S}(\mathcal{H}) = O\left( \sqrt{\sum_{l=1}^d h_l h_{l-1}} \sqrt{d K_0 K_1} \sqrt{M} \right),
\end{equation*}
where
\begin{equation*}
    \begin{aligned}
        K_0 &= 2 B_{X,\epsilon}^{1, \infty} \cdot \left( B_{X^+}^\prime + B_{X^-}^\prime \right), \ K_1 = \frac{K_0}{2} + B_{X,\epsilon}^\prime \cdot \left( B_{X^+}^{1, \infty} + B_{X^-}^{1, \infty} \right),
    \end{aligned}
\end{equation*}
where
\begin{equation*}
    \begin{aligned}
       B_{X,\epsilon}^\prime &= L^{d-1} \prod_{l=1}^d h_l M_l^{1, \infty} \ m^{1-\frac{1}{p}}\left( \normvec{X}_{p, \infty} + \epsilon \right), \ \ B_X^\prime = L^{d-1} \prod_{l=1}^d h_l M_l^{1, \infty} \ m^{1-\frac{1}{p}} \normvec{X}_{p, \infty}, \\[0.2mm]
       B_{X,\epsilon}^{1, \infty} &= L^{d-1} \prod_{l=1}^d M_l^{1, \infty} \ \left( \normvec{X}_{p, \infty} + \epsilon \right), \ \ B_X^{1, \infty} = L^{d-1} \prod_{l=1}^d M_l^{1, \infty} \ \normvec{X}_{p, \infty}.
    \end{aligned}
\end{equation*}
\end{theorem}
The proof can be found in the Appendix \ref{apd-prf-thm-nn-Rs(H)-bound-1,infnorm}.

\begin{remark}
Combining Theorem \ref{thm-nn-Rs(H)-bound-1,infnorm}  with \eqref{eq-AG}, we have:
\begin{equation*}
    \setlength{\belowdisplayskip}{1pt}
    AG_M = O\left( \eta \sqrt{\frac{d K_0 K_1 \sum_{l=1}^d h_l h_{l-1}}{M}} + B \sqrt{\frac{log \frac{1}{\delta}}{M}} \right).
    \setlength{\belowdisplayskip}{1pt}
\end{equation*}
\end{remark}

\begin{remark} \label{rmk-of-theory-part}
Our bound has important implications for the design of regularizers for adversarial contrastive learning. To achieve superior robust performance on the downstream tasks, the usual approach is to make $\normvec{X}_{p, \infty}$ small. For example, Pytorch scales the images to tensors with entries within the range $[0,1]$. Moreover, Theorem \ref{thm-nn-Rs(H)-bound-Fnorm} shows that we can take the norms of the layers as the regularizers to reduce the adversarial supervised risk.
\end{remark}

In our analysis for the Rademacher complexity, we consider models with norm-constrained weights, which means that the hypothesis class is uniformly Lipschitz, although the Lipschitz constant may be large (the product of maximal weight norms for the layers). One may wonder what will happen if we remove the constrains on the norm of the weights. For simplicity, let's consider a hypothesis class $\mathcal{H} \subseteq \{\pm 1\}^\mathcal{X}$ for binary classification, we have:
\begin{equation*}
    \mathcal{R}_\mathcal{S}(\mathcal{H}) = \underset{\pmb{\sigma}}{\mathbb{E}} \left[ \underset{h \in \mathcal{H}}{\sup} \frac{1}{n} \sum_{i=1}^n \sigma_i h(x_i) \right],
\end{equation*}
where $\sigma_1, \dots, \sigma_n \in \{ \pm 1 \}$ are i.i.d. uniform random variables. We can regard $\sigma_1, \dots, \sigma_n$ as random labels that we need to fit by hypothesis from $\mathcal{H}$, so we can interpret $\mathcal{R}_\mathcal{S}(\mathcal{H})$ as the ability of $\mathcal{H}$ to fit random $\pm 1$ binary labels. Now let $\mathcal{H}$ be the neural network, if we do not constrain the norm of the weights, theoretically, the universal approximation theorem \citep{DBLP:journals/ijon/MaiorovP99} tells us that neural networks can fit any continuous function on a bounded input space, which means that in this case $\mathcal{R}_\mathcal{S}(\mathcal{H}) \approx 1$, leading to vacuous bounds in the binary classification case; experimentally, \citet{DBLP:conf/iclr/ZhangBHRV17} show that deep neural networks easily fit random labels.

From another perspective, to derivate an upper bound for $\mathcal{R}_\mathcal{S}(\mathcal{H})$ by covering number, we need to find a $\delta$-covering set for $\mathcal{H}$ under some metric $\rho(\cdot, \cdot)$. If the weights of the layers are not bounded, we can not cover $\mathcal{H}$ by a finite subset of $\mathcal{H}$, the $\delta$-covering number of $\mathcal{H}$ under $\rho(\cdot, \cdot)$ will be infinite, which means that Lemma \ref{lma-dudley-int} does not hold. So it is difficult to go beyond the Lipschitz network.

\section{Experiments} \label{sec-exp}
\begin{table}[t]
\centering
\setlength{\tabcolsep}{4mm}{
\begin{tabular}{ccccccc}
\hline \hline
\multirow{2}{*}{\textsc{Attack}} & \multirow{2}{*}{$\epsilon$} & \multirow{2}{*}{\textsc{Type}} & \multicolumn{4}{c}{$\lambda$}         \\ \cline{4-7}
                        &                          &                       & 0 & 0.002 & 0.05 & 0.2  \\ \hline
\multirow{4}{*}{PGD}    & \multirow{2}{*}{0.01}       & \textsc{Clean}                 & $\bm{75.73}$   & $74.8$  & $74.5$  & $75.67$  \\ \cline{3-7}
                        &                          & \textsc{Adv}                   & $67.67$  & $\bm{69.25}$ & $68.59$  & $67.65$ \\ \cline{2-7}
                        & \multirow{2}{*}{0.02}       & \textsc{Clean}                 & $53.11$  & $55.72$ & $\bm{55.73}$  & $55.72$ \\ \cline{3-7}
                        &                          & \textsc{Adv}                    & $46.71$  &$\bm{48.17}$ & $\bm{48.17}$ & $48.16$ \\ \hline
\multirow{4}{*}{FGSM}   & \multirow{2}{*}{0.01}       & \textsc{Clean}                 & $74.42$  & $\bm{76.13}$ & $76.12$ & $76.11$ \\ \cline{3-7}
                        &                          & \textsc{Adv}                    & $67.28$  & $68.79$ & $\bm{68.8}$ & $\bm{68.8}$ \\ \cline{2-7}
                        & \multirow{2}{*}{0.02}       & \textsc{Clean}                 & $54.28$  & $54.29$ & $54.28$ & $\bm{66.68}$ \\ \cline{3-7}
                        &                          & \textsc{Adv}                    & $45.93$ & $45.94$ & $45.92$ & $\bm{55.64}$ \\ \hline \hline
\end{tabular}}
\caption{Results of experiments on the regularizer. In this table, we list the clean accuracy (\textsc{Clean}) and adversarial accuracy (\textsc{Adv}) of the mean classifier under the PGD and FGSM attack with $\epsilon=0.01$ and $\epsilon=0.02$. $\lambda$ is chosen from $\{0, 0.002, 0.05, 0.2\}$, and $\lambda=0$ indicates no regularizer.} \label{table-regu}
\end{table}

\begin{figure}
    \centering    
    \subfigure[Influence on clean accuracy] {
        \label{fig:a}     
        \includegraphics[scale=0.4]{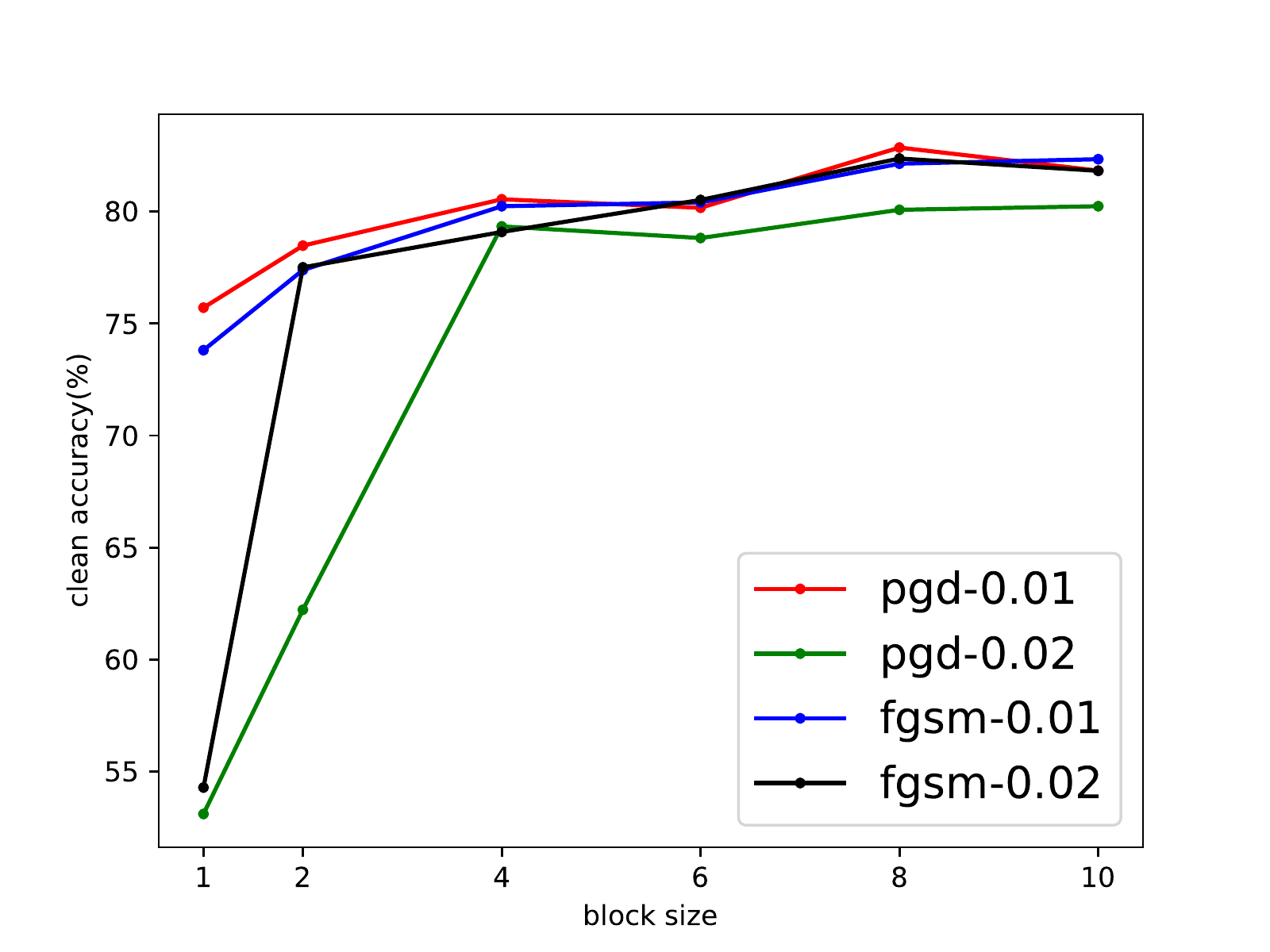}  
    }
    \subfigure[Influence on adversarial accuracy] { 
        \label{fig:b} 
        \includegraphics[scale=0.4]{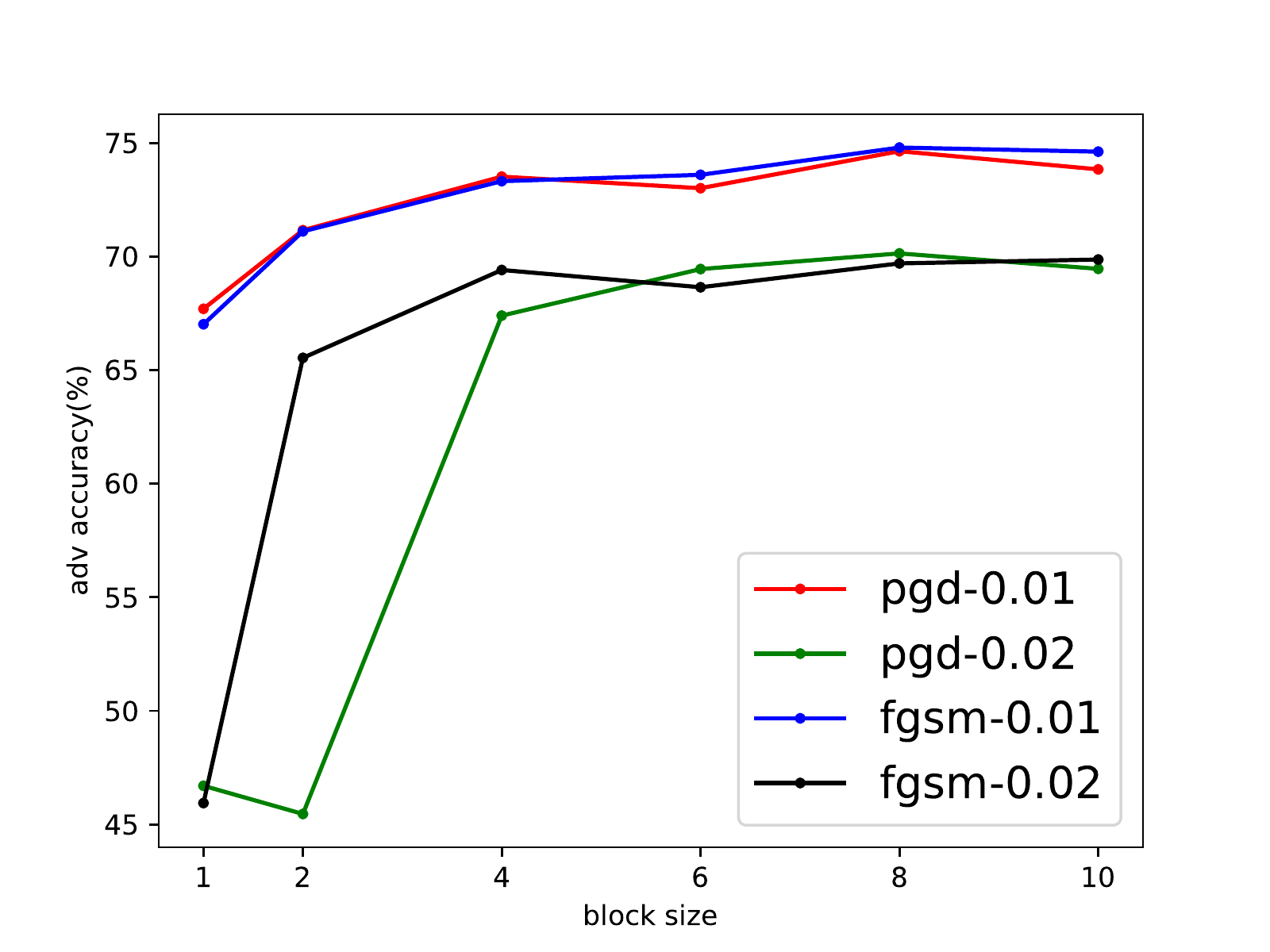}     
    }
    \caption{The effect of block size on the accuracy. In the figure, we show the clean accuracy and the adversarial accuracy of the mean classifier under PGD and FGSM attack with $\epsilon=0.01$ and $\epsilon=0.02$. The block size is choosen from $\{1, 2, 4, 6, 8, 10\}$. \textbf{(a) The influence on the clean accuracy;} \textbf{(b) The influence on the adversarial accuracy.}} \label{fig:block-exp}
\end{figure}

In this section, we conduct several experiments to \textbf{support our theory}. We emphasize that we are not proposing a method to try to get better robustness on the downstream tasks, we do the experiments to verify our two claims from our theoretical results: (1) As shown in Remark \ref{rmk-of-thm-adv-sup-leq-adv-sun-block}, using the blocks may improve the robust performance; (2) As shown in Remark \ref{rmk-of-theory-part}, using the norms of the layers of the neural networks as the regularizer may help improve the robust performance.

\textbf{Data sets.} We use two data sets \citep{krizhevsky2009cifar} in our experiments: (1) the CIFAR-10 data set and (2) the CIFAR-100 data set. CIFAR-10 contains 50000/10000 train/test images with size $32 \times 32$, which are categorized into 10 classes. CIFAR-100 contains 50000/10000 train/test images with size $32 \times 32$, which are categorized into 100 classes.

\textbf{Model.} We use a neural network with two convolutional layers and one fully connected layer. Following \citet{DBLP:conf/cvpr/He0WXG20-MoCo}, we use the Stochastic Gradient Descent (SGD) optimizer with momentum $0.9$ but set the weight decay to be $5\times 10^{-4}$ and the learning rate to be $0.001$.

\textbf{Evaluation of robustness.} For representation $f$, we first calculate $\widehat{u}_c = \frac{1}{n_c} \sum_{i=1}^{n_c} f(x_i)$ to estimate the $c$-th row of the mean classifier, where $x_1, \dots, x_{n_c}$ are the data points with label $c$ in our training set. Denote $\widehat{W}^\mu$ as the estimator of $W$, we use the robustness of the classifier $\widehat{W}^\mu f$ as an evaluation of the robustness of $f$ on the downstream task.

We show the results for \textbf{CIFAR-10} here; the results for CIFAR-100 can be found in the Appendix \ref{apd-exp-results}.

\subsection{Improvement from the regularizer} \label{subsec-exp-regu}
Inspired by our bound \eqref{eq-nn-AG} in Theorem \ref{thm-nn-Rs(H)-bound-Fnorm} and Theorem \ref{thm-2-adv-Lsup-mu-leq-adv-Lsun}, the adversarial supervised risk can be upper bounded by the sum of the adversarial unsupervised loss and $AG_M$, which is related to the maximal Frobenius-norm of the network layers. We choose to simultaneously optimize the contrastive upstream pre-train risk and the Frobenius norm of the parameters of the model. We set the norm of the parameters for the layers as a regularizer and test the performance of the mean classifier; here, $W^\mu$ is calculated by averaging all features of the training data set as done in \citet{DBLP:conf/uai/NozawaGG20}. We use a hyper-parameter $\lambda$ to balance the trade-off of the the contrastive upstream pre-train risk and the Frobenius norm of the parameters of the model. We choose to minimize the following regularized empirical risk:
\begin{equation} \label{obj:experiments-optimization}
    L(f) = \widehat{\widetilde{L}}_{sun}(f) + \lambda N(f)
\end{equation}
where $N(f)$ is a regularizer that constrains the Frobenius norm of the parameters of the model $f$, here we choose $N(f) = \sum_{l=1}^d \normmm{W_l}_F$ where $\normmm{W_l}_F$ is the Frobenius norm of the parameters for the $l$-th layer of $f$ and:
\begin{equation*}
    \widehat{\widetilde{L}}_{sun}(f) \!\!=\!\! \frac{1}{M}\sum_{j=1}^M \! \underset{x_j^\prime \in \mathcal{U}(x_j)}{\max} \ell(\!\{f(x_j^\prime)^T\!(f(x_j^+)\!-\!f(x_{ji}^-))\}_{i=1}^k).
\end{equation*}
More details about our algorithm are in Algorithm \ref{alg:AERM-alg}.

\begin{algorithm} \label{alg:AERM-alg}
\caption{The AERM algorithm for adversarial contrastive learning}\label{algorithm}
\SetKwInOut{Input}{Input}\SetKwInOut{Output}{Output}
\Input{The training data set $\mathcal{S} = \left\{(x_j,x_j^+,x_{j1}^-, \dots, x_{jk}^-)\right\}_{j=1}^M$ sampled from $\mathcal{D}_{sim} \times \mathcal{D}_{neg}^k$; the hyper-parameter $\lambda$ in \eqref{obj:experiments-optimization}; the adversarial perturbation $\mathcal{U}$; learning rate $\alpha$; the total iteration number $T$;}
initialize $\theta^0$ to be randomized parameters\;
$t \leftarrow 0$\;
\While{$t < T$}{
    randomly sample a batch of data with size $N$: $\mathcal{B} \subseteq \mathcal{S}$ \;
    $\widetilde{\mathcal{B}} \leftarrow \emptyset$\;
    \For{$(x,x^+,x_{1}^-, \dots, x_{k}^-)$ in $\mathcal{B}$}{
        calculate adversarial example $\tilde{x} \in \underset{x^\prime \in \mathcal{U}(x)}{\arg\sup} \ \ell\left( \left\{ f_{\theta^t}(x^\prime)^T\left( f_{\theta^t}(x^+) - f_{\theta^t}(x_i^-) \right) \right\}_{i=1}^k \right)$\;
        $\widetilde{\mathcal{B}} \leftarrow \widetilde{\mathcal{B}} \cup \{(\tilde{x},x^+,x_{1}^-, \dots, x_{k}^-)\}$\;
    }
    $L(f_{\theta^t}) = \frac{1}{N} \sum_{(\tilde{x},x^+,x_{1}^-, \dots, x_{k}^-) \in \widetilde{\mathcal{B}}} \ell\left( \left\{ f_{\theta^t}(\tilde{x})^T\left( f_{\theta^t}(x^+) - f_{\theta^t}(x_i^-) \right) \right\}_{i=1}^k \right) + \lambda N(f_{\theta^t})$\;
    $\theta^{t+1} \leftarrow \theta^t - \alpha \nabla_{\theta^t} L(f_{\theta^t})$\;
    $t\leftarrow t+1$\;
}
\Output{The feature extractor $f_{\theta^{T}} \in \mathcal{F}$ that tries to minimize \eqref{obj:experiments-optimization};}
\end{algorithm}

The results are shown in Table \ref{table-regu}. From Table \ref{table-regu}, we can see that the $F$-norm regularizer can improve the adversarial accuracy (the prediction performance of a model on adversarial examples generated by attacker) of the mean classifier, which is in line with our Theorem \ref{thm-nn-Rs(H)-bound-Fnorm}.

\subsection{Effect of block size} \label{subsec-exp-block}
To verify Theorem \ref{thm-adv-sup-leq-adv-sun-block}, we analyze the effect of block size on the adversarial accuracy of the mean classifier. Figure~\ref{fig:block-exp} presents the results for clean accuracy and adversarial accuracy, respectively. From Figure~\ref{fig:block-exp}, we can see that a larger block size will yield better adversarial accuracy. The results are consistent with Theorem~\ref{thm-adv-sup-leq-adv-sun-block}: as the block size grows, Theorem~\ref{thm-adv-sup-leq-adv-sun-block} shows that we are optimizing a tighter bound, which leads to better performance as shown in Figure~\ref{fig:block-exp}.

\section{Conclusion} \label{sec-conclusion}
This paper studies the generalization performance of adversarial contrastive learning. We first extend the contrastive learning framework to the adversarial case, then we upper bound the average adversarial risk of the downstream tasks with the adversarial unsupervised risk of the upstream task and an adversarial Rademacher complexity term. Furthermore, we provide the upper bound of the adversarial Rademacher complexity for linear models and multi-layer neural networks. Finally, we conduct several experiments and the experimental results are consistent with our theory.


\acks{This work is supported by the National Natural Science Foundation of China under Grant 61976161.}


\newpage

\appendix
\setcounter{table}{0}
\setcounter{figure}{0}
\setcounter{equation}{0}

\renewcommand{\thefigure}{\thesection.\arabic{figure}}
\renewcommand{\thetable}{\thesection.\arabic{table}}
\renewcommand{\theequation}{\thesection.\arabic{equation}}

\section{Proofs} \label{sec::apd}
In this section, we display the proofs of our theorems, lemmas and corollaries. For reading convenience, we will restate the theorem before proving.

\subsection{Proof of Theorem \ref{thm-2-adv-Lsup-mu-leq-adv-Lsun}} \label{apd-prf-thm-2-adv-Lsup-mu-leq-adv-Lsun}
In the below, we present some useful lemmas that will be used in the proofs of our main theorems.
\begin{lemma} \label{lma-adv-Lsup-leq-adv-Lun}
For any $f \in \mathcal{F}$, we have:
\begin{equation} \label{eq-adv-Lsup-leq-adv-Lun}
    \begin{aligned}
    	\widetilde{L}_{sup}(f) & \leq \widetilde{L}_{sup}^{\mu}(f) \leq \frac{1}{1-\tau} (\widetilde{L}_{un}(f)-\tau \ell (0)) \leq \frac{1}{1-\tau} (\widetilde{L}_{sun}(f)-\tau \ell (0)).
    \end{aligned}
\end{equation}
\end{lemma}

\begin{remark}
$\widehat{\widetilde{L}}_{sun}(f)$ denotes the empirical $\widetilde{L}_{sun}(f)$  and $ \widehat{f} \in \underset{f \in \mathcal{F}}{\arg\min}\ \widehat{\widetilde{L}}_{sun}(f)$. Applying Lemma \ref{lma-adv-Lsup-leq-adv-Lun} to $\widehat{f}$ shows that, if we can train a robust feature extractor with low surrogate adversarial unsupervised risk, we can obtain a robust classifier with low adversarial supervised risk on the downstream task.
\end{remark}
\begin{lemma} \label{lma-k-adv-Lun(fhat)-leq-adv-Lun(f)}
Let $\ell: \mathbb{R}^k \xrightarrow[]{} \mathbb{R}$ be bounded by $B$. Then, for any $\delta \in (0,1)$, with a probability of at least $1 - \delta$ over the choice of the training set $\mathcal{S}=\{(x_j,x_j^+,x_{j1}^-,\dots,x_{jk}^-)\}_{j=1}^M = \{z_j\}_{j=1}^M$, for any $f \in \mathcal{F}$:
\begin{equation*}
    \widetilde{L}_{sun}(\widehat{f}) \leq \widetilde{L}_{sun}(f) + AG_M,
\end{equation*}
where $AG_M = O(\frac{\mathcal{R}_\mathcal{S}(\mathcal{G})}{M} + B \sqrt{\frac{log \frac{1}{\delta}}{M}}),\ \  \mathcal{R}_\mathcal{S}(\mathcal{G}) \coloneqq \underset{\bm{\sigma} \sim \{\pm 1\}^M}{\mathbb{E}} \left[\underset{f \in \mathcal{F}}{\sup} \left< \bm{\sigma},(g_f)_{|\mathcal{S}}\right>\right],$
and $\mathcal{G} \coloneqq \{ g_f(x,x^+,x_1^-,\dots,x_k^-) = \underset{x^\prime \in \mathcal{U}(x)}{\sup} \ell\left(\{f(x^\prime)^T(f(x^+)-f(x_i^-))\}_{i=1}^k\right) | f \in \mathcal{F}\},$
where $\bm{\sigma}$ is an $M$-dimensional Rademacher random vector with i.i.d. entries and $(g_f)_{|\mathcal{S}} = \left(g_f(z_1), \dots, g_f(z_M)\right)$.
\end{lemma}

\begin{proof}[Proof of Lemma \ref{lma-adv-Lsup-leq-adv-Lun}]
By the definition of $\widetilde{L}_{sup}(\mathcal{T},f)$, i.e., \eqref{eq-adv-Lsup-T}, it's obvious that $\widetilde{L}_{sup}(f) \leq \widetilde{L}_{sup}^{\mu}(f)$, so we only need to prove the second part of \eqref{eq-adv-Lsup-leq-adv-Lun}.
From definition \eqref{eq-U-adv-Lun}, we have, $\forall{f \in \mathcal{F}}$:
\vskip -0.2in
\begin{equation} \label{ieq-A1-1}
    \begin{aligned}
        \widetilde{L}_{un}(f) &= \underset{c^+,c^- \sim \rho^2}{\mathbb{E}} \left\{ \underset{x \sim \mathcal{D}_{c^+}}{\mathbb{E}} \left[ \underset{x^\prime \in \mathcal{U}(x)}{\sup} \left( \underset{x^- \sim \mathcal{D}_{c^-}}{\underset{x^+ \sim \mathcal{D}_{c^+}}{\mathbb{E}}} \left( \ell \left( f(x^\prime)^T \left( f(x^+) - f(x^-) \right) \right) \right) \right) \right] \right\}   \\[1mm]
        &\overset{\sroman{1}}{\ge} \underset{c^+,c^- \sim \rho^2}{\mathbb{E}} \left\{ \underset{x \sim \mathcal{D}_{c^+}}{\mathbb{E}} \left[ \underset{x^\prime \in \mathcal{U}(x)}{\sup} \ell \left( f(x^\prime)^T(\mu_{c^+} - \mu_{c^-}) \right) \right] \right\} \\[1mm]
        &\overset{\sroman{2}}{=} (1-\tau) \underset{c^+,c^- \sim \rho^2}{\mathbb{E}} \left\{ \underset{x \sim \mathcal{D}_{c^+}}{\mathbb{E}} \left[ \underset{x^\prime \in \mathcal{U}(x)}{\sup} \ell \left( f(x^\prime)^T(\mu_{c^+} - \mu_{c^-}) \right) \right] \Bigg| c^+ \neq c^- \right\} + \tau \ell(0),
    \end{aligned}
\end{equation}
where $\sroman{1}$ comes from the convexity of $\ell$ and Jensen's Inequality, and $\sroman{2}$ comes from the property of conditional expectations. Then we have:
\vskip -0.2in
\begin{equation} \label{ieq-A1-add2}
    \begin{aligned}
        & \underset{c^+,c^- \sim \rho^2}{\mathbb{E}} \left\{ \underset{x \sim \mathcal{D}_{c^+}}{\mathbb{E}} \left[ \underset{x^\prime \in \mathcal{U}(x)}{\sup} \ell \left( f(x^\prime)^T(\mu_{c^+} - \mu_{c^-}) \right) \right] \Bigg| c^+ \neq c^- \right\} \\[1mm]
        & \overset{\sroman{1}}{=} \underset{c^+,c^- \sim \rho^2}{\mathbb{E}} \Bigg\{ D_\mathcal{T}(c^+) \underset{x \sim \mathcal{D}_{c^+}}{\mathbb{E}} \left[ \underset{x^\prime \in \mathcal{U}(x)}{\sup} \ell \left( f(x^\prime)^T(\mu_{c^+} - \mu_{c^-}) \right) \right] \\[1mm]
        &\ \ \ \ \ \ \ \ \ \ \ \ + D_\mathcal{T}(c^-) \underset{x \sim \mathcal{D}_{c^-}}{\mathbb{E}} \left[ \underset{x^\prime \in \mathcal{U}(x)}{\sup} \ell \left( f(x^\prime)^T(\mu_{c^-} - \mu_{c^+}) \right) \right] \Bigg| c^+ \neq c^- \Bigg\} \\[1mm]
        & \overset{\sroman{2}}{=} \underset{c^+,c^- \sim \rho^2}{\mathbb{E}} \Bigg\{ D_\mathcal{T}(c^+) \underset{x \sim \mathcal{D}_{c^+}}{\mathbb{E}} \left[ \underset{x^\prime \in \mathcal{U}(x)}{\sup} \ell \left( g(x^\prime)_{c^+} - g(x^\prime)_{c^-} \right) \right] \\[1mm]
        &\ \ \ \ \ \ \ \ \ \ \ \ + D_\mathcal{T}(c^-) \underset{x \sim \mathcal{D}_{c^-}}{\mathbb{E}} \left[ \underset{x^\prime \in \mathcal{U}(x)}{\sup} \ell \left( g(x^\prime)_{c^-} - g(x^\prime)_{c^+} \right) \right] \Bigg| c^+ \neq c^- \Bigg\},
    \end{aligned}
\end{equation}
where $\mathcal{T} = \{c^+, c^-\}$, $g(x) = \begingroup
\renewcommand*{\arraystretch}{1.5}
\begin{bmatrix} \mu_{c^+}^T f(x) \\ \mu_{c^-}^T f(x) \end{bmatrix}
\endgroup = W^\mu f(x)$ and $\sroman{1}$ comes from the symmetry of $c^+, c^-$; $\sroman{2}$ is directly from some linear algebras.
From \eqref{ieq-A1-add2} we know that:
\begin{equation} \label{ieq-A1-2}
    \begin{aligned}
        & \underset{c^+,c^- \sim \rho^2}{\mathbb{E}} \left\{ \underset{x \sim \mathcal{D}_{c^+}}{\mathbb{E}} \left[ \underset{x^\prime \in \mathcal{U}(x)}{\sup} \ell \left( f(x^\prime)^T(\mu_{c^+} - \mu_{c^-}) \right) \right] \Bigg| c^+ \neq c^- \right\} \\[1mm]
        & \overset{\sroman{1}}{=} \underset{c^+,c^- \sim \rho^2}{\mathbb{E}} \left\{ \underset{c \sim \mathcal{D}_\mathcal{T}}{\mathbb{E}} \underset{x \sim \mathcal{D}_c}{\mathbb{E}} \left[ \underset{x^\prime \in \mathcal{U}(x)}{\sup} \ell \left( \left\{ g(x^\prime)_c - g(x^\prime)_{c^\prime} \right\}_{c^\prime \neq c} \right) \right] \Bigg| c^+ \neq c^- \right\} \\[1mm]
        & \overset{\sroman{2}}{=} \underset{c^+,c^- \sim \rho^2}{\mathbb{E}} \left[ \widetilde{L}_{sup}\left( \left\{ c^+, c^- \right\},W^\mu f \right) \Bigg| c^+ \neq c^- \right] \\[1mm]
        & \overset{\sroman{3}}{=} \underset{c^+,c^- \sim \rho^2}{\mathbb{E}} \left[ \widetilde{L}_{sup}^\mu\left( \left\{ c^+, c^- \right\}, f \right) \Bigg| c^+ \neq c^- \right] \overset{\sroman{4}}{=} \widetilde{L}_{sup}^\mu (f),
    \end{aligned}
\end{equation}
where $\sroman{1}$ is due to the tower property of expectation;$\sroman{2}$ is obvious by the definition of $\widetilde{L}_{sup}(\mathcal{T}, g)$; $\sroman{3}$ comes from the definition of $\widetilde{L}_{sup}^\mu(\mathcal{T}, f)$ and $\sroman{4}$ is from the definition of $\widetilde{L}_{sup}^\mu(f)$.
Combine \eqref{ieq-A1-1} with \eqref{ieq-A1-2}, we conclude that:
\begin{equation*}
    \widetilde{L}_{un}(f) \ge (1-\tau) \widetilde{L}_{sup}^\mu (f) + \tau \ell(0), \forall{f \in \mathcal{F}}.
\end{equation*}
So we have:
\begin{equation} \label{ieq-A1-final}
    (1-\tau) \widetilde{L}_{sup} (f) + \tau \ell(0) \leq (1-\tau) \widetilde{L}_{sup}^\mu (f) + \tau \ell(0) \leq \widetilde{L}_{un}(f), \forall{f \in \mathcal{F}}.
\end{equation}
Rearranging \eqref{ieq-A1-final} yields \eqref{eq-adv-Lsup-leq-adv-Lun}.
\end{proof}

\begin{proof}[Proof of Lemma \ref{lma-k-adv-Lun(fhat)-leq-adv-Lun(f)}]
Denote $(x,x^+,x_1^-,\cdots,x_k^-)$ by $z$, then by the Theorem 3.3 in \cite{DBLP:books/daglib/0034861}, we have: With probability at least $1-\frac{\delta}{2}$ over the choice of the training set $\mathcal{S}$, 
\begin{equation*}
    \underset{z}{\mathbb{E}}\left[ \frac{1}{B} g_f(z) \right] \leq \frac{1}{M} \sum_{i=1}^M \frac{1}{B} g_f(z_i) + \frac{2}{M} \mathcal{R}_\mathcal{S}(\frac{\mathcal{G}}{B}) + 3\sqrt{\frac{log \frac{4}{\delta}}{M}},
\end{equation*}
which is equivalent to:
\begin{equation} \label{ieq-A2-1}
    \widetilde{L}_{sun}(f) \leq \widehat{\widetilde{L}}_{sun}(f) + \frac{2}{M} \mathcal{R}_\mathcal{S}(\mathcal{G}) + 3 B \sqrt{\frac{log \frac{4}{\delta}}{M}},  \forall{f \in \mathcal{F}}.
\end{equation}
Let $f^* = \underset{f \in \mathcal{F}}{\arg\min} \ \widetilde{L}_{sun}(f)$, since $\underset{\mathcal{S}}{\mathbb{E}} \left[ \widehat{\widetilde{L}}_{sun}(f) \right] = \widetilde{L}_{sun}(f)$ and $\ell$ is bounded by $B$, Hoeffding's inequality tells us that: $\forall{f \in \mathcal{F}}$, $\forall{t \in \mathbb{R}}$:
\begin{equation*}
    \mathbb{P} \left[ \widehat{\widetilde{L}}_{sun}(f) - \widetilde{L}_{sun}(f) \ge t \right] \leq e^{-\frac{2Mt^2}{B^2}}.
\end{equation*}
Set $t=B\sqrt{\frac{log \frac{2}{\delta}}{2M}}$, we have: $\forall{f \in \mathcal{F}}$, with probability at least $1-\frac{\delta}{2}$,
\begin{equation} \label{ieq-A2-2}
     \widehat{\widetilde{L}}_{sun}(f) - \widetilde{L}_{sun}(f) \leq B\sqrt{\frac{log \frac{2}{\delta}}{2M}}.
\end{equation}
Combine \eqref{ieq-A2-1} with \eqref{ieq-A2-2}, the union bound tells us that: $\forall{f \in \mathcal{F}}$, with probability at least $1-\delta$ over the choice of the training set $\mathcal{S}$, 
\begin{equation*}
    \begin{aligned}
        \widetilde{L}_{sun}(\widehat{f}) &\overset{\sroman{1}}{\leq} \widehat{\widetilde{L}}_{sun}(\widehat{f}) + O\left( \frac{\mathcal{R}_\mathcal{S}(\mathcal{G})}{M} + B \sqrt{\frac{log \frac{1}{\delta}}{M}} \right) \overset{\sroman{2}}{\leq} \widehat{\widetilde{L}}_{sun}(f^*) + O\left( \frac{\mathcal{R}_\mathcal{S}(\mathcal{G})}{M} + B \sqrt{\frac{log \frac{1}{\delta}}{M}} \right) \\[1mm]
        &\overset{\sroman{3}}{\leq} \widetilde{L}_{sun}(f^*) + O\left( \frac{\mathcal{R}_\mathcal{S}(\mathcal{G})}{M} + B\sqrt{\frac{log \frac{1}{\delta}}{M}} \right) \overset{\sroman{4}}{\leq} \widetilde{L}_{sun}(f) + O\left( \frac{\mathcal{R}_\mathcal{S}(\mathcal{G})}{M} + B\sqrt{\frac{log \frac{1}{\delta}}{M}} \right),
    \end{aligned}
\end{equation*}
\vskip -0.1in
where $\sroman{1}$ comes from \eqref{ieq-A2-1};$\sroman{2}$ is directly from the fact that $\widehat{\widetilde{L}}_{sun}(\widehat{f}) \leq \widehat{\widetilde{L}}_{sun}(f^*)$, which is from the definition of $\widehat{f}$;$\sroman{3}$ is a result of \eqref{ieq-A2-2} and $\sroman{4}$ is obvious by the definition of $f^*$.
\end{proof}

\begin{customthm}{{\ref{thm-2-adv-Lsup-mu-leq-adv-Lsun}}}
Let $\ell: \mathbb{R}^k \xrightarrow[]{} \mathbb{R}$ be bounded by $B$. Then, for any $\delta \in (0,1)$,with a probability of at least $1 - \delta$ over the choice of the training set $\mathcal{S}=\{(x_j,x_j^+,x_{j}^-)\}_{j=1}^M$, for any $f \in \mathcal{F}$:
\begin{equation*}
    \widetilde{L}_{sup}(\widehat{f}) \leq  \widetilde{L}_{sup}^{\mu}(\widehat{f}) \leq \frac{1}{1-\tau} (\widetilde{L}_{sun}(f) - \tau \ell (0)) + \frac{1}{1-\tau} AG_M.
\end{equation*}
\end{customthm}

\begin{proof}
From Lemma \ref{lma-adv-Lsup-leq-adv-Lun} we know that:
\begin{equation*}
    \widetilde{L}_{sup}(\widehat{f}) \leq \widetilde{L}_{sup}^{\mu}(\widehat{f}) \leq \frac{1}{1-\tau} (\widetilde{L}_{un}(\widehat{f})-\tau \ell (0)) \leq \frac{1}{1-\tau} (\widetilde{L}_{sun}(\widehat{f})-\tau \ell (0)),
\end{equation*}
Then Lemma \ref{lma-k-adv-Lun(fhat)-leq-adv-Lun(f)} directly yields the result we need.
\end{proof}

\subsection{Proof of Theorem \ref{thm-adv-sup-leq-adv-sun-block}} \label{apd-prf-thm-adv-sup-leq-adv-sun-block}
\begin{customthm} {{\ref{thm-adv-sup-leq-adv-sun-block}}}
For any $f \in \mathcal{F}$, we have:
\begin{equation*}
    \begin{aligned}
        \widetilde{L}_{sup}(f) &\leq \frac{1}{1-\tau} \left(\widetilde{L}_{sun}^{block}(f) - \tau \ell(0)\right) \leq \frac{1}{1-\tau} \left(\widetilde{L}_{sun}(f) - \tau \ell(0)\right).
    \end{aligned}
\end{equation*}
\end{customthm}

\begin{proof}
By the convexity of $\ell$ and Jensen's inequality, we have: $\forall{x^\prime, x_i^+, x_i^-}$:
\begin{equation*}
    \begin{aligned}
        \ell \left( f(x^\prime)^T \left( \frac{\sum_{i=1}^b f(x_i^+)}{b} - \frac{\sum_{i=1}^b f(x_i^-)}{b} \right) \right) &= \ell \left( \frac{1}{b} \sum_{i=1}^b f(x^\prime)^T\left( f(x_i^+) - f(x_i^-) \right)\right) \\
        &\leq \frac{1}{b} \sum_{i=1}^b \ell \left( f(x^\prime)^T\left( f(x_i^+) - f(x_i^-) \right) \right).
    \end{aligned}
\end{equation*}
Take maximization about $x^\prime$ both sides, we have:
\begin{equation*}
    \begin{aligned}
        \underset{x^\prime \in \mathcal{U}(x)}{\sup} \!\ell\! \left(\! f(x^\prime)^T \!\!\left( \frac{\sum_{i=1}^b f(x_i^+\!)}{b} \!-\! \frac{\sum_{i=1}^b f(x_i^-\!)}{b} \right) \!\!\!\right) &\!\leq \!\frac{1}{b} \underset{x^\prime \in \mathcal{U}(x)}{\sup} \!\!\left[ \sum_{i=1}^b \! \ell \!\left( f(x^\prime)^T\!\!\left( f(x_i^+\!) \!-\! f(x_i^-) \right) \right) \!\right] \\[1mm]
        &\!\leq\! \frac{1}{b} \sum_{i=1}^b \!\left[ \underset{x^\prime \in \mathcal{U}(x)}{\sup}\! \ell \!\left( f(x^\prime)^T\!\!\left( f(x_i^+\!) \!-\! f(x_i^-\!)\! \right) \!\right)\! \right].
    \end{aligned}
\end{equation*}
Taking expectations both sides yields:
\begin{equation*}
    \begin{aligned}
        \widetilde{L}_{sun}^{block} (f) &= \underset{x,x_i^+,x_i^-}{\mathbb{E}} \left[ \underset{x^\prime \in \mathcal{U}(x)}{\sup} \ell \left( f(x^\prime)^T \left( \frac{\sum_{i=1}^b f(x_i^+)}{b} - \frac{\sum_{i=1}^b f(x_i^-)}{b} \right) \right) \right] \\[1mm]
        &\leq \underset{x,x_i^+,x_i^-}{\mathbb{E}} \left\{ \frac{1}{b} \sum_{i=1}^b \left[ \underset{x^\prime \in \mathcal{U}(x)}{\sup} \ell \left( f(x^\prime)^T\left( f(x_i^+) - f(x_i^-) \right) \right) \right] \right\} \\[1mm]
        &= \underset{x,x^+,x^-}{\mathbb{E}} \left[ \underset{x^\prime \in \mathcal{U}(x)}{\sup} \ell \left( f(x^\prime)^T\left( f(x^+) - f(x^-) \right) \right) \right] = \widetilde{L}_{sun} (f),
    \end{aligned}
\end{equation*}
which proves the second inequality. For the first inequality, we have, $\forall{f \in \mathcal{F}}$:
\vskip -0.1in
\begin{equation*}
    \begin{aligned}
        \widetilde{L}_{sun}^{block} (f) &= \underset{c^+,c^- \sim \rho^2}{\mathbb{E}} \left\{ \underset{x_i^- \sim \mathcal{D}_{c^-}^b}{\underset{x,x_i^+ \sim \mathcal{D}_{c^+}^{b+1}}{E}} \left[ \underset{x^\prime \in \mathcal{U}(x)}{\sup} \ell \left( f(x^\prime)^T \left( \frac{\sum_{i=1}^b f(x_i^+)}{b} - \frac{\sum_{i=1}^b f(x_i^-)}{b} \right) \right) \right] \right\} \\[1mm]
        &\overset{\sroman{1}}{\ge} \!\!\underset{c^+,c^- \sim \rho^2}{\mathbb{E}}\!\!\! \left\{ \!\underset{x \sim \mathcal{D}_{c^+}}{\mathbb{E}} \!\!\left\{ \underset{x^\prime \in \mathcal{U}(x)}{\sup} \underset{x_i^- \sim \mathcal{D}_{c^-}^b}{\underset{x_i^+ \sim \mathcal{D}_{c^+}^b}{E}} \!\!\left[ \ell \!\left( f(x^\prime)^T\!\! \left( \frac{\sum_{i=1}^b f(x_i^+)}{b} \!-\! \frac{\sum_{i=1}^b f(x_i^-)}{b} \right) \right) \right] \right\} \right\} \\[1mm]
        &\overset{\sroman{2}}{\ge} \!\!\underset{c^+,c^- \sim \rho^2}{\mathbb{E}} \!\!\!\left\{ \!\underset{x \sim \mathcal{D}_{c^+}}{\mathbb{E}} \!\!\left\{ \underset{x^\prime \in \mathcal{U}(x)}{\sup} \ell\!\! \left[  \underset{x_i^- \sim \mathcal{D}_{c^-}^b}{\underset{x_i^+ \sim \mathcal{D}_{c^+}^b}{E}} \!\!\!\left( f(x^\prime)^T \!\!\left( \frac{\sum_{i=1}^b f(x_i^+)}{b} \!-\! \frac{\sum_{i=1}^b f(x_i^-)}{b} \right) \right) \right] \right\} \right\} \\[1mm]
        &\overset{\sroman{3}}{=} \underset{c^+,c^- \sim \rho^2}{\mathbb{E}} \left\{ \underset{x \sim \mathcal{D}_{c^+}}{\mathbb{E}} \left[ \underset{x^\prime \in \mathcal{U}(x)}{\sup} \ell \left( f(x^\prime)^T(\mu_{c^+} - \mu_{c^-}) \right) \right] \right\} \overset{\sroman{4}}{=} (1-\tau) \widetilde{L}_{sup}^\mu (f) + \tau \ell(0),
    \end{aligned}
\end{equation*}
\vskip -0.1in
where $\sroman{1}$ and $\sroman{2}$ are directs result of Jensen's Inequality and convexity of maximization function and $\ell$; $\sroman{3}$ is from the linearity of expectation and the last equality follows the same argumentation as in \eqref{ieq-A1-2}, which proves the Theorem.
\end{proof}

\subsection{Proof of \eqref{ieq-cLsup-leq-Lsup}} \label{apd-prf-ieq-cLsup-leq-Lsup}
\begin{proof}
For any $f \in \mathcal{F}$, we have:
\vskip -0.1in
\begin{equation*}
    \begin{aligned}
        \widetilde{\mathcal{L}}_{sup}(f) \!&=\!\! \underset{\mathcal{T} \sim \mathcal{D}}{\mathbb{E}} \left[ \widetilde{L}_{sup} (\mathcal{T}, f) \right] \!\overset{\sroman{1}}{=}\! p\! \ \underset{\mathcal{T} \sim \mathcal{D}}{\mathbb{E}} \left[ \widetilde{L}_{sup} (\mathcal{T}, f) \Big| E_{distinct} \right] \!+\! (1\!-\!p) \underset{\mathcal{T} \sim \mathcal{D}}{\mathbb{E}} \left[ \widetilde{L}_{sup} (\mathcal{T}, f) \Big| \bar{E}_{distinct} \right] \\[1mm]
        &\ge  p \ \underset{\mathcal{T} \sim \mathcal{D}}{\mathbb{E}} \left[ \widetilde{L}_{sup} (\mathcal{T}, f) \Big| E_{distinct} \right] \overset{\sroman{2}}{=} p \ \widetilde{L}_{sup} (f),
    \end{aligned}
\end{equation*}
\vskip -0.1in
where $E_{distinct}$ is the event that $\{c^+,c_1^-,\dots,c_k^-\}$ is distinct and $p = \underset{(c^+,c_1^-,\dots,c_k^-) \sim \mathcal{D}}{P}[E_{distinct}]$ and $\sroman{1}$ is from the property of conditional expectation and $\sroman{2}$ comes from the definition of $\widetilde{L}_{sup} (f)$.
\end{proof}

\subsection{Proof of Proposition \ref{prop-loss-function}} \label{apd-prf-prop-loss-function}
\begin{custompro}{{\ref{prop-loss-function}}}
    The hinge loss and the logistic loss satisfy Assumption \ref{ass-loss-function}.
\end{custompro}

\begin{proof}
Since $I_1$ and $I_2$ are symmetric, We only need to prove \eqref{eq-loss-property-1}.
\textbf{For the Hinge Loss}:
\begin{equation*}
    \begin{aligned}
        \ell \left( \{ v_i \}_{i \in I_1} \right) &= \max \left\{ 0, 1 + \underset{i \in I_1}{\max} \{ -v_i \} \right\} &\overset{\sroman{1}}{\leq} \max \left\{ 0, 1 + \underset{i \in [d]}{\max} \{ -v_i \} \right\} = \ell \left( \{ v_i \}_{i \in [d]} \right),
    \end{aligned}
\end{equation*}
where $\sroman{1}$ is from the fact that $I_1 \subseteq [d]$, so the first inequality is proved, for the second one:
\begin{equation*}
    \begin{aligned}
        \ell \left( \{ v_i \}_{i \in [d]} \right) &= \max \left\{ 0, 1 + \underset{i \in [d]}{\max} \{ -v_i \} \right\} = \max \left\{ 0, 1 + \underset{i \in I_1 \cup I_2}{\max} \{ -v_i \} \right\},
    \end{aligned}
\end{equation*}
where the last equality is directly from the definition of $I_1$ and $I_2$.
\begin{enumerate}
\item if \ \ $1 + \underset{i \in I_1 \cup I_2}{\max} \{ -v_i \} \leq 0$, then $\ell \left( \{ v_i \}_{i \in I_1 \cup I_2} \right) = 0$, since Hinge Loss is non-negative, we have:
\begin{equation*}
    \ell \left( \{ v_i \}_{i \in I_1 \cup I_2} \right) = 0 + 0 \leq \ell \left( \{ v_i \}_{i \in I_1} \right) + \ell \left( \{ v_i \}_{i \in I_2} \right).
\end{equation*}
\item if \ \ $1 + \underset{i \in I_1 \cup I_2}{\max} \{ -v_i \} > 0$
    \begin{enumerate}
        \item if \ $1 + \underset{i \in I_1}{\max} \{ -v_i \} \leq 0$, then $\ell \left( \{ v_i \}_{i \in I_1} \right) = \max \left\{ 0, 1 + \underset{i \in I_1}{\max} \{ -v_i \} \right\} = 0$. So we have:
        \begin{equation*}
            \begin{aligned}
                \ell \left( \{ v_i \}_{i \in I_1 \cup I_2} \right) &= \max \left\{ 0, 1+\underset{i \in I_1 \cup I_2}{\max}\{ -v_i \} \right\} = \max \left\{ 0, 1+\underset{i \in I_2}{\max}\{ -v_i \} \right\} \\[1mm]
                &= 0 + \ell \left( \{ v_i \}_{i \in I_2} \right) = \ell \left( \{ v_i \}_{i \in I_1} \right) + \ell \left( \{ v_i \}_{i \in I_2} \right).
            \end{aligned}
        \end{equation*}
        \item if \ $1 + \underset{i \in I_2}{\max} \{ -v_i \} \leq 0$, by the same discussion of (a), we have:
        \begin{equation*}
            \ell \left( \{ v_i \}_{i \in I_1 \cup I_2} \right) \leq \ell \left( \{ v_i \}_{i \in I_1} \right) + \ell \left( \{ v_i \}_{i \in I_2} \right).
        \end{equation*}
        \item if \ $1 + \underset{i \in I_1}{\max} \{ -v_i \} > 0$ \  and \  $1 + \underset{i \in I_2}{\max} \{ -v_i \} > 0$,
        \begin{equation*}
            \begin{aligned}
                \ell \left( \{ v_i \}_{i \in I_1 \cup I_2} \right) \!&=\! \max\! \left\{ 0, 1\!+\!\!\!\underset{i \in I_1 \cup I_2}{\max}\{ -v_i \}\! \right\}\! \leq\! \max\! \left\{ 0, 1 \!+\!\! \underset{i \in I_1}{\max}\{ -v_i \} \!+\! 1 \!+\!\! \underset{i \in I_2}{\max}\{ -v_i \} \right\} \\[1mm]
                &\leq \max \left\{ 0, 1+\underset{i \in I_1}{\max}\{ -v_i \} \right\} + \max \left\{ 0, 1+\underset{i \in I_2}{\max}\{ -v_i \} \right\} \\[1mm]
                &= \ell \left( \{ v_i \}_{i \in I_1} \right) + \ell \left( \{ v_i \}_{i \in I_2} \right).
            \end{aligned}
        \end{equation*}
    \end{enumerate}
\end{enumerate}
So the second inequality is proved.
\textbf{For the Logistic Loss}:
\begin{equation*}
    \begin{aligned}
        \ell \left( \{ v_i \}_{i \in I_1} \right) &= log_2 \left( 1 + \sum_{i \in I_1} e^{-v_i} \right)\leq log_2 \left( 1 + \sum_{i \in [d]} e^{-v_i} \right) = \ell \left( \{ v_i \}_{i \in [d]} \right).
    \end{aligned}
\end{equation*}
So the first inequality is proved. For the second one:
\begin{equation*}
    \begin{aligned}
        \ell \left( \{ v_i \}_{i \in I_1 \cup I_2} \right) &= log_2 \left( 1 + \sum_{i \in I_1 \cup I_2} e^{-v_i} \right) \leq log_2 \left( 1 + \sum_{i \in I_1} e^{-v_i} + \sum_{i \in I_2} e^{-v_i} \right) \\[1mm]
        &\leq log_2 \left( 1 + \sum_{i \in I_1} e^{-v_i} + \sum_{i \in I_2} e^{-v_i} + \left( \sum_{i \in I_1} e^{-v_i} \right) \left( \sum_{i \in I_2} e^{-v_i} \right) \right) \\[1mm]
        &= log_2 \left[ \left( 1 + \sum_{i \in I_1} e^{-v_i} \right) \left( 1 + \sum_{i \in I_2} e^{-v_i} \right) \right] \\[1mm]
        &= log_2 \left( 1 + \sum_{i \in I_1} e^{-v_i} \right) + log_2 \left( 1 + \sum_{i \in I_2} e^{-v_i} \right) = \ell \left( \{ v_i \}_{i \in I_1} \right) + \ell \left( \{ v_i \}_{i \in I_2} \right),
    \end{aligned}
\end{equation*}
which proves the second inequality.
\end{proof}

\subsection{Proof of Theorem \ref{thm-k-general-bound}} \label{apd-prf-thm-k-general-bound}
We here introduce some further notations, which will be used in our main results.
For a tuple $(c^+,c_1^-,\dots,c_k^-)$, we define $I^+ \coloneqq \{i \in [k] | c_i^- = c^+\}$ and $Q$ as the set of distinct classes in $(c^+,c_1^-,\dots,c_k^-)$. We define $p_{\max}(\mathcal{T}) \coloneqq \max_c \mathcal{D}_\mathcal{T}(c)$, $\tau_k \coloneqq \underset{c^+,c_i^- \sim \rho^{k+1}}{\mathbb{P}}[I^+ \neq \emptyset]$, while $\ell_N(\Vec{0})$ is defined as the loss of the $N$-dimensional zero vector.
Let $\mathcal{T}$ be a task sample from distribution $\mathcal{D}$ and $\rho^+(\mathcal{T})$ be a distribution of $c^+$ when $(c^+,c_1^-,\dots,c_k^-)$ are sampled from $\rho^{k+1}$ conditioned on $Q = \mathcal{T}$ and $I^+ = \emptyset$ and $\rho_{\min}^+(\mathcal{T}) \coloneqq \underset{c \in \mathcal{T}}{\min}\ \rho^+(\mathcal{T})(c)$.

\begin{theorem} \label{thm-k-general-bound}
    Assume that $\ell$ satisfies Assumption \ref{ass-loss-function}. With a probability of at least $1-\delta$ over the choice of the training set $\mathcal{S}$, for any $f \in \mathcal{F}$, we have:
    \begin{equation} \label{eq-k-general-bound}
        \begin{aligned}
            &\underset{\mathcal{T} \sim \mathcal{D}}{\mathbb{E}}\!\!\left[\frac{\rho_{\min}^+(\mathcal{T})}{p_{\max}(\mathcal{T})} \widetilde{L}_{sup}(\mathcal{T}, \widehat{f})\right]\!\!\! \leq \!\!\underset{\mathcal{T} \sim \mathcal{D}}{\mathbb{E}}\!\!\left[\frac{\rho_{\min}^+(\mathcal{T})}{p_{\max}(\mathcal{T})} \widetilde{L}_{sup}^\mu(\mathcal{T}, \widehat{f})\right] \\[1mm]
            &\leq \frac{1}{1-\tau_k}\left(\widetilde{L}_{sun}(f) + AG_M\right)-\frac{\tau_k}{1-\tau_k}\underset{c^+,c_i^- \sim \rho^{k+1}}{\mathbb{E}}\left[\ell_{|I^+|}(\Vec{0})|I^+\neq \emptyset\right],
        \end{aligned}
    \end{equation}
	where $|I^+|$ is the cardinality of set $I^+$.
\end{theorem}

Before proceeding with the proof of Theorem \ref{thm-k-general-bound}, we introduce some useful lemmas.

\begin{lemma} \label{lma-rho-ge-DT}
    For any $\mathcal{T}$ sampled from $\mathcal{D}$, we have: $\rho^+(\mathcal{T})(c) \ge \frac{\rho_{\min}^+(\mathcal{T})}{p_{\max}(\mathcal{T})} \mathcal{D}_\mathcal{T}(c), \forall{c \in \mathcal{T}}.$
\end{lemma}

\begin{lemma} \label{lma-k-adv-Lsup-le-adv-Lsun}
	Assume that $f$ satisfies Assumption \ref{ass-loss-function}. For any $f \in \mathcal{F}$, we have:
	\begin{equation*}
		\begin{aligned}
			&(1-\tau_k) \underset{\mathcal{T} \sim \mathcal{D}}{\mathbb{E}}\left[\frac{\rho_{\min}^+(\mathcal{T})}{p_{\max}(\mathcal{T})} \widetilde{L}_{sup}(\mathcal{T}, f)\right] \leq (1-\tau_k) \underset{\mathcal{T} \sim \mathcal{D}}{\mathbb{E}}\left[\frac{\rho_{\min}^+(\mathcal{T})}{p_{\max}(\mathcal{T})} \widetilde{L}_{sup}^\mu(\mathcal{T}, f)\right] \\[1mm]
			&\leq \widetilde{L}_{sun}(f) - \tau_k \underset{c^+,c_i^- \sim \rho^{k+1}}{\mathbb{E}}\left[\ell_{|I^+|}(\Vec{0})|I^+\neq \emptyset \right].
		\end{aligned}
	\end{equation*}
\end{lemma}

\begin{proof}[Proof of Lemma \ref{lma-rho-ge-DT}]
By the definition of $\rho^+(\mathcal{T})$ and $\rho_{\min}^+(\mathcal{T})$, we can find that:
\vskip -0.1in
\begin{equation*}
    \forall{c \in \mathcal{T}}, \rho_{\min}^+(\mathcal{T}) \leq \rho^+(\mathcal{T})(c).
\end{equation*}
\vskip -0.1in
So we have:
\begin{equation} \label{ieq-A7-1}
    \forall{c \in \mathcal{T}}, \frac{\rho^+(\mathcal{T})(c)}{\rho_{\min}^+(\mathcal{T})} \ge 1.
\end{equation}
By the definition of $p_{\max}(\mathcal{T})$, we have:
\begin{equation} \label{ieq-A7-2}
    \forall{c \in \mathcal{T}}, \frac{\mathcal{D}_\mathcal{T}(c)}{p_{\max}(\mathcal{T})} \leq 1.
\end{equation}
Combine \eqref{ieq-A7-1} and \eqref{ieq-A7-2}, we have:
\begin{equation*}
    \forall{c \in \mathcal{T}}, \rho^+(\mathcal{T})(c) \ge \frac{\rho_{\min}^+(\mathcal{T})}{p_{\max}(\mathcal{T})} \mathcal{D}_\mathcal{T}(c).
\end{equation*}
\end{proof}
\vskip -0.1in

\begin{proof}[Proof of Lemma \ref{lma-k-adv-Lsup-le-adv-Lsun}]
By the definition of $\widetilde{L}_{sun} (f)$, we have:
\vskip -0.1in
\begin{equation} \label{ieq-A8-1}
    \begin{aligned}
        \widetilde{L}_{sun} (f) &=  \underset{x, x^+ \sim \mathcal{D}_{c^+}^2; x_i^- \sim \mathcal{D}_{c_i^-}}{\underset{c^+, c_i^- \sim \rho^{k+1}}{\mathbb{E}}} \left[ \underset{x^\prime \in \mathcal{U}(x)}{\sup} \ell \left( \left\{ f(x^\prime)^T \left( f(x^+) - f(x_i^-) \right) \right\}_{i=1}^k \right) \right] \\[1mm]
        &= \underset{x \sim \mathcal{D}_{c^+}}{\underset{c^+, c_i^- \sim \rho^{k+1}}{\mathbb{E}}} \left\{ \underset{x_i^- \sim \mathcal{D}_{c_i^-}}{\underset{x^+ \sim \mathcal{D}_{c^+}}{\mathbb{E}}} \left[ \underset{x^\prime \in \mathcal{U}(x)}{\sup} \ell \left( \left\{ f(x^\prime)^T \left( f(x^+) - f(x_i^-) \right) \right\}_{i=1}^k \right) \right] \right\} \\[1mm]
        &\overset{\sroman{1}}{\ge} \underset{x \sim \mathcal{D}_{c^+}}{\underset{c^+, c_i^- \sim \rho^{k+1}}{\mathbb{E}}} \left\{ \underset{x^\prime \in \mathcal{U}(x)}{\sup} \left[ \underset{x_i^- \sim \mathcal{D}_{c_i^-}}{\underset{x^+ \sim \mathcal{D}_{c^+}}{\mathbb{E}}} \ell \left( \left\{ f(x^\prime)^T \left( f(x^+) - f(x_i^-) \right) \right\}_{i=1}^k \right) \right] \right\} \\[1mm]
        &\overset{\sroman{2}}{\ge} \underset{x \sim \mathcal{D}_{c^+}}{\underset{c^+, c_i^- \sim \rho^{k+1}}{\mathbb{E}}} \left[ \underset{x^\prime \in \mathcal{U}(x)}{\sup} \ell \left( \left\{ f(x^\prime)^T \left( \mu_{c^+} - \mu_{c_i^-} \right) \right\}_{i=1}^k \right) \right] \triangleq R_1,
    \end{aligned}
\end{equation}
\vskip -0.1in
where $\sroman{1},\sroman{2}$ is from the Jensen's inequality and convexity of $\ell$. Then we analyze lower bound of $R_1$:
\vskip -0.2in
\begin{equation} \label{ieq-A8-2}
    \begin{aligned}
        R_1 &\overset{\sroman{1}}{=} (1-\tau_k) \underset{x \sim \mathcal{D}_{c^+}}{\underset{c^+, c_i^- \sim \rho^{k+1}}{\mathbb{E}}} \left[ \underset{x^\prime \in \mathcal{U}(x)}{\sup} \ell \left( \left\{ f(x^\prime)^T \left( \mu_{c^+} - \mu_{c_i^-} \right) \right\}_{i=1}^k \right) \bigg| I^+ = \emptyset \right] \\[1mm]
        &\ \ \ \ \ \ \ \ \ \ \ \ \ \ \ \ \ \ \ \ + \tau_k \underset{x \sim \mathcal{D}_{c^+}}{\underset{c^+, c_i^- \sim \rho^{k+1}}{\mathbb{E}}} \left[ \underset{x^\prime \in \mathcal{U}(x)}{\sup} \ell \left( \left\{ f(x^\prime)^T \left( \mu_{c^+} - \mu_{c_i^-} \right) \right\}_{i=1}^k \right) \bigg| I^+ \neq \emptyset \right] \\[1mm]
        &\overset{\sroman{2}}{\ge} (1-\tau_k) \underset{x \sim \mathcal{D}_{c^+}}{\underset{c^+, c_i^- \sim \rho^{k+1}}{\mathbb{E}}} \left[ \underset{x^\prime \in \mathcal{U}(x)}{\sup} \ell \left( \left\{ f(x^\prime)^T \left( \mu_{c^+} - \mu_{c} \right) \right\}_{c \sim Q, c \neq c^+} \right) \bigg| I^+ = \emptyset \right] \\[1mm]
        &\ \ \ \ \ \ \ \ \ \ \ \ \ \ \ \ \ \ \ \ + \tau_k \underset{x \sim \mathcal{D}_{c^+}}{\underset{c^+, c_i^- \sim \rho^{k+1}}{\mathbb{E}}} \left[ \underset{x^\prime \in \mathcal{U}(x)}{\sup} \ell \left( \left\{ f(x^\prime)^T \left( \mu_{c^+} - \mu_{c_i^-} \right) \right\}_{i=1}^k \right) \bigg| I^+ \neq \emptyset \right] \\[1mm]
        &\overset{\sroman{3}}{\ge} (1-\tau_k) \underset{x \sim \mathcal{D}_{c^+}}{\underset{c^+, c_i^- \sim \rho^{k+1}}{\mathbb{E}}} \left[ \underset{x^\prime \in \mathcal{U}(x)}{\sup} \ell \left( \left\{ f(x^\prime)^T \left( \mu_{c^+} - \mu_{c} \right) \right\}_{c \sim Q, c \neq c^+} \right) \bigg| I^+ = \emptyset \right]  \\[1mm]
        &\ \ \ \ \ \ \ \ \ \ \ \ \ \ \ \ \ \ \ \ + \tau_k \underset{c^+, c_i^- \sim \rho^{k+1}}{\mathbb{E}} \left[ \ell_{|I^+|}(\Vec{0}) \bigg| I^+ \neq \emptyset \right],
    \end{aligned}
\end{equation}
where $\sroman{1}$ comes from the property of conditional expectation, $\sroman{2}$ is a result of the fact that $Q \subseteq [k]$ and $\ell$ satisfies Assumption \ref{ass-loss-function} and $\sroman{3}$ is from the fact that $[|I^+|] \subseteq [k]$ and $\ell$ satisfies Assumption \ref{ass-loss-function}. Let $R_2 = \underset{x \sim \mathcal{D}_{c^+}}{\underset{c^+, c_i^- \sim \rho^{k+1}}{\mathbb{E}}} \left[ \underset{x^\prime \in \mathcal{U}(x)}{\sup} \ell \left( \left\{ f(x^\prime)^T \left( \mu_{c^+} - \mu_{c} \right) \right\}_{c \sim Q, c \neq c^+} \right) \bigg| I^+ = \emptyset \right]$, we have:
\begin{equation} \label{ieq-A8-3}
    \begin{aligned}
        R_2 &= \underset{x \sim \mathcal{D}_{c^+}}{\underset{c^+, c_i^- \sim \rho^{k+1}}{\mathbb{E}}} \left[ \underset{x^\prime \in \mathcal{U}(x)}{\sup} \ell \left( \left\{ f(x^\prime)^T \left( \mu_{c^+} - \mu_{c} \right) \right\}_{c \sim Q, c \neq c^+} \right) \bigg| I^+ = \emptyset \right] \\[1mm]
        &\overset{\sroman{1}}{=} \underset{\mathcal{T} \sim \mathcal{D}}{E} \left\{ \underset{x \sim \mathcal{D}_{c^+}}{\underset{c^+, c_i^- \sim \rho^{k+1}}{\mathbb{E}}} \left[ \underset{x^\prime \in \mathcal{U}(x)}{\sup} \ell \left( \left\{ f(x^\prime)^T \left( \mu_{c^+} - \mu_{c} \right) \right\}_{c \sim Q, c \neq c^+} \right) \bigg| Q = \mathcal{T} , I^+ = \emptyset \right] \right\} \\[1mm]
        &\overset{\sroman{2}}{=} \underset{\mathcal{T} \sim \mathcal{D}}{E} \left\{ \underset{x \sim \mathcal{D}_{c^+}}{\underset{c^+ \sim \rho^+(\mathcal{T})}{\mathbb{E}}} \left[ \underset{x^\prime \in \mathcal{U}(x)}{\sup} \ell \left( \left\{ f(x^\prime)^T \left( \mu_{c^+} - \mu_{c} \right) \right\}_{c \sim \mathcal{T}, c \neq c^+} \right) \right] \right\},
    \end{aligned}
\end{equation}
\vskip -0.1in
where $\sroman{1}$ is from the tower property of expectation and $\sroman{2}$ is directly obtained by the definition of $\rho^+(\mathcal{T})$. Let $R_3 = \underset{x \sim \mathcal{D}_{c^+}}{\underset{c^+ \sim \rho^+(\mathcal{T})}{\mathbb{E}}} \left[ \underset{x^\prime \in \mathcal{U}(x)}{\sup} \ell \left( \left\{ f(x^\prime)^T \left( \mu_{c^+} - \mu_{c} \right) \right\}_{c \sim \mathcal{T}, c \neq c^+} \right) \right]$, we have:
\vskip -0.2in
\begin{equation} \label{ieq-A8-4}
    \begin{aligned}
        R_3 &= \underset{x \sim \mathcal{D}_{c^+}}{\underset{c^+ \sim \rho^+(\mathcal{T})}{\mathbb{E}}} \left[ \underset{x^\prime \in \mathcal{U}(x)}{\sup} \ell \left( \left\{ f(x^\prime)^T \left( \mu_{c^+} - \mu_{c} \right) \right\}_{c \sim \mathcal{T}, c \neq c^+} \right) \right] \\[1mm]
        &\overset{\sroman{1}}{\ge} \frac{\rho_{\min}^+(\mathcal{T})}{p_{\max}(\mathcal{T})}\ \  \underset{x \sim \mathcal{D}_{c^+}}{\underset{c^+ \sim \mathcal{D}_\mathcal{T}}{\mathbb{E}}} \left[ \underset{x^\prime \in \mathcal{U}(x)}{\sup} \ell \left( \left\{ f(x^\prime)^T \left( \mu_{c^+} - \mu_{c} \right) \right\}_{c \sim \mathcal{T}, c \neq c^+} \right) \right] \\[1mm]
        &\overset{\sroman{2}}{=} \frac{\rho_{\min}^+(\mathcal{T})}{p_{\max}(\mathcal{T})} \ \ \widetilde{L}_{sup}^\mu (\mathcal{T}, f) \overset{\sroman{3}}{\ge} \frac{\rho_{\min}^+(\mathcal{T})}{p_{\max}(\mathcal{T})} \ \ \widetilde{L}_{sup} (\mathcal{T}, f),
    \end{aligned}
\end{equation}
where $\sroman{1}$ is directly from Lemma \ref{lma-rho-ge-DT};$\sroman{2}$ and $\sroman{3}$ are from the definition of $\widetilde{L}_{sup}^\mu (\mathcal{T}, f)$ and $\widetilde{L}_{sup} (\mathcal{T}, f)$, respectively.
Combing \eqref{ieq-A8-1} \eqref{ieq-A8-2} \eqref{ieq-A8-3} and \eqref{ieq-A8-4} yields:
\begin{equation*}
    \begin{aligned}
        (1-\tau_k) \underset{\mathcal{T} \sim \mathcal{D}}{\mathbb{E}}\left[\frac{\rho_{\min}^+(\mathcal{T})}{p_{\max}(\mathcal{T})} \widetilde{L}_{sup}(\mathcal{T}, f)\right]
            &\leq (1-\tau_k) \underset{\mathcal{T} \sim \mathcal{D}}{\mathbb{E}}\left[\frac{\rho_{\min}^+(\mathcal{T})}{p_{\max}(\mathcal{T})} \widetilde{L}_{sup}^\mu(\mathcal{T}, f)\right] \\[1mm]
           &\leq \widetilde{L}_{sun}(f) - \tau_k \underset{c^+,c_i^- \sim \rho^{k+1}}{\mathbb{E}}\left[\ell_{|I^+|}(\Vec{0})|I^+\neq \emptyset \right].
    \end{aligned}
\end{equation*}
\end{proof}

Equipped with the above lemmas, now we can turn to the proof of Theorem \ref{thm-k-general-bound}.
\begin{proof}[Proof of Theorem \ref{thm-k-general-bound}]
	From Lemma \ref{lma-k-adv-Lun(fhat)-leq-adv-Lun(f)} we know that with probability at least $1-\delta$ over the choice of the training set $\mathcal{S}$, $\forall{f \in \mathcal{F}}$:
	$$\widetilde{L}_{sun}(\widehat{f}) \leq \widetilde{L}_{sun}(f) + AG_M.$$
	Combing this with Lemma \ref{lma-k-adv-Lsup-le-adv-Lsun} directly yields \eqref{eq-k-general-bound}.
\end{proof}

\subsection{Proof of Theorem \ref{thm-of-thm-k-general-bound}} \label{apd-prf-thm-of-thm-k-general-bound}
\begin{customthm}{{\ref{thm-of-thm-k-general-bound}}}
    Suppose $\mathcal{C}$ is finite. For any $c \in \mathcal{C}$, $\rho(c) > 0$, and $\ell$ satisfies Assumption \ref{ass-loss-function}. Then, with a probability of at least $1-\delta$ over the choice of the training set $\mathcal{S}$, $\forall{f \in \mathcal{F}}$:
    \begin{equation*}
        \widetilde{\mathcal{L}}_{sup}(\widehat{f}) \leq \alpha(\rho)\left(\widetilde{L}_{sun}(f) + AG_M\right) - \beta,
    \end{equation*}
    where $\alpha(\rho) = \frac{1}{1-\tau_k} \underset{T distinct}{\underset{|\mathcal{T}|\leq k+1}{\max}}\frac{p_{\max}(\mathcal{T})}{\rho_{\min}^+(\mathcal{T})}$ is a positive constant depending on $\rho$ and
    \begin{equation*}
        \beta = \alpha(\rho) \tau_k \underset{c^+,c_i^- \sim \rho^{k+1}}{\mathbb{E}}\left[\ell_{|I^+|}(\Vec{0})|I^+\neq \emptyset\right].
    \end{equation*}
\end{customthm}

\begin{proof}
Since $\mathcal{C}$ is finite and $\rho(c) > 0$ for any $c \in \mathcal{C}$, we can see that:
\begin{equation} \label{eq-A9-1}
    \frac{1}{\alpha (1-\tau_k)} = \underset{|\mathcal{T}| \leq k+1, \mathcal{T} \text{distinct}}{\min} \ \ \left\{ \frac{\rho_{\min}^+(\mathcal{T})}{p_{\max}(\mathcal{T})} \right\}  > 0.
\end{equation}
By Theorem \ref{thm-k-general-bound} and the definition of $\widetilde{\mathcal{L}}_{sup}(f)$, we have: $\forall{f \in \mathcal{F}}$, with probability at least $1-\delta$ over the choice of the training set $\mathcal{S}$, 
\begin{equation*}
    \begin{aligned}
        \frac{1}{\alpha (1-\tau_k)} \widetilde{\mathcal{L}}_{sup}(\widehat{f}) &\overset{\sroman{1}}{=} \underset{\mathcal{T} \sim \mathcal{D}}{\mathbb{E}}\left[ \frac{1}{\alpha (1-\tau_k)} \widetilde{L}_{sup}(\mathcal{T}, \widehat{f})\right] \overset{\sroman{2}}{\leq} \underset{\mathcal{T} \sim \mathcal{D}}{\mathbb{E}}\left[\frac{\rho_{\min}^+(\mathcal{T})}{p_{\max}(\mathcal{T})} \widetilde{L}_{sup}(\mathcal{T}, \widehat{f})\right] \\[1mm]
        &\overset{\sroman{3}}{\leq} \frac{1}{1-\tau_k}\left(\widetilde{L}_{sun}(f) + AG_M\right)-\frac{\tau_k}{1-\tau_k} \ \ \underset{c^+,c_i^- \sim \rho^{k+1}}{\mathbb{E}}\left[\ell_{|I^+|}(\Vec{0})|I^+\neq \emptyset\right],
    \end{aligned}
\end{equation*}
where $\sroman{1}$ is from the definition of $\widetilde{\mathcal{L}}_{sup}(\widehat{f})$;$\sroman{2}$ is from \eqref{eq-A9-1} and $\sroman{3}$ is from Theorem \ref{thm-k-general-bound}. Rearranging yields the result.
\end{proof}

\subsection{Proof of Lemma \ref{lma-unit-ball-covering-number-bound}} \label{apd-prf-lma-unit-ball-covering-number-bound}
\begin{customlemma}{{\ref{lma-unit-ball-covering-number-bound}}}
Let $\mathbb{B}_p(r)$ be the $p$-norm ball in $\mathbb{R}^d$ with radius $r$. The $\delta$-covering number of $\mathbb{B}_p(r)$ with respect to $\Vert \cdot \Vert_p$ thus obeys the following bound:
$$
\mathcal{N}(\delta;\mathbb{B}_p(r),\Vert \cdot \Vert_p) \leq \left( 1 + \frac{2r}{\delta} \right)^d,
$$
where $\mathcal{N}(\delta;\mathbb{B},\Vert \cdot \Vert)$ is the $\delta$-covering number of $\mathbb{B}$ with respect to the norm $\Vert \cdot \Vert$.
\end{customlemma}
\begin{proof}
Set $\normvec{\cdot},\normvec{\cdot}^\prime$ in Lemma \ref{lma-hds-lma5.7} to be $\normvec{\cdot}_p$, we have:
\begin{equation*}
    \begin{aligned}
        \mathcal{N}(\delta;\mathbb{B}_p(1),\normvec{\cdot}_p) &\leq \frac{vol \left( \frac{2}{\delta} \mathbb{B}_p(1) + \mathbb{B}_p(1) \right)}{vol \left( \mathbb{B}_p(1) \right)} = \frac{vol \left( \left(1+\frac{2}{\delta} \right) \mathbb{B}_p(1) \right)}{vol \left( \mathbb{B}_p(1) \right)} \\[1mm]
        &\overset{\sroman{1}}{\leq} \left( 1+\frac{2}{\delta} \right)^d \frac{vol \left( \mathbb{B}_p(1) \right)}{vol \left( \mathbb{B}_p(1) \right)} = \left( 1+\frac{2}{\delta} \right)^d,
    \end{aligned}
\end{equation*}
where $\sroman{1}$ is true because $\mathbb{B}_p(1) \subset \mathbb{R}^d$.
Now suppose that $\{x^1,\cdots,x^N\}$ is the minimal $\delta$-covering of $\mathbb{B}_p(1)$, then:
\begin{equation*}
    \forall{x \in \mathbb{B}_p(1)}, \exists{x^i \in \{x^1,\cdots,x^N\}} \ \ s.t.  \normvec{x - x^i}_p \leq \delta.
\end{equation*}
So we have:
\begin{equation*}
    \forall{rx \in \mathbb{B}_p(r)}, \exists{rx^i \in \{rx^1,\cdots,rx^N\}} \ \ s.t.  \normvec{rx - rx^i}_p \leq r\delta.
\end{equation*}
So $\{rx^1,\cdots,rx^N\}$ is a $r\delta$-covering of $\mathbb{B}_p(r)$, so we have:
\begin{equation*}
    \mathcal{N}(r\delta;\mathbb{B}_p(r),\normvec{\cdot}_p) \leq \mathcal{N}(\delta;\mathbb{B}_p(1),\normvec{\cdot}_p) \leq \left( 1+\frac{2}{\delta} \right)^d.
\end{equation*}
Letting $\frac{\delta}{r}$ take place of $\delta$, we have $\mathcal{N}(\delta;\mathbb{B}_p(r),\normvec{\cdot}_p) \leq \left( 1+\frac{2r}{\delta} \right)^d.$
\end{proof}
\subsection{Proof of Theorem \ref{thm-linear-Rs(H)-bound}} \label{apd-prf-thm-linear-Rs(H)-bound}
\begin{customthm}{{\ref{thm-linear-Rs(H)-bound}}}
Let $\mathcal{U}(x)=\left\{ x^\prime | \normvec{x^\prime - x}_r \leq \epsilon \right\}$ (i.e. consider the $\ell_r$ attack). We then have :
\begin{equation*}
    \begin{aligned}
        &\mathcal{R}_\mathcal{S}(\mathcal{H}) = O\left( \left[ PP^* + \epsilon R^* s(r^*,p,m) \right]\left[ m s(p^*,p,m) w^2 \sqrt{M} \right]  \right),
    \end{aligned}
\end{equation*}
where $s(p,q,n) \coloneqq n^{\max\left\{\frac{1}{p}-\frac{1}{q},\frac{1}{q}-\frac{1}{p} \right\}}$, $\frac{1}{p} + \frac{1}{p^*} = 1,\frac{1}{r} + \frac{1}{r^*} = 1$, and $R^*$ is defined similarly to \eqref{eq-def-P}.
\end{customthm}
Before giving the proof of the theorem, we introduce some lemmas will be used in our proof.
\begin{lemma} \label{lma-vec-norm}
For any $x \in \mathbb{R}^n$,$0 < p_2 \leq p_1$, we have: $\normvec{x}_{p_1} \leq \normvec{x}_{p_2} \leq n^{\frac{1}{p_2}-\frac{1}{p_1}}\normvec{x}_{p_1}.$
\end{lemma}
\begin{proof}[Proof of Lemma \ref{lma-vec-norm}]
Firstly, we prove $\normvec{x}_{p_1} \leq \normvec{x}_{p_2}$.
For any $x \in \mathbb{R}^n$, suppose $a_i = |x_i|$, lep $f(p) = \left( \sum_{i=1}^n a_i^p \right)^{\frac{1}{p}}$. In order to prove $\normvec{x}_{p_1} \leq \normvec{x}_{p_2}$, it suffices to prove that: $\forall{a_i \ge 0, i=1,2,\dots,n}, f(p)$ is non-increasing on $(0,1]$.
\begin{enumerate}
    \item if $a_i = 0, \forall{i}$, f is a constant function, so f is non-increasing.
    \item if $\exists{i}, a_i \neq 0$, suppose $\{a_i| a_i \neq 0\}$ has $K$ elements, without loss of generality, suppose $\{a_i| a_i \neq 0\} = \{a_1, a_2,\cdots, a_K \}$ and $a_1 = \underset{1\leq k \leq K}{\max} \{a_k\}$, then we have:
    \begin{equation} \label{eq-A11-1}
        0 < \left( \frac{a_k}{a_1} \right)^p \leq 1, \sum_{k=1}^K \left( \frac{a_k}{a_1} \right)^p \ge 1, ln \left( \sum_{k=1}^K \left( \frac{a_k}{a_1} \right)^p \right) \ge 0, k=1,2,\dots,K.
    \end{equation}
    We can write $f(p) = \left( \sum_{i=1}^n a_i^p \right)^{\frac{1}{p}} = a_1 \left( \sum_{k=1}^K \left( \frac{a_k}{a_1} \right)^p \right)^{\frac{1}{p}} = a_1 exp\{ \frac{1}{p} ln \left( \sum_{k=1}^K \left( \frac{a_k}{a_1} \right)^p \right) \}$, let $g(p) =  \frac{1}{p} ln \left( \sum_{k=1}^K \left( \frac{a_k}{a_1} \right)^p \right)$, by the monotone property of composite functions, to prove $f(p)$ is non-increasing, it suffices to prove $g(p)$ is non-increasing.Taking derivation of $g$ yields:
    \begin{equation*}
        g^\prime(p) = \frac{\sum_{k=1}^K \left[ \left( \frac{a_k}{a_1} \right)^p ln \left( \frac{a_k}{a_1} \right) \right]}{p \sum_{k=1}^K \left( \frac{a_k}{a_1} \right)^p} - \frac{ln \left( \sum_{k=1}^K \left( \frac{a_k}{a_1} \right)^p \right)}{p^2}.
    \end{equation*}
    From \eqref{eq-A11-1} we know that $\frac{\sum_{k=1}^K \left[ \left( \frac{a_k}{a_1} \right)^p ln \left( \frac{a_k}{a_1} \right) \right]}{p \sum_{k=1}^K \left( \frac{a_k}{a_1} \right)^p} \leq 0$ and $\frac{ln \left( \sum_{k=1}^K \left( \frac{a_k}{a_1} \right)^p \right)}{p^2} \ge 0$, so we have $g^\prime(p) \leq 0$, so $g$ is non-increasing, which means that $\normvec{x}_{p_1} \leq \normvec{x}_{p_2}, \forall{x \in \mathbb{R}^n}$.
\end{enumerate}
Nextly, we prove that $\normvec{x}_{p_2} \leq n^{\frac{1}{p_2}-\frac{1}{p_1}}\normvec{x}_{p_1}$.
By the definition of $\normvec{\cdot}_p$, we have: $\normvec{x}_{p_2}^{p_1} = \left( \sum_{i=1}^n |x_i|^{p_2} \right)^{\frac{p_1}{p_2}} = \left( \frac{1}{n} \sum_{i=1}^n |x_i|^{p_2} \right)^{\frac{p_1}{p_2}} \cdot n^{\frac{p_1}{p_2}}$.Since $p_2 \leq p_1$, i.e. $\frac{p_1}{p_2} \ge 1$, we know the function $h(t) = t^{\frac{p_1}{p_2}}$ is convex. By Jensen's Inequality we know that:
\begin{equation} \label{ieq-A11-2}
    \forall{\lambda_i\  s.t. \ \lambda_i \ge 0\  and\  \sum_{i=1}^n \lambda_i = 1},\forall{t_i \in \mathbb{R}}: h\left( \sum_{i=1}^n \lambda_i t_i \right) \leq \sum_{i=1}^n \lambda_i h(t_i).
\end{equation}
Set $\lambda_1 = \lambda_2 = \cdots = \lambda_n = \frac{1}{n}, t_i = |x_i|^{p_2}, i = 1,2,\dots,n$ in \eqref{ieq-A11-2}, we have:
\begin{equation*}
    \left( \frac{1}{n} \sum_{i=1}^n |x_i|^{p_2} \right)^{\frac{p_1}{p_2}} \leq \frac{1}{n} \sum_{i=1}^n \left( |x_i|^{p_2} \right)^{\frac{p_1}{p_2}} = \frac{1}{n} \sum_{i=1}^n |x_i|^{p_1}.
\end{equation*}
Multiplying both sides by $n^{\frac{p_1}{p_2}}$ yields $\left( \frac{1}{n} \sum_{i=1}^n |x_i|^{p_2} \right)^{\frac{p_1}{p_2}} \cdot n^{\frac{p_1}{p_2}} \leq \frac{1}{n} \sum_{i=1}^n |x_i|^{p_1} \cdot n^{\frac{p_1}{p_2}},$
i.e.
\begin{equation*}
    \normvec{x}_{p_2}^{p_1} \leq n^{\frac{p_1}{p_2} - 1} \cdot \normvec{x}_{p_1}^{p_1}.
\end{equation*}
Taking both sides the $p_1$th power yields $\normvec{x}_{p_2} \leq n^{\frac{1}{p_2}-\frac{1}{p_1}}\normvec{x}_{p_1}.$
\end{proof}
\begin{lemma} \label{lma-vec-norm-to-mat-norm}
Suppose that $\exists{A, B}\in \mathbb{R^+} \  s.t. \  A \normvec{x}_{p} \leq \normvec{x}_{q} \leq B \normvec{x}_{p} \  (p,q \ge 1)$ for any $x \in \mathbb{R}^n$, then $\forall{W \in \mathbb{R}^{m\times n}}$: $\normmm{W}_p \leq \frac{B}{A} \normmm{W}_q.$
\end{lemma}
\begin{proof}[Proof of Lemma \ref{lma-vec-norm-to-mat-norm}]
If we have $A \normvec{x}_{p} \leq \normvec{x}_{q} \leq B \normvec{x}_{p}$, then:
\begin{equation*}
    \begin{aligned}
        \normmm{W}_p &= \underset{\normvec{x}_p \leq 1}{\max} \normvec{W x}_p \overset{\sroman{1}}{\leq} \underset{\normvec{x}_p \leq B}{\max} \normvec{W x}_p \overset{\sroman{2}}{\leq} \underset{\normvec{x}_p \leq B}{\max} \frac{1}{A} \normvec{W x}_q = \underset{\normvec{x}_p \leq 1}{\max} \frac{B}{A} \normvec{W x}_q = \frac{B}{A} \normmm{W}_q,
    \end{aligned}
\end{equation*}
where $\sroman{1}$ is from the fact that $\normvec{x}_{q} \leq B \normvec{x}_{p}$, which means that $\{ x | \normvec{x}_p \leq 1 \} \subseteq \{ x | \normvec{x}_q \leq 1 \}$ and $\sroman{2}$ is from the condition that $A \normvec{x}_{p} \leq \normvec{x}_{q}$.
\end{proof}
\begin{lemma} \label{lma-norm-difference}
For any $p, q \ge 1$, then $\forall{W \in \mathbb{R}^{m\times n}}$: $\normmm{W}_q \leq n^{\max\{\frac{1}{p} - \frac{1}{q}, \frac{1}{q} - \frac{1}{p} \}} \normmm{W}_p.$
\end{lemma}
\begin{remark}
If we denote $n^{\max\{\frac{1}{p} - \frac{1}{q} , \frac{1}{q} - \frac{1}{p}\}}$ by $s(p,q,n)$, then we have: $\forall{W \in \mathbb{R}^{m\times n}}, \normmm{W}_q \leq s(p,q,n) \normmm{W}_p$.
\end{remark}
\begin{proof} [Proof of Lemma \ref{lma-norm-difference}]
We discuss it in two cases.
\begin{enumerate}
    \item If $p \leq q$, by Lemma \ref{lma-vec-norm}, we have: $\forall{x \in \mathbb{R}^n}$, $\normvec{x}_{q} \leq \normvec{x}_{p} \leq n^{\frac{1}{p}-\frac{1}{q}}\normvec{x}_{q}.$
    By Lemma \ref{lma-vec-norm-to-mat-norm}, we know that $\forall{W \in \mathbb{R}^{m \times n}}$, $\normmm{W}_q \leq n^{\frac{1}{p} - \frac{1}{q}} \normmm{W}_p.$
    \item If $p \ge q$, by Lemma \ref{lma-vec-norm}, we have: $\forall{x \in \mathbb{R}^n}$, $\normvec{x}_{p} \leq \normvec{x}_{q} \leq n^{\frac{1}{q}-\frac{1}{p}}\normvec{x}_{p},$
    i.e.
    \begin{equation*}
        n^{\frac{1}{p}-\frac{1}{q}} \normvec{x}_{q} \leq \normvec{x}_{p} \leq \normvec{x}_{q}.
    \end{equation*}
    By Lemma \ref{lma-vec-norm-to-mat-norm}, we know that $\forall{W \in \mathbb{R}^{m \times n}}$, $\normmm{W}_q \leq n^{\frac{1}{q} - \frac{1}{p}} \normmm{W}_p.$
\end{enumerate}
So we conclude that $\forall{W \in \mathbb{R}^{m \times n}}, \normmm{W}_q \leq n^{\max\{\frac{1}{p} - \frac{1}{q} , \frac{1}{q} - \frac{1}{p}\}} \normmm{W}_p$.
\end{proof}
Now we turn to bound the Rademacher complexity of $\mathcal{H}$, where
\begin{equation*}
    \begin{aligned}
        \mathcal{H} &= \left\{ h(x,x^+,x^-) = \underset{x^\prime \in \mathcal{U}(x)}{\min} \left( f(x^\prime)^T\left( f(x^+) - f(x^-) \right) \right) |f \in \mathcal{F} \right\}, \\[1mm]
        \mathcal{F} &= \{f: x \xrightarrow[]{} Wx | W \in \mathbb{R}^{n\times m},\normmm{W}_p \leq w\}.
    \end{aligned}
\end{equation*}
Let 
\begin{equation*}
    \begin{aligned}
        \mathcal{H}_0 &= \left\{ h(x,x^+,x^-) = f(x)^T\left( f(x^+) - f(x^-) \right)|f \in \mathcal{F} \right\}\\[1mm]
        &= \left\{ (x,x^+,x^-)\xrightarrow[]{} x^TW^TW(x^+-x^-) \big| W \in \mathbb{R}^{n \times m},\normmm{W}_p \leq w \right\}.
    \end{aligned}
\end{equation*}

\begin{proof} [Proof of Theorem \ref{thm-linear-Rs(H)-bound}]
Firstly, we drop the complicated and unwieldy min operate by directly solving the minimization problem.
\begin{equation*}
    \begin{aligned}
        h(x,x^+,x^-) &= \underset{\normvec{x^\prime - x}_r \leq \epsilon}{\min} f(x^\prime)^T \left( f(x^+) - f(x^-) \right) = \underset{\normvec{\delta}_r \leq \epsilon}{\min} \left( W(x+\delta) \right)^T \left( Wx^+ - Wx^- \right) \\[1mm]
        &= x^TW^TW(x^+-x^-) + \underset{\normvec{\delta}_r \leq \epsilon}{\min} \delta^T W^T W(x^+ - x^-) \\[1mm]
        &\overset{\sroman{1}}{=} x^TW^TW(x^+-x^-) - \epsilon \normvec{W^TW(x^+-x^-)}_{r^*},
    \end{aligned}
\end{equation*}
where $\sroman{1}$ is from Holder's Inequality and taking the value of $\delta$ that can get equality.
Define:
\begin{equation*}
    \begin{aligned}
        \mathcal{H}_1 &= \left\{ (x,x^+,x^-) \xrightarrow[]{} \normvec{W^TW(x^+-x^-)}_{r^*} \big| W \in \mathbb{R}^{n \times m},\normmm{W}_p \leq w \right\}, \\[1mm]
        \mathcal{H}_0^w &= \left\{ (x,x^+,x^-)\xrightarrow[]{} x^TA(x^+-x^-) \big| A \in \mathbb{R}^{m \times m},\normmm{A}_p \leq w \right\}, \\[1mm]
        \mathcal{H}_1^w &= \left\{ (x,x^+,x^-) \xrightarrow[]{} \normvec{A(x^+-x^-)}_{r^*} \big| A \in \mathbb{R}^{m \times m},\normmm{A}_p \leq w \right\}.
    \end{aligned}
\end{equation*}
Since $\forall{W \in \mathbb{R}^{n \times m}} \ s.t.\  \normmm{W}_p \leq w$, we have:
\begin{equation*}
    \normmm{W^TW}_p \overset{\sroman{1}}{\leq} \normmm{W^T}_p \cdot \normmm{W}_p = \normmm{W}_{p^*} \cdot \normmm{W}_p \overset{\sroman{2}}{\leq} s(p^*,p,m) \normmm{W}_p^2 \leq s(p^*,p,m) w^2,
\end{equation*}
where $\frac{1}{p} + \frac{1}{p^*} = 1$ and $\sroman{1}$ comes from the Submultiplicativity of matrix norm and $\sroman{2}$ is from Lemma \ref{lma-norm-difference}. So we have:
\begin{equation*}
    \left\{ W \Big| \normmm{W}_p \leq w \right\} \subseteq \left\{ W \Big| \normmm{W^TW}_p \leq s(p^*,p,m) w^2 \right\}.
\end{equation*}
And since $W^TW$ is positive semi-definite, it's easy to see that:
\begin{equation*}
    \left\{ W^TW \Big| \normmm{W^TW}_p \leq s(p^*,p,m) w^2 \right\} \subseteq \left\{ A \Big| \normmm{A}_p \leq s(p^*,p,m) w^2 \right\}.
\end{equation*}
So we have:
\begin{equation*}
    \begin{aligned}
        \mathcal{H}_0 &= \left\{ (x,x^+,x^-)\xrightarrow[]{} x^TW^TW(x^+-x^-) \big| W \in \mathbb{R}^{n \times m},\normmm{W}_p \leq w \right\} \\[1mm]
        &\subseteq \left\{ (x,x^+,x^-)\xrightarrow[]{} x^TW^TW(x^+-x^-) \big| W \in \mathbb{R}^{n \times m},\normmm{W^TW}_p \leq s(p^*,p,m) w^2 \right\} \\[1mm]
        &\subseteq \left\{ (x,x^+,x^-)\xrightarrow[]{} x^TA(x^+-x^-) \big| A \in \mathbb{R}^{m \times m},\normmm{A}_p \leq s(p^*,p,m) w^2 \right\} = \mathcal{H}_0^{s(p^*,p,m) w^2}.
    \end{aligned}
\end{equation*}
Similarly, $\mathcal{H}_1 \subseteq \mathcal{H}_1^{s(p^*,p,m) w^2}.$
Given the training set $\mathcal{S}=\{(x_i,x_i^+,x_i^-)\}_{i=1}^M$, we have:
\begin{equation} \label{ieq-dvd-Rs(H)}
    \begin{aligned}
        \mathcal{R}_\mathcal{S}(\mathcal{H}) &= \underset{\bm{\sigma}}{\mathbb{E}} \left\{ \underset{\normmm{W}_p \leq w}{\sup} \sum_{i=1}^M \bm{\sigma}_i \left[ x_i^TW^TW(x_i^+ - x_i^-) - \epsilon \normvec{W^TW(x_i^+ - x_i^-)}_{r^*}  \right] \right\} \\[1mm]
        &\overset{\sroman{1}}{\leq} \underset{\bm{\sigma}}{\mathbb{E}} \left\{ \underset{\normmm{W}_p \leq w}{\sup} \sum_{i=1}^M \bm{\sigma}_i \left[ x_i^TW^TW(x_i^+ - x_i^-) \right] \right\}\\[1mm]
        &\ \ \ \ \ \ \ \ \ \ \ \ \ \ \ \ \ \ \ \ \ \ \ \ \ + \epsilon \  \underset{\bm{\sigma}}{\mathbb{E}} \left\{ \underset{\normmm{W}_p \leq w}{\sup} \sum_{i=1}^M \left[ -\bm{\sigma}_i \normvec{W^TW(x_i^+ - x_i^-)}_{r^*} \right] \right\} \\[1mm]
        &\overset{\sroman{2}}{=} \underset{\bm{\sigma}}{\mathbb{E}} \left\{ \underset{\normmm{W}_p \leq w}{\sup} \sum_{i=1}^M \bm{\sigma}_i \left[ x_i^TW^TW(x_i^+ - x_i^-) \right] \right\}\\[1mm]
        &\ \ \ \ \ \ \ \ \ \ \ \ \ \ \ \ \ \ \ \ \ \ \ \ \ + \epsilon \  \underset{\bm{\sigma}}{\mathbb{E}} \left\{ \underset{\normmm{W}_p \leq w}{\sup} \sum_{i=1}^M \left[\bm{\sigma}_i \normvec{W^TW(x_i^+ - x_i^-)}_{r^*} \right] \right\} \\[1mm]
        &\overset{\sroman{3}}{=} \mathcal{R}_\mathcal{S}(\mathcal{H}_0) + \epsilon \  \mathcal{R}_\mathcal{S}(\mathcal{H}_1) \overset{\sroman{4}}{\leq} \mathcal{R}_\mathcal{S}(\mathcal{H}_0^{s(p^*,p,m) w^2}) + \epsilon \  \mathcal{R}_\mathcal{S}(\mathcal{H}_1^{s(p^*,p,m) w^2}),
    \end{aligned}
\end{equation}
where $\bm{\sigma}$ is a Random Vector whose elements are i.i.d. Rademacher Random Variables and $\sroman{1}$ comes from the subadditivity of sup function;$\sroman{2}$ is from the fact that $-\bm{\sigma}_i$ has the same distribution as $\bm{\sigma}_i$, $\forall{i = 1,2,\dots,M}$; $\sroman{3}$ is by the definition of $\mathcal{R}_\mathcal{S}(\mathcal{H}_0)$ and $\mathcal{R}_\mathcal{S}(\mathcal{H}_1)$ and $\sroman{4}$ is from the monotone property of Rademacher Complexity and the fact that $\mathcal{H}_0 \subseteq \mathcal{H}_0^{s(p^*,p,m) w^2}$ and $\mathcal{H}_1 \subseteq \mathcal{H}_1^{s(p^*,p,m) w^2}$.
Secondly, we upper bound the Rademacher Complexity of $\mathcal{H}_0^w$ and $\mathcal{H}_1^w$.
\textbf{For} \bm{$\mathcal{R}_\mathcal{S}(\mathcal{H}_0^w)$}:
Given the training set $\mathcal{S}=\{(x_i,x_i^+,x_i^-)\}_{i=1}^M = \{(z_i)\}_{i=1}^M$, define the $\ell_2$-norm for a function in $\mathcal{H}_0^w$ as:
\begin{equation*}
    \forall{h \in \mathcal{H}_0^w}, \  \normvec{h}_2 \coloneqq \sqrt{\sum_{i=1}^M \left[ h(z_i) \right]^2},
\end{equation*}
Define $\mathcal{H}_0^w(\mathcal{S}) = \left\{\left(h(z_1),h(z_2),\dots,h(z_M)\right) \Big| h \in \mathcal{H}_0^w \right\}$, we have that for any $h \in \mathcal{H}_0^w$ and $v_h = \left(h(z_1),h(z_2),\dots,h(z_M)\right)$ be the corresponding vector in $\mathcal{H}_0^w(\mathcal{S})$, we have:
\begin{equation*}
    \normvec{h}_2 = \normvec{v_h}_2.
\end{equation*}
So we know that any $\delta$-covering of $\mathcal{H}_0^w$ ($\left\{ h^1, h^2, \cdots, h^N \right\}$) with respect to $\ell_2$ norm in the functional space, corresponds to a $\delta$-covering of $\mathcal{H}_0^w(\mathcal{S})$ with respect to the $\ell_2$ norm in the Euclidean Space, i.e.
\begin{equation*}
    \left\{ \left( \begin{array}{c}  h^1(z_1)  \\  h^1(z_2) \\  \vdots \\  h^1(z_M) \end{array} \right) ,
    \left( \begin{array}{c}  h^2(z_1)  \\  h^2(z_2) \\  \vdots \\  h^2(z_M) \end{array} \right), \cdots, 
    \left( \begin{array}{c}  h^N(z_1)  \\  h^N(z_2) \\  \vdots \\  h^N(z_M) \end{array} \right) \right\}.
\end{equation*}
So we have: 
\begin{equation} \label{eq-A11-3}
    \mathcal{N}\left(\delta; \mathcal{H}_0^w(\mathcal{S}), \normvec{\cdot}_2 \right) = \mathcal{N}\left(\delta; \mathcal{H}_0^w, \normvec{\cdot}_2 \right).
\end{equation}
By the definition of Rademacher Complexity, we know that $\mathcal{R}_\mathcal{S}(\mathcal{H}_0^w)$ is just the expectation of the Rademacher Process with respect to $\mathcal{H}_0^w(\mathcal{S})$, which is $\mathbb{E}[\underset{\theta \in \mathcal{H}_0^w(\mathcal{S})}{\sup} X_\theta]$.
To use Lemma \ref{lma-dudley-int}, we must show that Rademacher Process is a sub-Gaussian Process with respect to some metric $\rho_X$. Denote the Euclidean metric by $\rho_2$, we have: for Rademacher Process $\{X_\theta, \theta \in \mathbb{T}\}$, $\forall{\theta, \widetilde{\theta} \in \mathbb{T} \ and \ \lambda \in \mathbb{R}}$:
\begin{equation*}
    \begin{aligned}
        \mathbb{E} \left[ e^{\lambda (X_\theta-X_{\widetilde{\theta}})} \right] &\overset{\sroman{1}}{=} \mathbb{E} \left[ e^{\lambda \left(\left<\bm{\sigma},\theta\right>-\left<\bm{\sigma},\widetilde{\theta}\right>\right)} \right] = \mathbb{E} \left[ e^{\lambda \left(\left<\bm{\sigma},\theta-\widetilde{\theta}\right>\right)} \right] \overset{\sroman{2}}{\leq} \prod_{i=1}^M \underset{\bm{\sigma}_i}{\mathbb{E}} \left[ e^{\lambda \bm{\sigma}_i (\theta_i - \widetilde{\theta}_i)} \right] \\[1mm]
        &\overset{\sroman{3}}{\leq} \prod_{i=1}^M e^{\frac{ \lambda^2 (\theta_i - \widetilde{\theta}_i)^2}{2}} = e^{\frac{\lambda^2}{2} \sum_{i=1}^M (\theta_i - \widetilde{\theta}_i)^2 } = e^{\frac{\lambda^2}{2} \rho_X^2(\theta,\widetilde{\theta}) },
    \end{aligned}
\end{equation*}
where $\bm{\sigma}$ is a Random Vector whose elements are i.i.d. Rademacher Random Variables and $\sroman{1}$ is from the definition of Rademacher Process; $\sroman{2}$ is from the expectation property of i.i.d. random variables and $\sroman{3}$ is from Example 2.3 in \cite{wainwright2019high}. So we proved that the Rademacher Process is a sub-Gaussian Process with respect to the Euclidean metric $\rho_2$.
So by Lemma \ref{lma-dudley-int} and \eqref{ieq-dudley-remark}, we know that: $\forall{\delta \in (0,D]}$,
\begin{equation*}
    \begin{aligned}
        \mathcal{R}_\mathcal{S}(\mathcal{H}_0^w) \!=\! \mathbb{E}[\underset{\theta \in \mathcal{H}_0^w(\mathcal{S})}{\sup} X_\theta] \!\leq\! \mathbb{E}\left[ \underset{\theta,\widetilde{\theta} \in \mathcal{H}_0^w(\mathcal{S})}{\sup} \!(X_\theta \!-\! X_{\widetilde{\theta}})\! \right] \!\leq\! 2 \mathbb{E} \left[ \underset{\underset{\rho_X(\gamma,\gamma^\prime)\leq \delta}{\gamma,\gamma^\prime \in \mathcal{H}_0^w(\mathcal{S})}}{\sup}\!(X_\gamma \!-\! X_{\gamma^\prime})\! \right] \!+\! 32 \mathcal{J}(\delta/4; D),
    \end{aligned}
\end{equation*}
where
\begin{equation} \label{ieq-A11-4}
    \begin{aligned}
        D &= \underset{\theta, \theta^\prime \in \mathcal{H}_0^w(\mathcal{S})}{\sup} \normvec{\theta - \theta^\prime}_2  \leq 2 \underset{\theta \in \mathcal{H}_0^w(\mathcal{S})}{\sup} \normvec{\theta}_2 = 2 \underset{h \in \mathcal{H}_0^w}{\sup} \normvec{h}_2 = 2 \underset{h \in \mathcal{H}_0^w}{\sup} \sqrt{\sum_{i=1}^M \left[ h(z_i) \right]^2} \\[1mm]
        &\overset{\sroman{1}}{\leq} 2 \sqrt{M} \underset{1 \leq i \leq M}{\underset{\normmm{A}_p \leq w}{\sup}} | x_i^T A (x_i^+ - x_i^-) |  \overset{\sroman{2}}{\leq} 2 \sqrt{M} \underset{1 \leq i \leq M}{\underset{\normmm{A}_p \leq w}{\sup}} \ \normvec{x_i}_{p^*} \normvec{A (x_i^+ - x_i^-)}_p \\[1mm]
        &\overset{\sroman{3}}{\leq} 2 \sqrt{M} \underset{1 \leq i \leq M}{\underset{\normmm{A}_p \leq w}{\sup}} \ \normvec{x_i}_{p^*} \normmm{A}_p \normvec{x_i^+ - x_i^-}_p \overset{\sroman{4}}{\leq} 4 \sqrt{M} P^* P w,
    \end{aligned}
\end{equation}
and $\mathcal{J}(a;b) = \int_a^b \sqrt{ ln \mathcal{N}\left(u; \mathcal{H}_0^w(\mathcal{S}), \normvec{\cdot}_2 \right)} du = \int_a^b \sqrt{ln \mathcal{N}\left(u; \mathcal{H}_0^w, \normvec{\cdot}_2 \right)} du$. Where $\sroman{1}$ is from the definition of $f$; $\sroman{2}$ is from the Holder's Inequality; $\sroman{3}$ is a result of properties of matrix norm and $\sroman{4}$ is from the definition of $P^*$ and $P$.
Similar to the discussion of upper bound for $D$, for all $h_1, h_2 \in \mathcal{H}_0^w$, we have:
\begin{equation} \label{ieq-A11-5}
    \begin{aligned}
        \normvec{h_1 - h_2}_2 &= \sqrt{\sum_{i=1}^M \left[ h_1(z_i) - h_2(z_i) \right]^2} \leq \sqrt{M} \underset{1 \leq i \leq M}{\sup} | h_1(z_i) - h_2(z_i) | \\[1mm]
        &= \sqrt{M} \underset{1 \leq i \leq M}{\sup} | x_i^T (A_1 - A_2) (x_i^+ - x_i^-) | \\[1mm]
        &\overset{\sroman{1}}{\leq} \sqrt{M} \underset{1 \leq i \leq M}{\sup} \normvec{x_i}_{p^*} \ \normmm{A_1 - A_2}_p \normvec{x_i^+ - x_i^-}_p \overset{\sroman{2}}{\leq} 2 \sqrt{M} P^* P \normmm{A_1 - A_2}_p,
    \end{aligned}
\end{equation}
where $A_1,A_2$ are the matrices corresponding to $h_1,h_2$, respectively and $\sroman{1}$ is from the same argument as $\sroman{2},\sroman{3}$ in \eqref{ieq-A11-4} and $\sroman{2}$ is from the definition of $P^*$ and $P$.
Suppose $N_A  = \left\{ A^1,\cdots, A^N \right\}$ is a $\frac{\delta}{2 P P^* \sqrt{M}}$-covering of $S_A = \left\{ A \Big| \normmm{A}_p \leq w \right\}$ with respect to $\normmm{\cdot}_p$ , i.e.:
\begin{equation*}
    \forall{A \in S_A}\ ,\exists{A^j \in N_A}, s.t. \ \normmm{A-A^j}_p \leq \frac{\delta}{2 P P^* \sqrt{M}}.
\end{equation*}
Combine this with \eqref{ieq-A11-5}, let $A$ be the matrix corresponding to $h$ and $A^j$ be the matrix corresponding to $h^j$ and let $N_h = \left\{ h^1, h^2, \cdots, h^N \right\}$, we have:
\begin{equation*}
    \forall{h \in \mathcal{H}_0^w}\ ,\exists{h^j \in N_h}, s.t. \ \normvec{h-h^j}_2 \leq 2 \sqrt{M} P^* P \  \frac{\delta}{2 P P^* \sqrt{M}} = \delta.
\end{equation*}
So $N_h$ is a $\delta$-covering of $\mathcal{H}_0^w$ with respect to $\normvec{\cdot}_2$ , so we have:
\begin{equation*}
    \mathcal{N}(\delta; \mathcal{H}_0^w, \normvec{\cdot}_2) \leq \mathcal{N}(\frac{\delta}{2 P P^* \sqrt{M}}; S_A, \normmm{\cdot}_p).
\end{equation*}
By Lemma \ref{lma-unit-ball-covering-number-bound}, we know that:
\begin{equation} \label{ieq-A11-6}
    \mathcal{N}(\delta; \mathcal{H}_0^w, \normvec{\cdot}_2) \leq \mathcal{N}(\frac{\delta}{2 P P^* \sqrt{M}}; S_A, \normmm{\cdot}_p) \leq \left( 1 + \frac{4 P P^* w \sqrt{M}}{\delta} \right)^{m^2}.
\end{equation}
So we have:
\begin{equation*}
    \begin{aligned}
        \mathcal{J}(0;D) &= \int_0^D \sqrt{ ln \  \mathcal{N}\left(u; \mathcal{H}_0^w, \normvec{\cdot}_2 \right)} du \overset{\sroman{1}}{\leq} \int_0^D \sqrt{m^2 ln \left( 1 + \frac{4 P P^* w \sqrt{M}}{u} \right)} du \\[1mm]
        &\overset{\sroman{2}}{\leq} m \int_0^D \sqrt{\frac{4 P P^* w \sqrt{M}}{u}} du = 2m\sqrt{PP^*w} \sqrt[4]{M} \int_0^D u^{-\frac{1}{2}} du \\[1mm]
        &= 4m\sqrt{PP^*wD} \sqrt[4]{M} \overset{\sroman{3}}{\leq} 8m PP^* w \sqrt{M},
    \end{aligned}
\end{equation*}
where $\sroman{1}$ is from \eqref{ieq-A11-6}; $\sroman{2}$ is because that for any $x \ge 0$, $ln(1+x) \leq x$ and $\sroman{3}$ is from \eqref{ieq-A11-4}.
So take $\delta \xrightarrow[]{} 0^+$, we have:
\begin{equation*}
    \begin{aligned}
        \mathcal{R}_\mathcal{S}(\mathcal{H}_0^w) \leq 32 \mathcal{J}(0; D) \leq 256m PP^* w \sqrt{M}.
    \end{aligned}
\end{equation*}
\textbf{For} \bm{$\mathcal{R}_\mathcal{S}(\mathcal{H}_1^w)$}:
The same as \bm{$\mathcal{R}_\mathcal{S}(\mathcal{H}_0^w)$}, we consider $\ell_2$ norm for a function $h \in \mathcal{H}_1^w$ and define $\mathcal{H}_1^w(\mathcal{S}) = \left\{\left(h(z_1),h(z_2),\dots,h(z_M)\right) \Big| h \in \mathcal{H}_1^w \right\}$, by similar argument in \eqref{eq-A11-3}, we have:
\begin{equation*}
    \mathcal{N}\left(\delta; \mathcal{H}_1^w(\mathcal{S}), \normvec{\cdot}_2 \right) = \mathcal{N}\left(\delta; \mathcal{H}_1^w, \normvec{\cdot}_2 \right).
\end{equation*}
By the definition of Rademacher Complexity, we know that $\mathcal{R}_\mathcal{S}(\mathcal{H}_1^w)$ is just the expectation of the Rademacher Process with respect to $\mathcal{H}_1^w(\mathcal{S})$, which is $\mathbb{E}[\underset{\theta \in \mathcal{H}_1^w(\mathcal{S})}{\sup} X_\theta]$.
By Lemma \ref{lma-dudley-int} and \eqref{ieq-dudley-remark}, we know that: $\forall{\delta \in (0,D]}$,
\vskip -0.15in
\begin{equation*}
    \begin{aligned}
        \mathcal{R}_\mathcal{S}(\mathcal{H}_1^w) \!&=\! \mathbb{E}[\underset{\theta \in \mathcal{H}_1^w(\mathcal{S})}{\sup} \!X_\theta] \!\leq\! \mathbb{E}\!\left[ \underset{\theta,\widetilde{\theta} \in \mathcal{H}_1^w(\mathcal{S})}{\sup}\!\!(X_\theta - X_{\widetilde{\theta}})\! \right] \!\leq\! 2 \mathbb{E}\! \left[ \underset{\underset{\rho_X(\gamma,\gamma^\prime)\leq \delta}{\gamma,\gamma^\prime \in \mathcal{H}_1^w(\mathcal{S})}}{\sup}\!\!(X_\gamma - X_{\gamma^\prime})\! \right] \!+\! 32 \mathcal{J}(\delta/4; D),
    \end{aligned}
\end{equation*}
where
\vskip -0.1in
\begin{equation} \label{ieq-A11-8}
    \begin{aligned}
        D &= \underset{\theta, \theta^\prime \in \mathcal{H}_1^w(\mathcal{S})}{\sup} \normvec{\theta - \theta^\prime}_2  \leq 2 \underset{\theta \in \mathcal{H}_1^w(\mathcal{S})}{\sup} \normvec{\theta}_2 = 2 \underset{h \in \mathcal{H}_1^w}{\sup} \normvec{h}_2 = 2 \underset{h \in \mathcal{H}_1^w}{\sup} \sqrt{\sum_{i=1}^M \left[ h(z_i) \right]^2} \\[1mm]
        &\overset{\sroman{1}}{\leq} 2 \sqrt{M} \underset{1 \leq i \leq M}{\underset{\normmm{A}_p \leq w}{\sup}} \normvec{A(x_i^+ - x_i^-)}_{r^*} \overset{\sroman{2}}{\leq} 2 \sqrt{M} \underset{1 \leq i \leq M}{\underset{\normmm{A}_p \leq w}{\sup}} \  \normmm{A}_{r^*} \normvec{x_i^+ - x_i^-}_{r^*} \\[1mm]
        &\overset{\sroman{3}}{\leq} 4 \sqrt{M} R^* \underset{\normmm{A}_p \leq w}{\sup} \normmm{A}_{r^*} \overset{\sroman{4}}{\leq} 4 \sqrt{M} R^* \underset{\normmm{A}_p \leq w}{\sup} s(r*,p,m) \normmm{A}_p \leq 4w R^* s(r^*,p,m) \sqrt{M},
    \end{aligned}
\end{equation}
and $\mathcal{J}(a;b) = \int_a^b \sqrt{ ln \mathcal{N}\left(u; \mathcal{H}_1^w(\mathcal{S}), \normvec{\cdot}_2 \right)} du = \int_a^b \sqrt{ln \mathcal{N}\left(u; \mathcal{H}_1^w, \normvec{\cdot}_2 \right)} du$. Where $\sroman{1}$ is from the definition of $f$; $\sroman{2}$ is a result of the properties of matrix norm; $\sroman{3}$ is from the definition of $R^*$ and $\sroman{4}$ comes from Lemma \ref{lma-norm-difference}.
Similar to the discussion of upper bound for $D$, for all $h_1, h_2 \in \mathcal{H}_1^w$, we have:
\vskip -0.15in
\begin{equation} \label{ieq-A11-9}
    \begin{aligned}
        \normvec{h_1 - h_2}_2 &= \sqrt{\sum_{i=1}^M \left[ h_1(z_i) - h_2(z_i) \right]^2} \leq \sqrt{M} \underset{1 \leq i \leq M}{\sup} | h_1(z_i) - h_2(z_i) | \\[1mm]
        &= \sqrt{M} \underset{1 \leq i \leq M}{\sup} \normvec{(A_1 - A_2) (x_i^+ - x_i^-)}_{r^*} \overset{\sroman{1}}{\leq} \sqrt{M} \underset{1 \leq i \leq M}{\sup}\ \normmm{A_1 - A_2}_{r^*} \normvec{x_i^+ - x_i^-}_{r^*} \\[1mm]
        &\overset{\sroman{2}}{\leq} 2 \sqrt{M} R^* \normmm{A_1 - A_2}_{r^*} \overset{\sroman{3}}{\leq} 2 R^* s(r*,p,m) \sqrt{M} \ \normmm{A_1 - A_2}_p,
    \end{aligned}
\end{equation}
where $A_1,A_2$ are the matrices corresponding to $h_1,h_2$, respectively and $\sroman{1}$ from the properties of matrix norm and $\sroman{2}$ is from the definition of $R^*$ and $\sroman{3}$ comes from Lemma \ref{lma-norm-difference}.
Suppose $N_A  = \left\{ A^1,\cdots, A^N \right\}$ is a $\frac{\delta}{2 R^* s(r*,p,m) \sqrt{M}}$-covering of $S_A = \left\{ A \Big| \normmm{A}_p \leq w \right\}$ with respect to $\normmm{\cdot}_p$ , i.e.:
\begin{equation*}
    \forall{A \in S_A}\ ,\exists{A^j \in N_A}, s.t. \ \normmm{A-A^j}_p \leq \frac{\delta}{2 R^* s(r*,p,m) \sqrt{M}}.
\end{equation*}
Combine this with \eqref{ieq-A11-9}, let $A$ be the matrix corresponding to $h$ and $A^j$ be the matrix corresponding to $h^j$ and let $N_h = \left\{ h^1, h^2, \cdots, h^N \right\}$, we have:
\begin{equation*}
    \forall{h \in \mathcal{H}_1^w}\ ,\exists{h^j \in N_h}, s.t. \ \normvec{h-h^j}_2 \leq 2 R^* s(r*,p,m) \sqrt{M} \  \frac{\delta}{2 R^* s(r*,p,m) \sqrt{M}} = \delta.
\end{equation*}
So $N_h$ is a $\delta$-covering of $\mathcal{H}_1^w$ with respect to $\normvec{\cdot}_2$ , so we have:
\begin{equation*}
    \mathcal{N}(\delta; \mathcal{H}_1^w, \normvec{\cdot}_2) \leq \mathcal{N}(\frac{\delta}{2 R^* s(r*,p,m) \sqrt{M}}; S_A, \normmm{\cdot}_p).
\end{equation*}
By Lemma \ref{lma-unit-ball-covering-number-bound}, we know that:
\begin{equation} \label{ieq-A11-11}
    \mathcal{N}(\delta; \mathcal{H}_1^w, \normvec{\cdot}_2) \leq \mathcal{N}(\frac{\delta}{2 R^* s(r*,p,m) \sqrt{M}}; S_A, \normmm{\cdot}_p) \leq \left( 1 + \frac{4 R^* s(r*,p,m) w \sqrt{M}}{\delta} \right)^{m^2}.
\end{equation}
So we have:
\begin{equation*}
    \begin{aligned}
        \mathcal{J}(0;D) &= \int_0^D \sqrt{ ln \  \mathcal{N}\left(u; \mathcal{H}_1^w, \normvec{\cdot}_2 \right)} du \overset{\sroman{1}}{\leq} \int_0^D \sqrt{m^2 ln \left( 1 + \frac{4 R^* s(r*,p,m) w \sqrt{M}}{u} \right)} du \\[1mm]
        &\overset{\sroman{2}}{\leq} m \int_0^D \sqrt{\frac{4 R^* s(r*,p,m) w \sqrt{M}}{u}} du = 2m\sqrt{R^* s(r*,p,m)w} \sqrt[4]{M} \int_0^D u^{-\frac{1}{2}} du \\[1mm]
        &= 4m\sqrt{R^* s(r*,p,m) wD} \sqrt[4]{M} \overset{\sroman{3}}{\leq} 8m R^* s(r*,p,m) w \sqrt{M},
    \end{aligned}
\end{equation*}
where $\sroman{1}$ is from \eqref{ieq-A11-11}; $\sroman{2}$ is because that for any $x \ge 0$, $ln(1+x) \leq x$ and $\sroman{3}$ is from \eqref{ieq-A11-8}.
So take $\delta \xrightarrow[]{} 0^+$, we have:
\begin{equation*}
    \begin{aligned}
        \mathcal{R}_\mathcal{S}(\mathcal{H}_1^w) \leq 32 \mathcal{J}(0; D) \leq 256m R^* s(r*,p,m) w \sqrt{M}.
    \end{aligned}
\end{equation*}
Combine upper bounds of $\mathcal{R}_\mathcal{S}(\mathcal{H}_0^w)$ and $\mathcal{R}_\mathcal{S}(\mathcal{H}_1^w)$ with \eqref{ieq-dvd-Rs(H)}, we have:
\begin{equation*}
    \begin{aligned}
        \mathcal{R}_\mathcal{S}(\mathcal{H}) &\leq \mathcal{R}_\mathcal{S}(\mathcal{H}_0^{s(p^*,p,m) w^2}) + \epsilon \  \mathcal{R}_\mathcal{S}(\mathcal{H}_1^{s(p^*,p,m) w^2}) \\[1mm]
        &\leq 256m PP^* s(p^*,p,m) w^2 \sqrt{M} + \epsilon \  256m R^* s(r*,p,m) s(p^*,p,m) w^2 \sqrt{M} \\[1mm]
        &= 256 m \ s(p^*,p,m) w^2 \sqrt{M} \left( PP^* + \epsilon R^* s(r*,p,m) \right) \\[1mm]
        &= O\left( \left( PP^* + \epsilon R^* s(r*,p,m) \right) \left( m \ s(p^*,p,m) w^2 \sqrt{M} \right) \right).
    \end{aligned}
\end{equation*}
\end{proof}

\subsection{Proof of Theorem \ref{thm-nn-Rs(H)-bound-Fnorm}} \label{apd-prf-thm-nn-Rs(H)-bound-Fnorm}
\begin{customthm}{{\ref{thm-nn-Rs(H)-bound-Fnorm}}}
Let $\mathcal{U}(x)=\left\{ x^\prime | \normvec{x^\prime - x}_p \leq \epsilon \right\}$ (i.e. consider the $\ell_p$ attack), $\sigma(0)=0$ with Lipschitz constant $L$ and let $\mathcal{F}=\left\{W_d \sigma(W_{d-1} \sigma(\cdots \sigma(W_1 x)))\big| \normmm{W_l}_F\!\!\leq\!\! M_l^F, l=\!1,\dots, d\! \right\}$. We then have:
\begin{equation*}
    \mathcal{R}_\mathcal{S}(\mathcal{H}) = O\left( \sqrt{\sum_{l=1}^d h_l h_{l-1}} K \sqrt{d} \sqrt{M} \right),
\end{equation*}
where $K = 2 B_{X,\epsilon}^F \cdot \left( B_{X^+}^F + B_{X^-}^F \right)$, where
\begin{equation*}
    B_{X,\epsilon}^F \!=\! L^{d-1}\! \prod_{l=1}^d M_l^F \max\left\{ 1, m^{\frac{1}{2}-\frac{1}{p}} \right\}\!\left( \normvec{X}_{p, \infty} \!+\! \epsilon \right), \  B_X^F \!=\! L^{d-1}\! \prod_{l=1}^d M_l^F \max\left\{ 1, m^{\frac{1}{2}-\frac{1}{p}} \right\}\! \normvec{X}_{p, \infty}.
\end{equation*}
\end{customthm}
Before giving the proof, we firstly introduce some useful lemmas.
\begin{lemma} \label{lma-perturbation-norm-bound}
If $x_i^* \in \mathcal{U}(x_i) = \left\{ x_i^\prime \big| \normvec{x_i - x_i^\prime}_p \leq \epsilon \right\}$, then for $\frac{1}{r^*} + \frac{1}{r} = 1$, we have:
\begin{equation*}
    \normvec{x_i^*}_{r^*} \leq \max \left\{ 1, m^{1-\frac{1}{r}-\frac{1}{p}} \right\} \left( \normvec{X}_{p,\infty} + \epsilon \right).
\end{equation*}
\end{lemma}
\begin{proof} [Proof of Lemma \ref{lma-perturbation-norm-bound}]
We divide it into two cases.
\begin{enumerate}
    \item If $p \ge r^*$, then by Holder's Inequality with $\frac{1}{r^*} = \frac{1}{p} + \frac{1}{s}$, we have:
    \begin{equation*}
        \normvec{x_i^*}_{r^*} \leq \sup \  \normvec{\bm{1}}_s \cdot \normvec{x_i^*}_p = \normvec{\bm{1}}_s \cdot \normvec{x_i^*}_p = m^{\frac{1}{s}}\normvec{x_i^*}_p = m^{1-\frac{1}{r}-\frac{1}{p}} \normvec{x_i^*}_p,
    \end{equation*}
    where the equality holds when all the entries are equal.
    \item If $p < r^*$, by Lemma \ref{lma-vec-norm}, we have $\normvec{x_i^*}_{r^*} \leq \normvec{x_i^*}_p,$
    where the equality holds when one of the entries of $x_i^*$ equals to one and the others equal to zero. 
\end{enumerate}
Then we have:
\begin{equation*}
    \begin{aligned}
        \normvec{x_i^*}_{r^*} &\leq \max \left\{ 1, m^{1-\frac{1}{r}-\frac{1}{p}} \right\} \ \normvec{x_i^*}_p \leq \max \left\{ 1, m^{1-\frac{1}{r}-\frac{1}{p}} \right\} \left( \normvec{x_i}_p + \normvec{x_i - x_i^*}_p \right) \\[1mm]
        &\leq \max \left\{ 1, m^{1-\frac{1}{r}-\frac{1}{p}} \right\} \left( \normvec{X}_{p,\infty} + \epsilon \right).
    \end{aligned}
\end{equation*}
\end{proof}

\begin{lemma} \label{lma-v2norm-mFnorm-property}
Let $A \in \mathbb{R}^{m \times n}, b \in \mathbb{R}^n$, then we have: $\normvec{A\cdot b}_2 \leq \normmm{A}_F \cdot \normvec{b}_2.$
\end{lemma}
\begin{proof} [Proof of Lemma \ref{lma-v2norm-mFnorm-property}]
Let $A_i$ be the rows of $A, i=1,2,\dots,m$, we have:
\begin{equation*}
    \normvec{A\cdot b}_2 = \sqrt{\sum_{i=1}^m\left( A_i b \right)^2} \overset{\sroman{1}}{\leq} \sqrt{\sum_{i=1}^m \normvec{A_i}_2^2 \cdot \normvec{b}_2^2} = \sqrt{\sum_{i=1}^m \normvec{A_i}_2^2} \cdot \sqrt{\normvec{b}_2^2} = \normmm{A}_F \cdot \normvec{b}_2,
\end{equation*}
where $\sroman{1}$ is from Holder's Inequality.
\end{proof}

\begin{lemma} \label{lma-sigma-2norm-Lip}
Suppose $\sigma$ is a $L$-Lipschitz function, then the elementwise vector map corresponding to $\sigma$ is also $L$-Lipschitz with respect to $\normvec{\cdot}_2$.
\end{lemma}
\begin{proof} [Proof of Lemma \ref{lma-sigma-2norm-Lip}]
\begin{equation*}
    \begin{aligned}
        \normvec{\sigma(x)-\sigma(y)}_2 &= \sqrt{\sum_{i=1}^n \left( \sigma(x)_i - \sigma(y)_i \right)^2} = \sqrt{\sum_{i=1}^n \left( \sigma(x_i) - \sigma(y_i) \right)^2} \\[1mm]
        &\overset{\sroman{1}}{\leq} \sqrt{\sum_{i=1}^n L^2 (x_i - y_i)^2} = L \sqrt{\sum_{i=1}^n (x_i - y_i)^2} = L \cdot \normvec{x - y}_2,
    \end{aligned}
\end{equation*}
where $\sroman{1}$ is because $\sigma$ is $L$-Lipschitz.
\end{proof}
Now we can turn to the proof of Theorem \ref{thm-nn-Rs(H)-bound-Fnorm}.
\begin{proof} [Proof of Theorem \ref{thm-nn-Rs(H)-bound-Fnorm}]
In this case, let $\mathcal{U}(x) = \left\{ x^\prime \big| \normvec{x^\prime - x}_p \leq \epsilon \right\}$, we have:
\begin{equation*}
    \mathcal{F}=\left\{ x \xrightarrow[]{} W_d \sigma(W_{d-1} \sigma(\cdots \sigma(W_1 x)))\  \big| \  \normmm{W_l}_F \leq M_l^F, l=1,\dots, d  \right\},
\end{equation*}
\begin{equation*}
    \mathcal{H} = \left\{ h(x,x^+,x^-) = \underset{x^\prime \in \mathcal{U}(x)}{\min} \left( f(x^\prime)^T\left( f(x^+) - f(x^-) \right) \right) |f \in \mathcal{F} \right\}.
\end{equation*}
Let $S_l^F = \left\{ W_l \in \mathbb{R}^{h_l \times h_{l-1}} \big| \normmm{W_l}_F \leq M_l^F \right\}, l=1,2,\dots,d$. Let $\mathcal{C}_l^F$ be the $\delta_l$-covering of $S_l^F$ and define:
\begin{equation*}
    \mathcal{F}^c = \left\{ f^c: x \xrightarrow[]{} W_d^c  \sigma \left( W_{d-1}^c \sigma \left( \cdots \sigma (W_1^c x) \right) \right) \Big| W_l^c \in \mathcal{C}_l^F, l=1,2,\dots,d \right\} \subseteq \mathcal{F},
\end{equation*}
\begin{equation*}
    \mathcal{H}^c = \left\{ h^c(x,x^+,x^-) = \underset{x^\prime \in \mathcal{U}(x)}{\min} \left( f(x^\prime)^T\left( f(x^+) - f(x^-) \right) \right) |f \in \mathcal{F}^c \right\} \subseteq \mathcal{H}.
\end{equation*}
Similar to the proof of Theorem \ref{thm-linear-Rs(H)-bound}, we know that the Rademacher Process is a sub-Gaussian Process with respect to the Euclidean metric, which induces the $\ell_2$ norm.

Given the training set $\mathcal{S}=\{(x_i,x_i^+,x_i^-)\}_{i=1}^M \triangleq \{z_i\}_{i=1}^M$, define the $\ell_2$-norm for a function in $\mathcal{H}$ as:
\begin{equation*}
    \forall{h \in \mathcal{H}}, \  \normvec{h}_2 \coloneqq \sqrt{\sum_{i=1}^M \left[ h(z_i) \right]^2}.
\end{equation*}
Define $\mathcal{H}(\mathcal{S}) = \left\{\left(h(z_1),h(z_2),\dots,h(z_M)\right) \Big| h \in \mathcal{H} \right\}$, we have that for any $h \in \mathcal{H}$ and $v_h = \left(h(z_1),h(z_2),\dots,h(z_M)\right)$ be the corresponding vector in $\mathcal{H}(\mathcal{S})$, we have $\normvec{h}_2 = \normvec{v_h}_2.$
So we know that any $\delta$-covering of $\mathcal{H}$ ($\left\{ h^1, h^2, \cdots, h^N \right\}$) with respect to $\ell_2$ norm in the functional space, corresponds to a $\delta$-covering of $\mathcal{H}(\mathcal{S})$ with respect to the $\ell_2$ norm in the Euclidean Space, i.e.
\begin{equation*}
    \left\{ \left( \begin{array}{c}  h^1(z_1)  \\  h^1(z_2) \\  \vdots \\  h^1(z_M) \end{array} \right) ,
    \left( \begin{array}{c}  h^2(z_1)  \\  h^2(z_2) \\  \vdots \\  h^2(z_M) \end{array} \right), \cdots, 
    \left( \begin{array}{c}  h^N(z_1)  \\  h^N(z_2) \\  \vdots \\  h^N(z_M) \end{array} \right) \right\}.
\end{equation*}
So we have: 
\begin{equation*}
    \mathcal{N}\left(\delta; \mathcal{H}(\mathcal{S}), \normvec{\cdot}_2 \right) = \mathcal{N}\left(\delta; \mathcal{H}, \normvec{\cdot}_2 \right).
\end{equation*}
By the definition of Rademacher Complexity, we know that $\mathcal{R}_\mathcal{S}(\mathcal{H})$ is just the expectation of the Rademacher Process with respect to $\mathcal{H}(\mathcal{S})$, which is $\mathbb{E}[\underset{\theta \in \mathcal{H}(\mathcal{S})}{\sup} X_\theta]$.

So by Lemma \ref{lma-dudley-int} and \eqref{ieq-dudley-remark}, we know that: $\forall{\delta \in (0,D]}$:
\begin{equation*}
    \begin{aligned}
        \mathcal{R}_\mathcal{S}(\mathcal{H}) \!=\! \mathbb{E}[\underset{\theta \in \mathcal{H}(\mathcal{S})}{\sup} X_\theta] \leq \mathbb{E}\left[ \underset{\theta,\widetilde{\theta} \in \mathcal{H}(\mathcal{S})}{\sup} \!\!(X_\theta \!-\! X_{\widetilde{\theta}}) \right] \leq 2 \mathbb{E} \left[ \underset{\underset{\normvec{\gamma^\prime-\gamma}_2\leq \delta}{\gamma,\gamma^\prime \in \mathcal{H}(\mathcal{S})}}{\sup}\!\!(X_\gamma \!-\! X_{\gamma^\prime}) \right] + 32 \mathcal{J}(\delta/4; D),
    \end{aligned}
\end{equation*}
where
\begin{equation} \label{ieq-A12-2}
    D \!=\! \underset{\theta, \theta^\prime \in \mathcal{H}(\mathcal{S})}{\sup} \!\normvec{\theta - \theta^\prime}_2  \!\leq\! 2 \underset{\theta \in \mathcal{H}(\mathcal{S})}{\sup} \normvec{\theta}_2 \!=\! 2 \underset{h \in \mathcal{H}}{\sup} \normvec{h}_2 \!=\! 2 \underset{h \in \mathcal{H}}{\sup} \sqrt{\sum_{i=1}^M \left[ h(z_i) \right]^2} \leq 2 \sqrt{M} \underset{1 \leq i \leq M}{\underset{h \in \mathcal{H}}{\sup}} |h(z_i)|,
\end{equation}
and $\mathcal{J}(a;b) = \int_a^b \sqrt{ ln \mathcal{N}\left(u; \mathcal{H}(\mathcal{S}), \normvec{\cdot}_2 \right)} du = \int_a^b \sqrt{ln \mathcal{N}\left(u; \mathcal{H}, \normvec{\cdot}_2 \right)} du$.

For any $f \in \mathcal{F}, x \in \mathcal{X}$, let $x^l$ be the output of $x$ passing through the first $l-1$ layers, we have:
\begin{equation*}
    \begin{aligned}
        \normvec{f(x)}_2 \!&=\! \normvec{W_d \sigma (W_{d-1} x^{d-1})}_2 \!\overset{\sroman{1}}{\leq}\! \normmm{W_d}_F \!\cdot\! \normvec{\sigma (W_{d-1} x^{d-1})}_2 \!\overset{\sroman{2}}{=}\! \normmm{W_d}_F \!\cdot\! \normvec{\sigma (W_{d-1} x^{d-1}) \!-\! \sigma (\bm{0})}_2 \\[1mm]
        &\overset{\sroman{3}}{\leq} L M_d^F \normvec{W_{d-1} x^{d-1}}_2 \leq \cdots \leq L^{d-1} \prod_{l=2}^d M_l^F \ \normvec{W_1 x}_2 \leq L^{d-1} \prod_{l=1}^d M_l^F \ \normvec{x}_2 \\[1mm]
        &\overset{\sroman{4}}{\leq} L^{d-1} \prod_{l=1}^d M_l^F \  \max \left\{ 1, m^{\frac{1}{2}-\frac{1}{p}} \right\} \normvec{x}_p \leq L^{d-1} \prod_{l=1}^d M_l^F \  \max \left\{ 1, m^{\frac{1}{2}-\frac{1}{p}} \right\} \normvec{X}_{p,\infty},
    \end{aligned}
\end{equation*}
where $\sroman{1}$ is from Lemma \ref{lma-v2norm-mFnorm-property}; $\sroman{2}$ is from the fact that $\sigma(0)=0$; $\sroman{3}$ comes from the assumption that $\sigma$ is $L$-Lipschitz and $\normmm{W_d}_F \leq M_d^F$ and $\sroman{4}$ is attained by setting $r = r^* =2$ in the proof of Lemma \ref{lma-perturbation-norm-bound}.

To simplify the notations, we define:
\begin{equation*}
    B_{X,\epsilon}^F \!=\! L^{d-1}\! \prod_{l=1}^d \!M_l^F \max\!\left\{ 1, m^{\frac{1}{2}-\frac{1}{p}} \right\}\!\left( \normvec{X}_{p, \infty} \!+\! \epsilon \right), B_X^F \!=\! L^{d-1} \!\prod_{l=1}^d M_l^F \max\left\{ 1, m^{\frac{1}{2}-\frac{1}{p}} \right\}\! \normvec{X}_{p, \infty}.
\end{equation*}
So we have:
\begin{equation} \label{ieq-f(x)-bound}
    \forall{x \in \mathcal{X}, f \in \mathcal{F}}, \normvec{f(x)}_2 \leq B_X^F.
\end{equation}
Similarly, we have:
\begin{equation} \label{ieq-f(x')-bound}
    \forall{x \in \mathcal{X}, f \in \mathcal{F}}, \  \forall{\ \normvec{x^\prime - x}_p \leq \epsilon}, \normvec{f(x^\prime)}_2 \leq B_{X,\epsilon}^F.
\end{equation}
For any $h \in \mathcal{H}, z \in \mathcal{X}^3$, let $x^* = \underset{\normvec{x^\prime - x}_p \leq \epsilon}{\arg\min} f(x^\prime)^T \left( f(x^+) - f(x^-) \right)$ and let $x^l$ be the output of $x^*$ passing through the first $l-1$ layers, we have:
\begin{equation*}
    \begin{aligned}
        |h(z)| &= | \underset{\normvec{x^\prime - x}_p \leq \epsilon}{\inf} f(x^\prime)^T \left( f(x^+) - f(x^-) \right) | = |f(x^*)^T \left( f(x^+) - f(x^-) \right)| \\[1mm]
        &\leq \normvec{f(x^*)}_2 \cdot \normvec{f(x^+) - f(x^-)}_2 \overset{\sroman{1}}{\leq} B_{X,\epsilon}^F \cdot (B_{X^+}^F + B_{X^-}^F),
    \end{aligned}
\end{equation*}
where $\sroman{1}$ is from \eqref{ieq-f(x)-bound} and \eqref{ieq-f(x')-bound}. So we have:
\begin{equation} \label{ieq-D-bound}
    D \overset{\sroman{1}}{\leq} 2 \sqrt{M} \underset{1 \leq i \leq M}{\underset{h \in \mathcal{H}}{\sup}} |h(z_i)| \leq 2 \sqrt{M} B_{X,\epsilon}^F \cdot (B_{X^+}^F + B_{X^-}^F) \triangleq \sqrt{M} K,
\end{equation}
where $\sroman{1}$ is from \eqref{ieq-A12-2}. Now, we need to find the smallest distance between $\mathcal{H}$ and $\mathcal{H}^c$, i.e.
\begin{equation*}
    \underset{h \in \mathcal{H}}{\sup} \underset{h^c \in \mathcal{H}^c}{\inf} \normvec{h - h^c}_2.
\end{equation*}
By the discussion in \eqref{ieq-A12-2}, we have $\normvec{h-h^c}_2 \leq \sqrt{M} \underset{1 \leq i \leq M}{\max} | h(z_i) - h^c(z_i) |$. For any $z_i = (x_i, x_i^+, x_i^-), i=1,2,\dots,M$, given $h$ and $h^c$ such that $\normmm{W_l - W_l^c}_F \leq \delta_l, l=1,2,\dots,d$, we have:
\begin{equation*}
    | h(z_i) - h^c(z_i) | = \big| \underset{\normvec{x_i^\prime - x_i}_p \leq \epsilon}{\inf} f(x_i^\prime)^T \left( f(x_i^+) - f(x_i^-) \right) - \underset{\normvec{x_i^\prime - x_i}_p \leq \epsilon}{\inf} f^c(x_i^\prime)^T \left( f^c(x_i^+) - f^c(x_i^-) \right) \big|.
\end{equation*}
Let $x_i^* = \underset{\normvec{x_i^\prime - x_i}_p \leq \epsilon}{\arg\inf} f(x_i^\prime)^T \left( f(x_i^+) - f(x_i^-) \right)$ and $x_i^c = \underset{\normvec{x_i^\prime - x_i}_p \leq \epsilon}{\arg\inf} f^c(x_i^\prime)^T \left( f^c(x_i^+) - f^c(x_i^-) \right)$, and let
\begin{equation*}
    y_i=\left\{
        \begin{array}{rcl}
        x_i^c       &      & {f(x_i^*)^T \left( f(x_i^+) - f(x_i^-) \right) \ge f^c(x_i^c)^T \left( f^c(x_i^+) - f^c(x_i^-) \right)}     \\ 
        x_i^*     &      & {otherwise}
        \end{array} \right..
\end{equation*}
Then we have:
\begin{equation} \label{ieq-A12-3}
    \begin{aligned}
        | h(z_i) - h^c(z_i) | &= | f(x_i^*)^T \left( f(x_i^+) - f(x_i^-) \right) - f^c(x_i^c)^T \left( f^c(x_i^+) - f^c(x_i^-) \right) | \\[1mm]
        &\overset{\sroman{1}}{\leq} | f(y_i)^T \left( f(x_i^+) - f(x_i^-) \right) - f^c(y_i)^T \left( f^c(x_i^+) - f^c(x_i^-) \right) | \\[1mm]
        &= | f(y_i)^T \left( f(x_i^+) - f(x_i^-) \right) - f^c(y_i)^T \left( f(x_i^+) - f(x_i^-) \right) \\[1mm]
        &\ \ \ \ \ \ \ \ \ \ \ \ \ \ \ \ \ + f^c(y_i)^T \left( f(x_i^+) - f(x_i^-) \right)  - f^c(y_i)^T \left( f^c(x_i^+) - f^c(x_i^-) \right) | \\[1mm]
        &\overset{\sroman{2}}{\leq} | \left( f(y_i) - f^c(y_i) \right)^T \left( f(x_i^+) - f(x_i^-) \right) | \\[1mm]
        &\ \ \ \ \ \ \ \ \ \ \ \ \ \ \ \ \ + | f^c(y_i)^T \left( f(x_i^+) - f^c(x_i^+) \right) | + | f^c(y_i)^T \left( f(x_i^-) - f^c(x_i^-) \right) | \\[1mm]
        &\overset{\sroman{3}}{\leq} (B_{X^+}^F + B_{X^-}^F) \normvec{f(y_i) - f^c(y_i)}_2+B_{X,\epsilon}^F \normvec{f(x_i^+)-f^c(x_i^+)}_2 \\[1mm]
        &\ \ \ \ \ \ \ \ \ \ \ \ \ \ \ \ \ + B_{X,\epsilon}^F \normvec{f(x_i^-) - f^c(x_i^-)}_2,
    \end{aligned}
\end{equation}
where $\sroman{1}$ is easily verified by the definition of $y_i$; $\sroman{2}$ is from the triangle inequality and $\sroman{3}$ is from \eqref{ieq-f(x)-bound} and \eqref{ieq-f(x')-bound}.

Define $g_b^a(\cdot)$ as:
\begin{equation*}
    g_b^a(y) = W_b \sigma \left( W_{b-1} \sigma\left( \cdots W_{a+1} \sigma \left( W_a^c \cdots \sigma (W_1^c y) \right) \right) \right).
\end{equation*}
Then we have:
\begin{equation*}
    \begin{aligned}
        \normvec{f(y_i) - f^c(y_i)}_2 &= \normvec{g_d^0(y_i) - g_d^d(y_i)}_2 \\[1mm]
        &= \normvec{g_d^0(y_i) - g_d^1(y_i) + g_d^1(y_i) - g_d^2(y_i) + \cdots + g_d^{d-1}(y_i) - g_d^d(y_i)}_2 \\[1mm]
        &\overset{\sroman{1}}{\leq} \normvec{g_d^0(y_i) - g_d^1(y_i)}_2 + \cdots + \normvec{g_d^{d-1}(y_i) - g_d^d(y_i)}_2,
    \end{aligned}
\end{equation*}
where $\sroman{1}$ is from the triangle inequality.

Then we calculate $\normvec{g_d^{l-1}(y_i) - g_d^l(y_i)}_2, l=1,2,\dots,d$:
\begin{equation} \label{ieq-gdln1-gdl-bound}
    \begin{aligned}
        \normvec{g_d^{l-1}(y_i) - g_d^l(y_i)}_2 &= \normvec{W_d \sigma \left( g_{d-1}^{l-1}(y_i) \right) - W_d \sigma \left( g_{d-1}^l(y_i) \right)  }_2 \\[1mm]
        &\overset{\sroman{1}}{\leq} \normmm{W_d}_F \cdot \normvec{\sigma \left( g_{d-1}^{l-1}(y_i) \right) - \sigma \left( g_{d-1}^l(y_i) \right)}_2 \\[1mm]
        &\overset{\sroman{2}}{\leq} L M_d^F \normvec{g_{d-1}^{l-1}(y_i) - g_{d-1}^l(y_i)}_2 \leq \cdots \\[1mm]
        &\overset{\sroman{3}}{\leq} L^{d-l}\cdot \prod_{j=l+1}^d M_j^F \cdot \normvec{W_l \sigma \left( g_{l-1}^{l-1}(y_i) \right) - W_l^c \sigma \left( g_{l-1}^{l-1}(y_i) \right) }_2 \\[1mm]
        &= L^{d-l}\cdot \prod_{j=l+1}^d M_j^F \cdot \normvec{\left( W_l - W_l^c \right) \sigma \left( g_{l-1}^{l-1}(y_i) \right) }_2 \\[1mm]
        &\overset{\sroman{4}}{\leq} L^{d-l}\cdot \prod_{j=l+1}^d M_j^F \cdot \delta_l \cdot \normvec{\sigma \left( g_{l-1}^{l-1}(y_i) \right)}_2,
    \end{aligned}
\end{equation}
where $\sroman{1}$ is from Lemma \ref{lma-v2norm-mFnorm-property}; $\sroman{2}$ comes from the assumption that $\sigma$ is $L$-Lipschitz and $\normmm{W_d}_F \leq M_d^F$; $\sroman{3}$ is from the definition of $g_b^a(\cdot)$ and $\sroman{4}$ is from Lemma \ref{lma-v2norm-mFnorm-property} and the choice of $h^c$ when $h$ is fixed, which means that $\normmm{W_l - W_l^c}_F \leq \delta_l$.

Next, we upper bound $\normvec{\sigma \left( g_{l-1}^{l-1}(y_i) \right)}_2$:
\begin{equation} \label{ieq-sigmagln1ln1-bound}
    \begin{aligned}
        \normvec{\sigma \left( g_{l-1}^{l-1}(y_i) \right)}_2 &= \normvec{\sigma \left( g_{l-1}^{l-1}(y_i) \right) - \sigma(\bm{0})}_2 \overset{\sroman{1}}{\leq} L \cdot \normvec{g_{l-1}^{l-1}(y_i)}_2 = L \cdot \normvec{W_{l-1}^c \sigma\left( g_{l-2}^{l-2}(y_i) \right)}_2 \\[1mm]
        &\overset{\sroman{2}}{\leq} L \cdot \normmm{W_{l-1}^c}_F \cdot \normvec{\sigma\left( g_{l-2}^{l-2}(y_i) \right)}_2 \overset{\sroman{3}}{\leq} L \ M_{l-1}^F \normvec{\sigma\left( g_{l-2}^{l-2}(y_i) \right)}_2 \\[1mm]
        &\leq \cdots \leq L^{l-1} \cdot \prod_{j=1}^{l-1} M_j^F \cdot \normvec{y_i}_2,
    \end{aligned}
\end{equation}
where $\sroman{1}$ is because $\sigma$ is $L$-Lipschitz; $\sroman{2}$ is from Lemma \ref{lma-v2norm-mFnorm-property} and $\sroman{3}$ is because $\normmm{W_{l-1}^c}_F \leq M_{l-1}^F$.

From \eqref{ieq-gdln1-gdl-bound} and \eqref{ieq-sigmagln1ln1-bound} we have:
\begin{equation*}
    \begin{aligned}
        \normvec{g_d^{l-1}(y_i)-g_d^l(y_i)}_2 &\leq L^{d-1} \frac{\prod_{j=1}^d M_j^F}{M_l^F} \delta_l \normvec{y_i}_2 \\[1mm]
        &\overset{\sroman{1}}{\leq} L^{d-1} \frac{\prod_{j=1}^d M_j^F}{M_l^F}\  \delta_l \  \max\left\{ 1,m^{\frac{1}{2}-\frac{1}{p}} \right\} \left( \normvec{X}_{p,\infty} + \epsilon \right) = B_{X,\epsilon}^F \  \frac{\delta_l}{M_l^F},
    \end{aligned}
\end{equation*}
where $\sroman{1}$ is from Lemma \ref{lma-perturbation-norm-bound}.

Similarly:
\begin{equation*}
    \begin{aligned}
        \normvec{g_d^{l-1}(x_i^+)-g_d^l(x_i^+)}_2 &\leq L^{d-1} \frac{\prod_{j=1}^d M_j^F}{M_l^F} \delta_l \normvec{x_i^+}_2 \\[1mm]
        &\leq L^{d-1} \frac{\prod_{j=1}^d M_j^F}{M_l^F}\  \delta_l \  \max\left\{ 1,m^{\frac{1}{2}-\frac{1}{p}} \right\} \normvec{X^+}_{p,\infty} = B_{X^+}^F \  \frac{\delta_l}{M_l^F},
    \end{aligned}
\end{equation*}
\begin{equation*}
    \begin{aligned}
        \normvec{g_d^{l-1}(x_i^-)-g_d^l(x_i^-)}_2 &\leq L^{d-1} \frac{\prod_{j=1}^d M_j^F}{M_l^F} \delta_l \normvec{x_i^-}_2 \\[1mm]
        &\leq L^{d-1} \frac{\prod_{j=1}^d M_j^F}{M_l^F}\  \delta_l \  \max\left\{ 1,m^{\frac{1}{2}-\frac{1}{p}} \right\} \normvec{X^-}_{p,\infty} = B_{X^-}^F \  \frac{\delta_l}{M_l^F}.
    \end{aligned}
\end{equation*}
Combining the above with \eqref{ieq-A12-3} yields:
\begin{equation*}
    \begin{aligned}
        | h(z_i) - h^c(z_i) | &\leq \left( B_{X^+}^F + B_{X^-}^F \right) \left( \normvec{g_d^0(y_i) - g_d^1(y_i)}_2 + \cdots + \normvec{g_d^{d-1}(y_i) - g_d^d(y_i)}_2 \right) \\[1mm]
        &\ \ \ \ \ \ \ \ + B_{X,\epsilon}^F \left( \normvec{g_d^0(x_i^+) - g_d^1(x_i^+)}_2 + \cdots + \normvec{g_d^{d-1}(x_i^+) - g_d^d(x_i^+)}_2 \right) \\[1mm]
        &\ \ \ \ \ \ \ \ + B_{X,\epsilon}^F \left( \normvec{g_d^0(x_i^-) - g_d^1(x_i^-)}_2 + \cdots + \normvec{g_d^{d-1}(x_i^-) - g_d^d(x_i^-)}_2 \right) \\[1mm]
        &=(B_{X^+}^F + B_{X^-}^F) B_{X,\epsilon}^F \sum_{l=1}^d \frac{\delta_l}{M_l^F} + B_{X,\epsilon}^F B_{X^+}^F \sum_{l=1}^d \frac{\delta_l}{M_l^F} + B_{X,\epsilon}^F B_{X^-}^F \sum_{l=1}^d \frac{\delta_l}{M_l^F} \\[1mm]
        &= 2 B_{X,\epsilon}^F (B_{X^+}^F + B_{X^-}^F) \sum_{l=1}^d \frac{\delta_l}{M_l^F} = K \sum_{l=1}^d \frac{\delta_l}{M_l^F}.
    \end{aligned}
\end{equation*}
So $\normvec{h-h^c}_2 = \sqrt{M} \underset{1 \leq i \leq M}{\max} | h(z_i) - h^c(z_i) | \leq \sqrt{M} \sum_{l=1}^d \frac{K \delta_l}{M_l^F}$. Let $\delta_l = \frac{M_l^F \delta}{d K \sqrt{M}}$, we have:
\begin{equation*}
    \normvec{h - h^c}_2 \leq \sqrt{M} \sum_{l=1}^d \frac{K}{M_l^F} \cdot \frac{M_l^F \delta}{d K \sqrt{M}} \leq \delta.
\end{equation*}
Then: $\forall{h \in \mathcal{H}}, \exists{h^c \in \mathcal{H}^c}\  s.t.\  \normvec{h - h^c}_2 \leq \delta$, which means that $\underset{h \in \mathcal{H}}{\sup} \underset{h^c \in \mathcal{H}^c}{\inf} \normvec{h - h^c}_2 \leq \delta$ when choosing $\delta_l = \frac{M_l^F \delta}{d K \sqrt{M}}$.

So $\mathcal{H}^c$ is a $\delta$-covering of $\mathcal{H}$, and $\mathcal{N}(\delta; \mathcal{H}, \normvec{\cdot}_2) \leq | \mathcal{H}^c | = \prod_{l=1}^d | \mathcal{C}_l^F |$. By Lemma \ref{lma-unit-ball-covering-number-bound} we know that $| \mathcal{C}_l^F | = \mathcal{N}(\frac{M_l^F \delta}{d K \sqrt{M}}; S_l^F, \normmm{\cdot}_F) \leq \left( 1 + \frac{2 d K \sqrt{M}}{\delta} \right)^{h_l \times h_{l-1}}$.
So we have:
\begin{equation} \label{ieq-A12-H-covering-bound}
    \mathcal{N}(\delta; \mathcal{H}, \normvec{\cdot}_2) \leq | \mathcal{H}^c | = \prod_{l=1}^d | \mathcal{C}_l^F | \leq \left( 1 + \frac{2 d K \sqrt{M}}{\delta} \right)^{\sum_{l=1}^d h_l \cdot h_{l-1}}.
\end{equation}
So we can conclude that:
\begin{equation*}
    \begin{aligned}
        \mathcal{J}(0;D) &= \int_0^D \sqrt{ ln \  \mathcal{N}\left(u; \mathcal{H}, \normvec{\cdot}_2 \right)} du \overset{\sroman{1}}{\leq} \int_0^D \sqrt{\left( \sum_{l=1}^d h_l \cdot h_{l-1} \right) ln \left( 1 + \frac{2 d K \sqrt{M}}{u} \right)} du \\[1mm]
        &\overset{\sroman{2}}{\leq} \sqrt{\sum_{l=1}^d h_l \cdot h_{l-1}} \int_0^D \sqrt{\frac{2 d K \sqrt{M}}{u}} du = \sqrt{2 d K  \sum_{l=1}^d h_l \cdot h_{l-1}} \sqrt[4]{M} \int_0^D u^{-\frac{1}{2}} du \\[1mm]
        &= 2 \sqrt{2 d K D \sum_{l=1}^d h_l \cdot h_{l-1}} \sqrt[4]{M} \overset{\sroman{3}}{\leq} 2\sqrt{2} \sqrt{\sum_{l=1}^d h_l \cdot h_{l-1}} \sqrt{d} K \sqrt{M},
    \end{aligned}
\end{equation*}
where $\sroman{1}$ is from \eqref{ieq-A12-H-covering-bound}; $\sroman{2}$ comes from the fact that $ln (1+x) \leq x, \forall{x \ge 0}$ and $\sroman{3}$ comes from \eqref{ieq-D-bound}.

Since we shows that $\mathcal{R}_\mathcal{S}(\mathcal{H}) \leq 2 \mathbb{E} \left[ \underset{\underset{\normvec{\gamma^\prime-\gamma}_2\leq \delta}{\gamma,\gamma^\prime \in \mathcal{H}(\mathcal{S})}}{\sup}(X_\gamma - X_{\gamma^\prime}) \right] + 32 \mathcal{J}(\delta/4; D)$ before, take $\delta \xrightarrow[]{} 0^+$, we have:
\begin{equation*}
    \mathcal{R}_\mathcal{S}(\mathcal{H}) \leq 32 \mathcal{J}(0; D) \leq 64 \sqrt{2} \sqrt{\sum_{l=1}^d h_l \cdot h_{l-1}} \sqrt{d} K \sqrt{M} = O\left( \sqrt{\sum_{l=1}^d h_l \cdot h_{l-1}} \sqrt{d} K \sqrt{M} \right).
\end{equation*}
\end{proof}

\subsection{Proof of Theorem \ref{thm-nn-Rs(H)-bound-1,infnorm}} \label{apd-prf-thm-nn-Rs(H)-bound-1,infnorm}
\begin{customthm}{{\ref{thm-nn-Rs(H)-bound-1,infnorm}}}
Let $\mathcal{U}(x)=\left\{ x^\prime | \normvec{x^\prime - x}_p \leq \epsilon \right\}$ (i.e. consider the $\ell_p$ attack), $\sigma(0)=0$ with Lipschitz constant $L$; moreover, let $\mathcal{F}=\left\{\!W_d \sigma(\!W_{d-1} \sigma(\!\cdots\!\sigma(\!W_1 x)))\big|\normvec{W_l}_{1, \infty}\!\!\leq\!\!M_l^{1, \infty}, l\!\!=\!\!1\!,\dots,\!d\! \right\}$. We then have:
\begin{equation*}
    \mathcal{R}_\mathcal{S}(\mathcal{H}) = O\left( \sqrt{\sum_{l=1}^d h_l h_{l-1}} \sqrt{d K_0 K_1} \sqrt{M} \right),
\end{equation*}
where
\begin{equation*}
    \begin{aligned}
        K_0 &= 2 B_{X,\epsilon}^{1, \infty} \cdot \left( B_{X^+}^\prime + B_{X^-}^\prime \right), \ K_1 = \frac{K_0}{2} + B_{X,\epsilon}^\prime \cdot \left( B_{X^+}^{1, \infty} + B_{X^-}^{1, \infty} \right),
    \end{aligned}
\end{equation*}
where
\begin{equation*}
    \begin{aligned}
       B_{X,\epsilon}^\prime &= L^{d-1} \prod_{l=1}^d h_l M_l^{1, \infty} \ m^{1-\frac{1}{p}}\left( \normvec{X}_{p, \infty} + \epsilon \right), \ 
       B_X^\prime = L^{d-1} \prod_{l=1}^d h_l M_l^{1, \infty} \ m^{1-\frac{1}{p}} \normvec{X}_{p, \infty}, \\[1mm]
       B_{X,\epsilon}^{1, \infty} &= L^{d-1} \prod_{l=1}^d M_l^{1, \infty} \ \left( \normvec{X}_{p, \infty} + \epsilon \right), \ 
       B_X^{1, \infty} = L^{d-1} \prod_{l=1}^d M_l^{1, \infty} \ \normvec{X}_{p, \infty}.
    \end{aligned}
\end{equation*}
\end{customthm}

Before giving the proof, we firstly introduce some useful lemmas.
\begin{lemma} \label{lma-vinfnorm-m1infnorm-property}
Let $A \in \mathbb{R}^{m \times n}, b \in \mathbb{R}^n$, then we have: $\normvec{A\cdot b}_\infty \leq \normvec{A}_{1,\infty} \cdot \normvec{b}_\infty.$
\end{lemma}
\begin{proof} [Proof of Lemma \ref{lma-vinfnorm-m1infnorm-property}]
Let $A_i$ be the rows of $A, i=1,2,\dots,m$, we have:
\begin{equation*}
    \normvec{A \cdot b}_\infty = \underset{1 \leq i \leq m}{\max} |A_i b| \overset{\sroman{1}}{\leq} \underset{1 \leq i \leq m}{\max} \left( \normvec{A_i}_1 \cdot \normvec{b}_\infty \right) = \normvec{A}_{1,\infty} \cdot \normvec{b}_\infty,
\end{equation*}
where $\sroman{1}$ is from the Holder's Inequality.
\end{proof}

\begin{lemma} \label{lma-v1norm-m1infnorm-property}
Let $A \in \mathbb{R}^{m \times n}, b \in \mathbb{R}^n$, then we have:
\begin{equation*}
    \normvec{A\cdot b}_1 \leq \normvec{A}_{\infty,1} \cdot \normvec{b}_1 \leq m \normvec{A}_{1,\infty} \cdot \normvec{b}_1.
\end{equation*}
\end{lemma}
\begin{proof} [Proof of Lemma \ref{lma-v1norm-m1infnorm-property}]
Let $A_i$ be the rows of $A, i=1,2,\dots,m$, we have:
\begin{equation*}
    \normvec{A \cdot b}_1 = \sum_{i=1}^m |A_i b| \overset{\sroman{1}}{\leq} \sum_{i=1}^m \left( \normvec{A_i}_\infty \cdot \normvec{b}_1 \right) = \normvec{A}_{\infty,1} \cdot \normvec{b}_1,
\end{equation*}
where $\sroman{1}$ is from the Holder's Inequality. And we have:
\begin{equation*}
    \begin{aligned}
        \normvec{A}_{\infty,1} &= \normvec{\left( \normvec{A_1}_\infty, \cdots, \normvec{A_m}_\infty \right)}_1 \overset{\sroman{1}}{\leq} \normvec{\left( \normvec{A_1}_1, \cdots, \normvec{A_m}_1 \right)}_1 \\[1mm]
        &\overset{\sroman{2}}{\leq} m \normvec{\left( \normvec{A_1}_1, \cdots, \normvec{A_m}_1 \right)}_\infty = m \normvec{A}_{1,\infty},
    \end{aligned}
\end{equation*}
where $\sroman{1}$ is from the fact that $\normvec{x}_\infty \leq \normvec{x}_1$ and $\sroman{2}$ is the from the fact that for all $x\in \mathcal{R}^m, \normvec{x}_1 \leq m \normvec{x}_\infty$.
So we have: $\normvec{A\cdot b}_1 \leq \normvec{A}_{\infty,1} \cdot \normvec{b}_1 \leq m \normvec{A}_{1,\infty} \cdot \normvec{b}_1$.
\end{proof}

\begin{lemma} \label{lma-sigma-infnorm-Lip}
Suppose $\sigma$ is a $L$-Lipschitz function, then the elementwise vector map corresponding to $\sigma$ is also $L$-Lipschitz with respect to $\normvec{\cdot}_\infty$.
\end{lemma}
\begin{proof} [Proof of Lemma \ref{lma-sigma-infnorm-Lip}]
\begin{equation*}
    \begin{aligned}
        \normvec{\sigma(x)-\sigma(y)}_\infty &= \underset{1 \leq i \leq n}{\max} | \sigma(x)_i - \sigma(y)_i | = \underset{1 \leq i \leq n}{\max} | \sigma(x_i) - \sigma(y_i) | \\[1mm]
        &\overset{\sroman{1}}{\leq} \underset{1 \leq i \leq n}{\max} L | x_i - y_i | = L \cdot \normvec{x - y}_\infty,
    \end{aligned}
\end{equation*}
where $\sroman{1}$ is because $\sigma$ is $L$-Lipschitz.
\end{proof}
Now we can turn to the proof of Theorem \ref{thm-nn-Rs(H)-bound-1,infnorm}.
\begin{proof} [Proof of Theorem \ref{thm-nn-Rs(H)-bound-1,infnorm}]
In this case, let $\mathcal{U}(x) = \left\{ x^\prime \big| \normvec{x^\prime - x}_p \leq \epsilon \right\}$, we have:
\begin{equation*}
    \mathcal{F}=\left\{ x \xrightarrow[]{} W_d \sigma(W_{d-1} \sigma(\cdots \sigma(W_1 x)))\  \big| \  \normvec{W_l}_\F \leq M_l^\F, l=1,\dots, d  \right\},
\end{equation*}
\begin{equation*}
    \mathcal{H} = \left\{ h(x,x^+,x^-) = \underset{x^\prime \in \mathcal{U}(x)}{\min} \left( f(x^\prime)^T\left( f(x^+) - f(x^-) \right) \right) |f \in \mathcal{F} \right\}.
\end{equation*}
Let $S_l^\F = \left\{ W_l \in \mathbb{R}^{h_l \times h_{l-1}} \big| \normvec{W_l}_\F \leq M_l^\F \right\}, l=1,2,\dots,d$. Let $\mathcal{C}_l^\F$ be the $\delta_l$-covering of $S_l^\F$ and define:
\begin{equation*}
    \mathcal{F}^c = \left\{ f^c: x \xrightarrow[]{} W_d^c  \sigma \left( W_{d-1}^c \sigma \left( \cdots \sigma (W_1^c x) \right) \right) \Big| W_l^c \in \mathcal{C}_l^\F, l=1,2,\dots,d \right\} \subseteq \mathcal{F},
\end{equation*}
\begin{equation*}
    \mathcal{H}^c = \left\{ h^c(x,x^+,x^-) = \underset{x^\prime \in \mathcal{U}(x)}{\min} \left( f(x^\prime)^T\left( f(x^+) - f(x^-) \right) \right) |f \in \mathcal{F}^c \right\} \subseteq \mathcal{H}.
\end{equation*}
Similar to the proof of Theorem \ref{thm-linear-Rs(H)-bound}, we know that the Rademacher Process is a sub-Gaussian Process with respect to the Euclidean metric, which induces the $\ell_2$ norm.

Similar to the proof of Theorem \ref{thm-nn-Rs(H)-bound-Fnorm}, given the training set $\mathcal{S}=\{(x_i,x_i^+,x_i^-)\}_{i=1}^M \triangleq \{z_i\}_{i=1}^M$, define the $\ell_2$-norm for a function in $\mathcal{H}$ as:
\begin{equation*}
    \forall{h \in \mathcal{H}}, \  \normvec{h}_2 \coloneqq \sqrt{\sum_{i=1}^M \left[ h(z_i) \right]^2}.
\end{equation*}
Define $\mathcal{H}(\mathcal{S}) = \left\{\left(h(z_1),h(z_2),\dots,h(z_M)\right) \Big| h \in \mathcal{H} \right\}$, with the same argument as in proof of Theorem \ref{thm-nn-Rs(H)-bound-Fnorm}, we have:
\begin{equation*}
    \mathcal{N}\left(\delta; \mathcal{H}(\mathcal{S}), \normvec{\cdot}_2 \right) = \mathcal{N}\left(\delta; \mathcal{H}, \normvec{\cdot}_2 \right),
\end{equation*}
and
\begin{equation*}
    \mathcal{R}_\mathcal{S}(\mathcal{H}) = \mathbb{E}[\underset{\theta \in \mathcal{H}(\mathcal{S})}{\sup} X_\theta].
\end{equation*}
So by Lemma \ref{lma-dudley-int} and \eqref{ieq-dudley-remark}, we know that: $\forall{\delta \in (0,D]}$:
\begin{equation*}
    \begin{aligned}
        \mathcal{R}_\mathcal{S}(\mathcal{H}) \!=\! \mathbb{E}[\underset{\theta \in \mathcal{H}(\mathcal{S})}{\sup} X_\theta] \!\leq\! \mathbb{E}\left[ \underset{\theta,\widetilde{\theta} \in \mathcal{H}(\mathcal{S})}{\sup} (X_\theta \!-\! X_{\widetilde{\theta}}) \right] \leq 2 \mathbb{E} \left[ \underset{\underset{\normvec{\gamma^\prime-\gamma}_2\leq \delta}{\gamma,\gamma^\prime \in \mathcal{H}(\mathcal{S})}}{\sup}(X_\gamma \!-\! X_{\gamma^\prime}) \right] \!+\! 32 \mathcal{J}(\delta/4; D),
    \end{aligned}
\end{equation*}
where
\begin{equation} \label{ieq-A13-2}
    D \!=\!\! \underset{\theta, \theta^\prime \in \mathcal{H}(\mathcal{S})}{\sup} \normvec{\theta - \theta^\prime}_2  \!\leq\! 2 \underset{\theta \in \mathcal{H}(\mathcal{S})}{\sup} \normvec{\theta}_2 \!=\! 2 \underset{h \in \mathcal{H}}{\sup} \normvec{h}_2 \!=\! 2 \underset{h \in \mathcal{H}}{\sup} \sqrt{\sum_{i=1}^M \left[ h(z_i) \right]^2} \leq 2 \sqrt{M} \underset{1 \leq i \leq M}{\underset{h \in \mathcal{H}}{\sup}} |h(z_i)|,
\end{equation}
and $\mathcal{J}(a;b) = \int_a^b \sqrt{ ln \mathcal{N}\left(u; \mathcal{H}(\mathcal{S}), \normvec{\cdot}_2 \right)} du = \int_a^b \sqrt{ln \mathcal{N}\left(u; \mathcal{H}, \normvec{\cdot}_2 \right)} du$

Then for any $f \in \mathcal{F}, x \in \mathcal{X}$, let $x^l$ be the output of $x$ passing through the first $l-1$ layers, we have:
\begin{equation*}
    \begin{aligned}
        \normvec{f(x)}_\infty &= \normvec{W_d \sigma (W_{d-1} x^{d-1})}_\infty \overset{\sroman{1}}{\leq} \normvec{W_d}_\F \cdot \normvec{\sigma (W_{d-1} x^{d-1})}_\infty \\[1mm]
        &\overset{\sroman{2}}{=} \normvec{W_d}_\F \cdot \normvec{\sigma (W_{d-1} x^{d-1}) - \sigma (\bm{0})}_\infty \overset{\sroman{3}}{\leq} L M_d^\F \normvec{W_{d-1} x^{d-1}}_\infty \leq \cdots \\[1mm]
        &\leq \!L^{d-1}\! \prod_{l=2}^d\! M_l^\F \normvec{W_1 x}_\infty \!\leq\! L^{d-1} \!\prod_{l=1}^d M_l^\F \normvec{x}_\infty \!\overset{\sroman{4}}{\leq}\! L^{d-1}\! \prod_{l=1}^d\! M_l^\F \max \left\{ 1, m^{-p} \right\} \normvec{x}_p \\[1mm]
        &= L^{d-1} \prod_{l=1}^d M_l^\F \ \normvec{x}_p \leq L^{d-1} \prod_{l=1}^d M_l^\F \  \normvec{X}_{p,\infty},
    \end{aligned}
\end{equation*}
where $\sroman{1}$ is from Lemma \ref{lma-vinfnorm-m1infnorm-property}; $\sroman{2}$ is from the fact that $\sigma(0)=0$; $\sroman{3}$ comes from the assumption that $\sigma$ is $L$-Lipschitz and $\normmm{W_d}_\F \leq M_d^\F$ and $\sroman{4}$ is attained by setting $r = 1, r^* = \infty$ in the proof of Lemma \ref{lma-perturbation-norm-bound}.

To simplify the notations, we define:
\begin{equation*}
    B_X^\F = L^{d-1} \cdot \prod_{l=1}^d M_l^\F \cdot \normvec{X}_{p, \infty}, B_X^\prime = L^{d-1} \cdot \left( \prod_{l=1}^d h_l \cdot M_l^\F \right) m^{1-\frac{1}{p}} \cdot \normvec{X}_{p, \infty}
\end{equation*}
\begin{equation*}
    B_{X, \epsilon}^\F = L^{d-1} \cdot \prod_{l=1}^d M_l^\F \cdot \left( \normvec{X}_{p, \infty} + \epsilon \right), B_{X,\epsilon}^\prime = L^{d-1} \cdot \left( \prod_{l=1}^d h_l \cdot M_l^\F \right) m^{1-\frac{1}{p}} \cdot \left( \normvec{X}_{p, \infty} + \epsilon \right).
\end{equation*}
So:
\begin{equation*}
    \forall{x \in \mathcal{X}, f \in \mathcal{F}}, \normvec{f(x)}_\infty \leq B_X^\F.
\end{equation*}
Similarly, we have:
\begin{equation} \label{ieq-A13-f(x')-linf-bound}
    \forall{x \in \mathcal{X}, f \in \mathcal{F}}, \  \forall{\ \normvec{x^\prime - x}_p \leq \epsilon}, \normvec{f(x^\prime)}_\infty \leq B_{X,\epsilon}^\F.
\end{equation}
Similarly,
\begin{equation*}
    \begin{aligned}
        \normvec{f(x)}_1 &= \normvec{W_d \sigma (W_{d-1} x^{d-1})}_1 \overset{\sroman{1}}{\leq} h_d \cdot \normvec{W_d}_\F \cdot \normvec{\sigma (W_{d-1} x^{d-1})}_1 \\[1mm]
        &\overset{\sroman{2}}{=} h_d \cdot \normvec{W_d}_\F \cdot \normvec{\sigma (W_{d-1} x^{d-1}) - \sigma (\bm{0})}_1 \overset{\sroman{3}}{\leq} h_d \cdot L \cdot M_d^\F \normvec{W_{d-1} x^{d-1}}_1 \leq \cdots \\[1mm]
        &\leq L^{d-1} \prod_{l=2}^d h_l M_l^\F \ \normvec{W_1 x}_1 \leq L^{d-1} \prod_{l=1}^d h_l M_l^\F \ \normvec{x}_1 \\[1mm]
        &\overset{\sroman{4}}{\leq} L^{d-1} \left( \prod_{l=1}^d h_l M_l^\F \right) \  \max \left\{ 1, m^{1-\frac{1}{p}} \right\} \normvec{x}_p \\[1mm]
        &= L^{d-1} \left( \prod_{l=1}^d h_l M_l^\F \right) m^{1-\frac{1}{p}} \normvec{x}_p \leq L^{d-1} \left( \prod_{l=1}^d h_l M_l^\F \right) m^{1-\frac{1}{p}} \normvec{X}_{p,\infty},
    \end{aligned}
\end{equation*}
where $\sroman{1}$ is from Lemma \ref{lma-v1norm-m1infnorm-property}; $\sroman{2}$ is from the fact that $\sigma(0)=0$; $\sroman{3}$ comes from the assumption that $\sigma$ is $L$-Lipschitz and $\normmm{W_d}_\F \leq M_d^\F$ and $\sroman{4}$ is attained by setting $r = \infty, r^* = 1$ in the proof of Lemma \ref{lma-perturbation-norm-bound}.

So we know that:
\begin{equation} \label{ieq-A13-f(x)-l1-bound}
    \forall{x \in \mathcal{X}, f \in \mathcal{F}}, \normvec{f(x)}_1 \leq B_X^\prime.
\end{equation}
Similarly, we have:
\begin{equation} \label{ieq-A13-f(x')-l1-bound}
    \forall{x \in \mathcal{X}, f \in \mathcal{F}}, \  \forall{\ \normvec{x^\prime - x}_p \leq \epsilon}, \normvec{f(x^\prime)}_1 \leq B_{X,\epsilon}^\prime.
\end{equation}
For any $h \in \mathcal{H}, z \in \mathcal{X}^3$, let $x^* = \underset{\normvec{x^\prime - x}_p \leq \epsilon}{\arg\min} f(x^\prime)^T \left( f(x^+) - f(x^-) \right)$ and let $x^l$ be the output of $x^*$ passing through the first $l-1$ layers, we have:
\begin{equation*}
    \begin{aligned}
        |h(z)| &= | \underset{\normvec{x^\prime - x}_p \leq \epsilon}{\inf} f(x^\prime)^T \left( f(x^+) - f(x^-) \right) | = |f(x^*)^T \left( f(x^+) - f(x^-) \right)| \\[1mm]
        &\leq \normvec{f(x^*)}_\infty \cdot \normvec{f(x^+) - f(x^-)}_1 \overset{\sroman{1}}{\leq} B_{X,\epsilon}^\F \cdot (B_{X^+}^\prime + B_{X^-}^\prime),
    \end{aligned}
\end{equation*}
where $\sroman{1}$ is from \eqref{ieq-A13-f(x')-linf-bound} and \eqref{ieq-A13-f(x)-l1-bound}. So we get:
\begin{equation} \label{ieq-A13-D-bound}
    D \overset{\sroman{1}}{\leq} 2 \sqrt{M} \underset{1 \leq i \leq M}{\underset{h \in \mathcal{H}}{\sup}} |h(z_i)| \leq 2 \sqrt{M} B_{X,\epsilon}^\F \cdot (B_{X^+}^\prime + B_{X^-}^\prime) \triangleq \sqrt{M} K_0,
\end{equation}
where $\sroman{1}$ is from \eqref{ieq-A13-2}. Now, we need to find the smallest distance between $\mathcal{H}$ and $\mathcal{H}^c$, i.e.
\begin{equation*}
    \underset{h \in \mathcal{H}}{\sup} \underset{h^c \in \mathcal{H}^c}{\inf} \normvec{h - h^c}_2.
\end{equation*}
By the discussion in \eqref{ieq-A13-2}, we have $\normvec{h-h^c}_2 \leq \sqrt{M} \underset{1 \leq i \leq M}{\max} | h(z_i) - h^c(z_i) |$. For any $z_i = (x_i, x_i^+, x_i^-), i=1,2,\dots,M$, given $h$ and $h^c$ such that $\normvec{W_l - W_l^c}_\F \leq \delta_l, l=1,2,\dots,d$, we have:
\begin{equation*}
    | h(z_i) - h^c(z_i) | = \big| \underset{\normvec{x_i^\prime - x_i}_p \leq \epsilon}{\inf} f(x_i^\prime)^T \left( f(x_i^+) - f(x_i^-) \right) - \underset{\normvec{x_i^\prime - x_i}_p \leq \epsilon}{\inf} f^c(x_i^\prime)^T \left( f^c(x_i^+) - f^c(x_i^-) \right) \big|.
\end{equation*}
Let $x_i^* = \underset{\normvec{x_i^\prime - x_i}_p \leq \epsilon}{\arg\inf} f(x_i^\prime)^T \left( f(x_i^+) - f(x_i^-) \right)$ and $x_i^c = \underset{\normvec{x_i^\prime - x_i}_p \leq \epsilon}{\arg\inf} f^c(x_i^\prime)^T \left( f^c(x_i^+) - f^c(x_i^-) \right)$, and let
\begin{equation*}
    y_i=\left\{
        \begin{array}{rcl}
        x_i^c       &      & {f(x_i^*)^T \left( f(x_i^+) - f(x_i^-) \right) \ge f^c(x_i^c)^T \left( f^c(x_i^+) - f^c(x_i^-) \right)}     \\ 
        x_i^*     &      & {otherwise}
        \end{array} \right..
\end{equation*}
Then we have:
\begin{equation} \label{ieq-A13-3}
    \begin{aligned}
        | h(z_i) - h^c(z_i) | &= | f(x_i^*)^T \left( f(x_i^+) - f(x_i^-) \right) - f^c(x_i^c)^T \left( f^c(x_i^+) - f^c(x_i^-) \right) | \\[1mm]
        &\overset{\sroman{1}}{\leq} | f(y_i)^T \left( f(x_i^+) - f(x_i^-) \right) - f^c(y_i)^T \left( f^c(x_i^+) - f^c(x_i^-) \right) | \\[1mm]
        &= | f(y_i)^T \left( f(x_i^+) - f(x_i^-) \right) - f^c(y_i)^T \left( f(x_i^+) - f(x_i^-) \right) \\[1mm]
        &\ \ \ \ \ \ \ \ \ \ \ \ \ \ \ \ \ + f^c(y_i)^T \left( f(x_i^+) - f(x_i^-) \right)  - f^c(y_i)^T \left( f^c(x_i^+) - f^c(x_i^-) \right) | \\[1mm]
        &\overset{\sroman{2}}{\leq} | \left( f(y_i) - f^c(y_i) \right)^T \left( f(x_i^+) - f(x_i^-) \right) | \\[1mm]
        &\ \ \ \ \ \ \ \ \ \ \ \ \ \ \ \ \ + | f^c(y_i)^T \left( f(x_i^+) - f^c(x_i^+) \right) | + | f^c(y_i)^T \left( f(x_i^-) - f^c(x_i^-) \right) | \\[1mm]
        &\overset{\sroman{3}}{\leq} (B_{X^+}^\prime + B_{X^-}^\prime) \normvec{f(y_i) - f^c(y_i)}_\infty + B_{X,\epsilon}^\prime \normvec{f(x_i^+)-f^c(x_i^+)}_\infty\\[1mm]
        &\ \ \ \ \ \ \ \ \ \ \ \ \ \ \ \ \  + B_{X,\epsilon}^\prime \normvec{f(x_i^-) - f^c(x_i^-)}_\infty,
    \end{aligned}
\end{equation}
where $\sroman{1}$ is easily verified by the definition of $y_i$; $\sroman{2}$ is from the triangle inequality and $\sroman{3}$ is from \eqref{ieq-A13-f(x)-l1-bound} and \eqref{ieq-A13-f(x')-l1-bound}.
Again we define $g_b^a(\cdot)$ as:
\begin{equation*}
    g_b^a(y) = W_b \sigma \left( W_{b-1} \sigma\left( \cdots W_{a+1} \sigma \left( W_a^c \cdots \sigma (W_1^c y) \right) \right) \right).
\end{equation*}
Then:
\begin{equation*}
    \begin{aligned}
        \normvec{f(y_i) - f^c(y_i)}_\infty &= \normvec{g_d^0(y_i) - g_d^d(y_i)}_\infty \\[1mm]
        &= \normvec{g_d^0(y_i) - g_d^1(y_i) + g_d^1(y_i) - g_d^2(y_i) + \cdots + g_d^{d-1}(y_i) - g_d^d(y_i)}_\infty \\[1mm]
        &\overset{\sroman{1}}{\leq} \normvec{g_d^0(y_i) - g_d^1(y_i)}_\infty + \cdots + \normvec{g_d^{d-1}(y_i) - g_d^d(y_i)}_\infty,
    \end{aligned}
\end{equation*}
where $\sroman{1}$ is from the triangle inequality.

Then we calculate $\normvec{g_d^{l-1}(y_i) - g_d^l(y_i)}_\infty, l=1,2,\dots,d$:
\begin{equation} \label{ieq-A13-gdln1-gdl-bound}
    \begin{aligned}
        \normvec{g_d^{l-1}(y_i) - g_d^l(y_i)}_\infty &= \normvec{W_d \sigma \left( g_{d-1}^{l-1}(y_i) \right) - W_d \sigma \left( g_{d-1}^l(y_i) \right)  }_\infty \\[1mm]
        &\overset{\sroman{1}}{\leq} \normvec{W_d}_\F \cdot \normvec{\sigma \left( g_{d-1}^{l-1}(y_i) \right) - \sigma \left( g_{d-1}^l(y_i) \right)}_\infty \\[1mm]
        &\overset{\sroman{2}}{\leq} L M_d^\F \normvec{g_{d-1}^{l-1}(y_i) - g_{d-1}^l(y_i)}_\infty \leq \cdots \\[1mm]
        &\overset{\sroman{3}}{\leq} L^{d-l}\cdot \prod_{j=l+1}^d M_j^\F \cdot \normvec{W_l \sigma \left( g_{l-1}^{l-1}(y_i) \right) - W_l^c \sigma \left( g_{l-1}^{l-1}(y_i) \right) }_\infty \\[1mm]
        &= L^{d-l}\cdot \prod_{j=l+1}^d M_j^\F \cdot \normvec{\left( W_l - W_l^c \right) \sigma \left( g_{l-1}^{l-1}(y_i) \right) }_\infty \\[1mm]
        &\overset{\sroman{4}}{\leq} L^{d-l}\cdot \prod_{j=l+1}^d M_j^\F \cdot \delta_l \cdot \normvec{\sigma \left( g_{l-1}^{l-1}(y_i) \right)}_\infty,
    \end{aligned}
\end{equation}
where $\sroman{1}$ is from Lemma \ref{lma-vinfnorm-m1infnorm-property}; $\sroman{2}$ comes from the assumption that $\sigma$ is $L$-Lipschitz and $\normvec{W_d}_\F \leq M_d^\F$; $\sroman{3}$ is from the definition of $g_b^a(\cdot)$ and $\sroman{4}$ is from Lemma \ref{lma-vinfnorm-m1infnorm-property} and the choice of $h^c$ when $h$ is fixed, which means that $\normvec{W_l - W_l^c}_\F \leq \delta_l$.

Next, we upper bound $\normvec{\sigma \left( g_{l-1}^{l-1}(y_i) \right)}_\infty$:
\begin{equation} \label{ieq-A13-sigmagln1ln1-bound}
    \begin{aligned}
        \normvec{\sigma \left( g_{l-1}^{l-1}(y_i) \right)}_\infty &= \normvec{\sigma \left( g_{l-1}^{l-1}(y_i) \right) - \sigma(\bm{0})}_\infty \overset{\sroman{1}}{\leq} L \cdot \normvec{g_{l-1}^{l-1}(y_i)}_\infty \\[1mm]
        &= L \cdot \normvec{W_{l-1}^c \sigma\left( g_{l-2}^{l-2}(y_i) \right)}_\infty \overset{\sroman{2}}{\leq} L \cdot \normvec{W_{l-1}^c}_\F \cdot \normvec{\sigma\left( g_{l-2}^{l-2}(y_i) \right)}_\infty \\[1mm]
        &\overset{\sroman{3}}{\leq} L \ M_{l-1}^\F \normvec{\sigma\left( g_{l-2}^{l-2}(y_i) \right)}_\infty \leq \cdots \leq L^{l-1} \cdot \prod_{j=1}^{l-1} M_j^\F \cdot \normvec{y_i}_\infty,
    \end{aligned}
\end{equation}
where $\sroman{1}$ is because $\sigma$ is $L$-Lipschitz; $\sroman{2}$ is from Lemma \ref{lma-vinfnorm-m1infnorm-property} and $\sroman{3}$ is because $\normvec{W_{l-1}^c}_\F \leq M_{l-1}^\F$.

\eqref{ieq-A13-gdln1-gdl-bound} and \eqref{ieq-A13-sigmagln1ln1-bound} show that:
\begin{equation*}
    \begin{aligned}
        \normvec{g_d^{l-1}(y_i)-g_d^l(y_i)}_\infty &\leq L^{d-1} \frac{\prod_{j=1}^d M_j^\F}{M_l^\F} \delta_l \normvec{y_i}_\infty \\[1mm]
        &\overset{\sroman{1}}{\leq} L^{d-1} \frac{\prod_{j=1}^d M_j^\F}{M_l^\F}\  \delta_l \  \left( \normvec{X}_{p,\infty} + \epsilon \right) = B_{X,\epsilon}^\F \  \frac{\delta_l}{M_l^\F},
    \end{aligned}
\end{equation*}
where $\sroman{1}$ is from Lemma \ref{lma-perturbation-norm-bound}.

Similarly, we have:
\begin{equation*}
    \begin{aligned}
        \normvec{g_d^{l-1}(x_i^+)-g_d^l(x_i^+)}_\infty &\leq L^{d-1} \frac{\prod_{j=1}^d M_j^\F}{M_l^\F} \delta_l \normvec{x_i^+}_\infty \\[1mm]
        &\leq L^{d-1} \frac{\prod_{j=1}^d M_j^\F}{M_l^\F}\  \delta_l \  \normvec{X^+}_{p,\infty} = B_{X^+}^\F \  \frac{\delta_l}{M_l^\F},
    \end{aligned}
\end{equation*}
\begin{equation*}
    \begin{aligned}
        \normvec{g_d^{l-1}(x_i^-)-g_d^l(x_i^-)}_\infty &\leq L^{d-1} \frac{\prod_{j=1}^d M_j^\F}{M_l^\F} \delta_l \normvec{x_i^-}_\infty \\[1mm]
        &\leq L^{d-1} \frac{\prod_{j=1}^d M_j^\F}{M_l^\F}\  \delta_l \  \normvec{X^-}_{p,\infty} = B_{X^-}^\F \  \frac{\delta_l}{M_l^\F}.
    \end{aligned}
\end{equation*}
Combine the above with \eqref{ieq-A13-3}:
\begin{equation*}
    \begin{aligned}
        | h(z_i) - h^c(z_i) | &\leq \left( B_{X^+}^\prime + B_{X^-}^\prime \right) \left( \normvec{g_d^0(y_i) - g_d^1(y_i)}_\infty + \cdots + \normvec{g_d^{d-1}(y_i) - g_d^d(y_i)}_\infty \right) \\[1mm]
        &\ \ \ \ \ \ \ \ + B_{X,\epsilon}^\prime \left( \normvec{g_d^0(x_i^+) - g_d^1(x_i^+)}_\infty + \cdots + \normvec{g_d^{d-1}(x_i^+) - g_d^d(x_i^+)}_\infty \right) \\[1mm]
        &\ \ \ \ \ \ \ \ + B_{X,\epsilon}^\prime \left( \normvec{g_d^0(x_i^-) - g_d^1(x_i^-)}_\infty + \cdots + \normvec{g_d^{d-1}(x_i^-) - g_d^d(x_i^-)}_\infty \right) \\[1mm]
        &=(B_{X^+}^\prime + B_{X^-}^\prime) B_{X,\epsilon}^\F \sum_{l=1}^d \frac{\delta_l}{M_l^\F} + B_{X,\epsilon}^\prime B_{X^+}^\F \sum_{l=1}^d \frac{\delta_l}{M_l^\F} + B_{X,\epsilon}^\prime B_{X^-}^\F \sum_{l=1}^d \frac{\delta_l}{M_l^\F} \\[1mm]
        &= \left[ B_{X,\epsilon}^\F (B_{X^+}^\prime + B_{X^-}^\prime) + B_{X,\epsilon}^\prime (B_{X^+}^\F + B_{X^-}^\F) \right] \sum_{l=1}^d \frac{\delta_l}{M_l^\F} \triangleq K_1 \sum_{l=1}^d \frac{\delta_l}{M_l^\F}.
    \end{aligned}
\end{equation*}
So $\normvec{h-h^c}_2 = \sqrt{M} \underset{1 \leq i \leq M}{\max} | h(z_i) - h^c(z_i) | \leq \sqrt{M} \sum_{l=1}^d \frac{K_1 \delta_l}{M_l^\F}$. Let $\delta_l = \frac{M_l^\F \delta}{d K_1 \sqrt{M}}$, then:
\begin{equation*}
    \normvec{h - h^c}_2 \leq \sqrt{M} \sum_{l=1}^d \frac{K_1}{M_l^\F} \cdot \frac{M_l^\F \delta}{d K_1 \sqrt{M}} \leq \delta,
\end{equation*}
which means that: $\forall{h \in \mathcal{H}}, \exists{h^c \in \mathcal{H}^c}\  s.t.\  \normvec{h - h^c}_2 \leq \delta$, so we have $\underset{h \in \mathcal{H}}{\sup} \underset{h^c \in \mathcal{H}^c}{\inf} \normvec{h - h^c}_2 \leq \delta$ when choosing $\delta_l = \frac{M_l^\F \delta}{d K_1 \sqrt{M}}$.

So $\mathcal{H}^c$ is a $\delta$-covering of $\mathcal{H}$, and $\mathcal{N}(\delta; \mathcal{H}, \normvec{\cdot}_2) \leq | \mathcal{H}^c | = \prod_{l=1}^d | \mathcal{C}_l^\F |$. By Lemma \ref{lma-unit-ball-covering-number-bound} we know that: $| \mathcal{C}_l^\F | = \mathcal{N}(\frac{M_l^\F \delta}{d K_1 \sqrt{M}}; S_l^\F, \normvec{\cdot}_\F) \leq \left( 1 + \frac{2 d K_1 \sqrt{M}}{\delta} \right)^{h_l \times h_{l-1}}$.

This means:
\begin{equation} \label{ieq-A13-H-covering-bound}
    \mathcal{N}(\delta; \mathcal{H}, \normvec{\cdot}_2) \leq | \mathcal{H}^c | = \prod_{l=1}^d | \mathcal{C}_l^\F | \leq \left( 1 + \frac{2 d K_1 \sqrt{M}}{\delta} \right)^{\sum_{l=1}^d h_l \cdot h_{l-1}}.
\end{equation}
So we can conclude that:
\begin{equation*}
    \begin{aligned}
        \mathcal{J}(0;D) &= \int_0^D \sqrt{ ln \  \mathcal{N}\left(u; \mathcal{H}, \normvec{\cdot}_2 \right)} du \overset{\sroman{1}}{\leq} \int_0^D \sqrt{\left( \sum_{l=1}^d h_l \cdot h_{l-1} \right) ln \left( 1 + \frac{2 d K_1 \sqrt{M}}{u} \right)} du \\[1mm]
        &\overset{\sroman{2}}{\leq} \sqrt{\sum_{l=1}^d h_l \cdot h_{l-1}} \int_0^D \sqrt{\frac{2 d K_1 \sqrt{M}}{u}} du = \sqrt{2 d K_1  \sum_{l=1}^d h_l \cdot h_{l-1}} \sqrt[4]{M} \int_0^D u^{-\frac{1}{2}} du \\[1mm]
        &= 2 \sqrt{2 d K_1 D \sum_{l=1}^d h_l \cdot h_{l-1}} \sqrt[4]{M} \overset{\sroman{3}}{\leq} 2\sqrt{2} \sqrt{\sum_{l=1}^d h_l \cdot h_{l-1}} \sqrt{d K_0 K_1} \sqrt{M},
    \end{aligned}
\end{equation*}
where $\sroman{1}$ is from \eqref{ieq-A13-H-covering-bound}; $\sroman{2}$ comes from the fact that $ln (1+x) \leq x, \forall{x \ge 0}$ and $\sroman{3}$ comes from \eqref{ieq-A13-D-bound}.

Since we shows that $\mathcal{R}_\mathcal{S}(\mathcal{H}) \leq 2 \mathbb{E} \left[ \underset{\underset{\normvec{\gamma^\prime-\gamma}_2\leq \delta}{\gamma,\gamma^\prime \in \mathcal{H}(\mathcal{S})}}{\sup}(X_\gamma - X_{\gamma^\prime}) \right] + 32 \mathcal{J}(\delta/4; D)$ before, take $\delta \xrightarrow[]{} 0^+$, we have:
\begin{equation*}
    \mathcal{R}_\mathcal{S}(\mathcal{H}) \!\leq\! 32 \mathcal{J}(0; D) \!\leq\! 64 \sqrt{2} \sqrt{\sum_{l=1}^d h_l \cdot h_{l-1}} \sqrt{d K_0 K_1} \sqrt{M} \!=\! O\!\!\left(\! \sqrt{\sum_{l=1}^d h_l \cdot h_{l-1}} \sqrt{d K_0 K_1} \sqrt{M} \right)\!.
\end{equation*}
\end{proof}

\section{Extra Experimental Results} \label{apd-exp-results}
In this section, we present our experimental results for \textbf{CIFAR-100}. The basic settings are the same as $\S$\ref{sec-exp}.
\subsection{Improvement from Regularizer}

\begin{figure}
    \centering    
    \subfigure[Influence on clean accuracy] {    
        \includegraphics[scale=0.4]{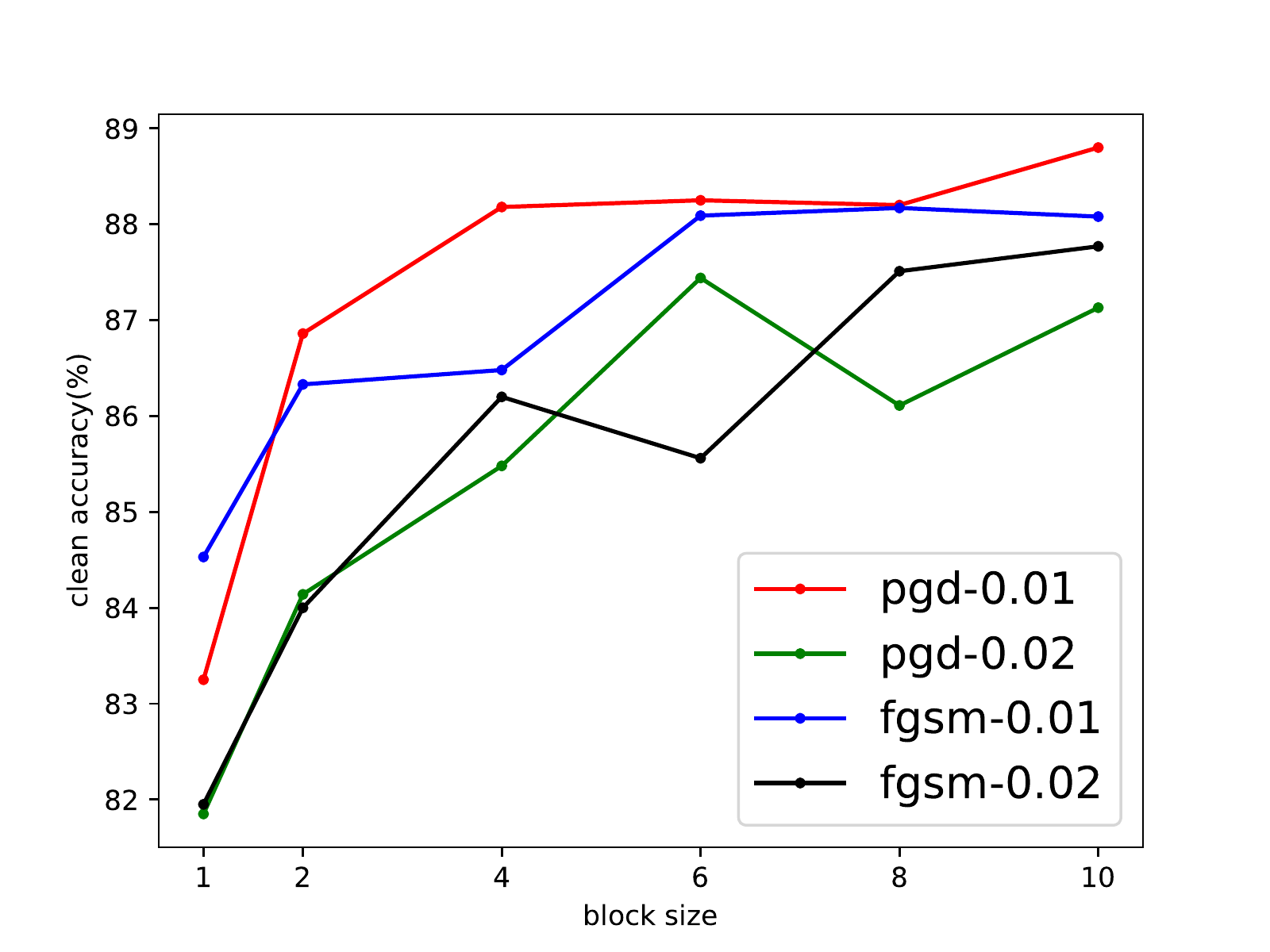}  
    }     
    \subfigure[Influence on adversarial accuracy] { 
        \includegraphics[scale=0.4]{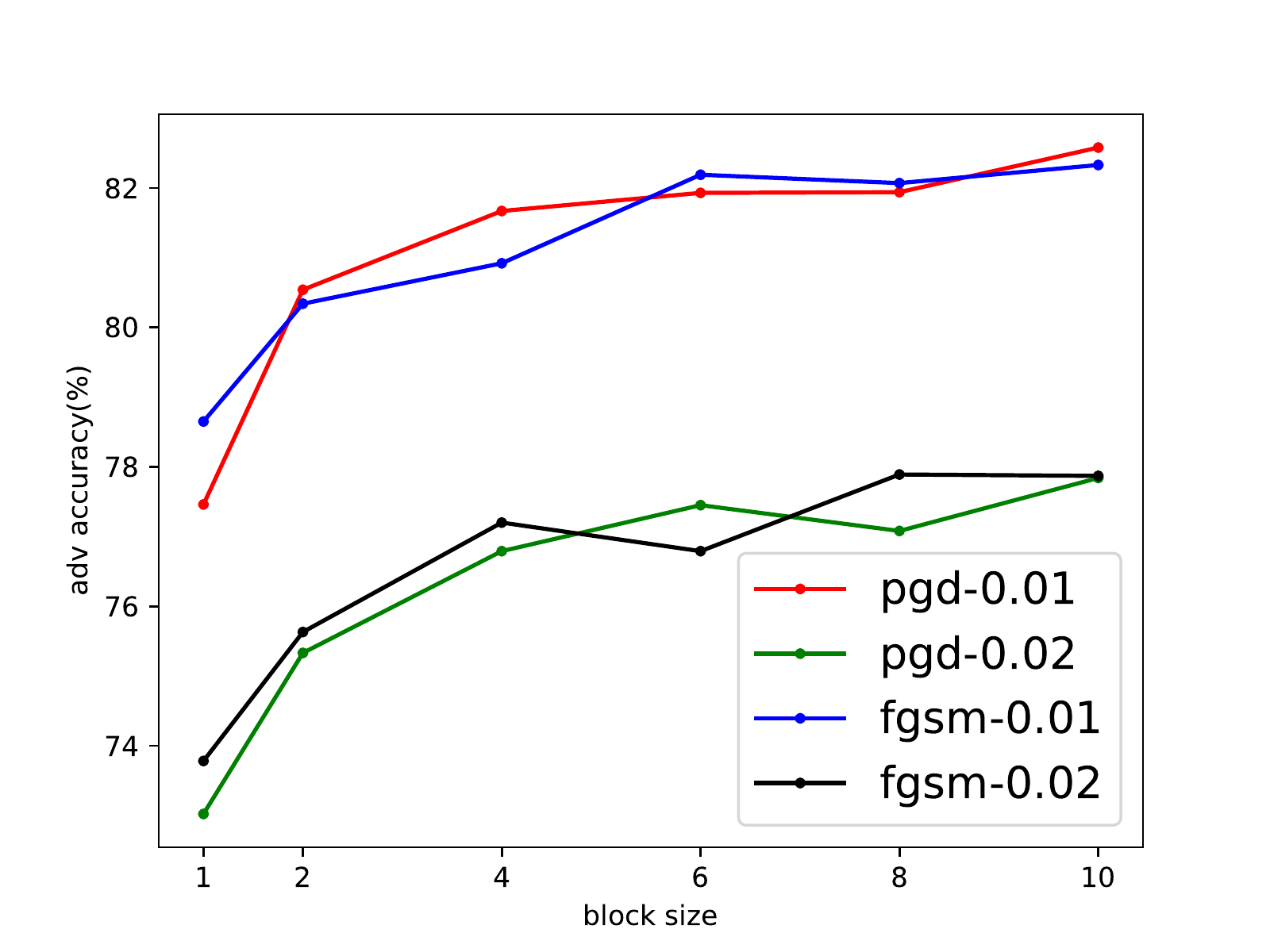}     
    }
    \caption{The effect of block size on the accuracy. In the figure, we show the clean accuracy and the adversarial accuracy of the mean classifier under PGD and FGSM attack with $\epsilon=0.01$ and $\epsilon=0.02$. The block size is choosen from $\{1, 2, 4, 6, 8, 10\}$. \textbf{(a) The influence on the clean accuracy;} \textbf{(b) The influence on the adversarial accuracy.}} \label{fig-100-b}
\end{figure}

\begin{table}
\centering
\setlength{\tabcolsep}{4mm}{
\begin{tabular}{cccccccc}
\hline \hline
\multirow{2}{*}{\textsc{Attack}} & \multirow{2}{*}{$\epsilon$} & \multirow{2}{*}{\textsc{Type}} & \multicolumn{5}{c}{$\lambda$}         \\ \cline{4-8} 
                        &                          &                       & $0$ & $0.002$ & $0.005$ & $0.01$ & $0.02$  \\ \hline
\multirow{4}{*}{PGD}    & \multirow{2}{*}{0.01}       & \textsc{Clean}     & $84.07$ & $\bm{85.65}$  & $85.48$  & $85.52$ & $58.53$ \\
                        &                          & \textsc{Adv}          & $78.49$  & $\bm{79.55}$ & $79.39$  & $79.41$ & $79.44$ \\ \cline{2-8} 
                        & \multirow{2}{*}{0.02}       & \textsc{Clean}     & $81.87$  & $81.91$ & $81.93$  & $\bm{82.08}$& $81.97$ \\
                        &                          & \textsc{Adv}          & $73.01$  &$73.05$ & $73.07$ & $\bm{73.16}$& $73.07$ \\ \hline
\multirow{4}{*}{FGSM}   & \multirow{2}{*}{0.01}       & \textsc{Clean}     & $84.61$  & $84.53$ & $\bm{84.95}$ & $84.58$& $84.55$ \\ 
                        &                          & \textsc{Adv}          & $78.70$  & $78.64$ & $\bm{79.18}$ & $78.70$& $78.63$ \\ \cline{2-8} 
                        & \multirow{2}{*}{0.02}       & \textsc{Clean}     & $80.47$ & $80.48$ & $80.39$ & $80.44$& $\bm{82.09}$ \\
                        &                          & \textsc{Adv}          & $72.45$ & $72.45$ & $72.36$ & $72.41$& $\bm{73.87}$ \\ \hline \hline
\end{tabular}}
\caption{Results of experiments on the regularizer on data set \textbf{CIFAR-100}. In this table, we list the clean accuracy (\textsc{Clean}) and adversarial accuracy (\textsc{Adv}) of the mean classifier under the PGD and FGSM attack with $\epsilon=0.01$ and $\epsilon=0.02$. $\lambda$ is chosen from $\{0, 0.002, 0.005, 0.01, 0.02\}$, and $\lambda=0$ indicates no regularizer.} \label{tab_100_reg}
\end{table}

Table \ref{tab_100_reg} shows that the $F$-norm regularizer helps to improve the adversarial robustness of the model, which agrees with our Theorem \ref{thm-nn-Rs(H)-bound-Fnorm}.

\subsection{Effect of Block Size}
Figure \ref{fig-100-b} records the influence of block size on the clean accuracy and the adversarial accuracy of the model, from which we can see that a larger block size will yield better adversarial accuracy.

\bibliography{refs}

\end{document}